\documentclass[11pt,a4paper]{article}
\usepackage{lmodern}
\usepackage[margin=1in]{geometry}
\usepackage{authblk}
\usepackage[numbers,compress]{natbib}

\usepackage[utf8]{inputenc} 
\usepackage[T1]{fontenc}    
\usepackage{hyperref}       
\usepackage{url}            
\usepackage{booktabs}       
\usepackage{amsfonts}       
\usepackage{nicefrac}       
\usepackage{microtype}      
\usepackage{xcolor}         
\usepackage{graphicx}       

\usepackage{amsmath}
\usepackage{amssymb}
\usepackage{mathtools}
\usepackage{amsthm}

\usepackage{multirow}
\usepackage{caption}
\usepackage{cancel}
\usepackage{array}
\usepackage{longtable}

\theoremstyle{plain}

\theoremstyle{definition}

\theoremstyle{remark}

\title{\textbf{Mask2Cause: Temporal Causal Discovery Beyond Causality in Mean}\thanks{Code is publicly available at: \url{https://github.com/omar826/Mask2Cause}}}

\author[1]{Omar Muhammad}
\author[1]{Pasupuleti Dhruv Shivkant}
\author[1]{Deepak N. Subramani}
\affil[1]{Indian Institute of Science, Bengaluru, India}
\date{}

\begin{document}

\maketitle

\begin{abstract}
    One approach to discovering causal relationships in multivariate time series is to ask whether a variable’s history improves prediction of another variable. When this improvement is measured solely by reductions in optimal squared error, the criterion misses relationships that affect the target’s conditional variance without changing its conditional mean. We propose Mask2Cause, an end-to-end Transformer framework for causal discovery through conditional mean and variance prediction. Each variable’s history is embedded as a separate token, and a single learnable adjacency matrix, shared across layers and attention heads, controls information exchange between tokens. Gaussian negative log-likelihood and a sparsity penalty on this matrix jointly train the forecasting model and graph. Our population analysis connects the optimal log-loss reduction to causally conditioned directed information, providing an information-theoretic motivation for the training objective. We introduce Mixed Physics, a benchmark with separately controlled mean and variance dependencies, to evaluate recovery of mean-only and variance-only edges. In its pure-variance setting, Mask2Cause achieves 1.000 off-diagonal AUROC, compared with 0.521 for a corresponding squared-error model. Experiments on linear, chaotic, biological, and real-data-derived benchmarks demonstrate competitive graph recovery, while the inferred graphs support downstream tasks such as causal pruning for efficient forecasting and root-cause localization.
  
\end{abstract}

\begingroup
\setlength{\textfloatsep}{12pt plus 2pt minus 2pt}
\setlength{\floatsep}{10pt plus 2pt minus 2pt}
\setlength{\intextsep}{10pt plus 2pt minus 2pt}
\captionsetup{font=small,skip=4pt}
\section{Introduction}

Uncovering directed relationships within multivariate time series is a fundamental challenge in machine learning, with applications in forecasting, root cause analysis, and anomaly detection. Temporal causal discovery seeks to recover these relationships from observational data under assumptions on the underlying dynamics, using approaches such as conditional-independence testing \citep{runge2019detecting}, graph-score optimization \citep{pamfil2020dynotears}, and restrictions on noise mechanisms \citep{hyvarinen2010estimation,peters2013causal}. We adopt the predictive perspective of Granger causality \citep{granger1969investigating}: a time series $Y$ Granger-causes another series $X$ if including $Y$'s history improves prediction of $X$ beyond what is possible using the remaining observed histories, including $X$'s own past. Such predictive dependence need not involve a change in the conditional mean: in heteroscedastic systems, a source may modulate the target’s conditional variance. In financial markets, \textit{volatility spillovers} describe how fluctuations in one asset's returns can help predict the volatility of another's \citep{diebold2009measuring}, even when no corresponding predictive relationship is detected in mean returns \citep{gemici2021causality}. A biological example arises in retinoic acid signaling, where changes in the binding protein Crabp2a alter temporal variability without changing mean retinoic acid levels \citep{sosnik2016noise}. Yet a criterion based solely on reductions in optimal squared-error prediction cannot detect a source whose contribution is confined to conditional variance. Recovering these relationships therefore requires a discovery objective sensitive to predictive information beyond the mean.

To address this problem, we propose Mask2Cause (M2C), an end-to-end neural framework for discovering dependencies expressed through conditional mean, conditional variance, or both. We combine inverted variable embedding \citep{liuitransformer}, which encodes each variable's history as a token, with adjacency-constrained masked attention, which gates attention according to learned parent relationships through a differentiable, sparsity-regularized adjacency matrix. Shared across encoder layers and attention heads, this matrix controls cross-variable information flow and is optimized directly alongside the forecasting weights. Its entries provide continuous edge scores, which are thresholded to obtain a stationary summary graph of lagged predictive dependencies. This design accommodates both squared-error prediction and heteroscedastic Gaussian likelihood prediction within this framework. In the latter, conditional mean and variance heads inform the same adjacency matrix, allowing either form of dependence to support an edge. Our population analysis motivates training objective with directed information \citep{quinn2015directed}. We introduce Mixed Physics to evaluate mean-only and variance-only edge recovery through controlled comparisons, alongside standard discovery benchmarks and downstream applications.

Our contributions are fourfold: \textbf{(1) A shared graph-learning end-to-end architecture.} Unlike post-hoc frameworks, Mask2Cause integrates the causal structure directly into the forward pass. The structural adjacency matrix and the forecasting weights are optimized simultaneously, and the matrix that the sparsity penalty acts on is the causal graph itself, so no separate extraction step stands between training and the graph. This reduces the risk of the model overfitting to spurious correlations.
\textbf{(2) A population-level analysis of the prediction objective.} We establish conditions under which a variance-only parent's history strictly reduces optimal Gaussian negative log-likelihood while leaving optimal squared error unchanged. \textbf{(3) A controlled benchmark for mean and variance dependence.} We introduce Mixed Physics, with separately controlled mean and variance parents, to evaluate graph recovery under combinations of the two dependence types. Corresponding squared-error and Gaussian-likelihood variants allow us to examine the effect of the prediction objective. \textbf{(4) An evaluation of recovery, efficiency, and downstream utility.} We assess graph recovery on linear, chaotic, biological, and real-data-derived benchmarks, examine architectural choices and computational costs, and evaluate the usefulness of inferred graphs for forecasting-model pruning and root cause localization.
\section{Related Works}

Classical Granger causal discovery commonly models linear dynamics through vector autoregression \citep{lozano2009grouped,lutkepohl2005new}. To accommodate nonlinear relationships, neural approaches introduced component-wise predictors, including SRU/eSRU, TCDF, and NGC \citep{Khanna2020Economy,nauta2019causal,tank2021neural}, with more recent work employing Kolmogorov–Arnold networks \citep{liu2025kolmogorov}. Separate target predictors limit parameter sharing and can increase costs as dimensionality grows. Shared architectures address this limitation: CUTS+ uses message-passing networks for irregular time series \citep{cheng2024cuts+}, while JRNGC trains a joint predictor with Jacobian regularization \citep{zhou2024jacobian}. Other approaches combine prediction with exogenous-variable modeling, as in AERCA \citep{han2025aerca}, or learn shared representations for dynamic discovery, as in UnCLe \citep{bi2025unclescalabledynamiccausal}. Attention flexibly models temporal and cross-variable interactions, motivating CausalFormer, CSAM, and LLM-GC \citep{kong2024causalformer,liu2025,liu2026llm}. Graph extraction varies: CausalFormer uses relevance propagation, UnCLe temporal perturbations, and CSAM attention aggregation with permutation-based verification. Even with training-time structural constraints, separate attribution introduces choices that can affect how faithfully edges reflect learned dependencies. M2C shares forecasting transformations across targets and directly optimizes a sparsity-regularized adjacency that both gates attention across layers and heads and supplies edge scores. This couples graph estimation to prediction without separate attribution before thresholding. However, those relying on squared-error prediction lack a direct learning signal for variance-only dependencies.

Causality in variance has long been studied through residual-based tests \citep{cheung1996causality}, multivariate GARCH models \citep{hafner2008testing}, and AR-BEKK formulations combining mean and variance dynamics \citep{luo2011granger}. These methods provide statistical inference under specified assumptions; complex nonlinear mean and variance interactions motivate neural alternatives. TS-LSNM models conditional location and scale using separate networks for each target, combining their weight-derived graph scores \citep{kikuchi2023structure}. RHINO uses conditional normalizing flows and variational graph inference to capture history-dependent noise \citep{gong2022rhino}, while FANTOM additionally learns regime-specific graphs and regime boundaries \citep{rahmani2026flow}. M2C shares an attention-based encoder across targets and couples mean and variance prediction through one directly optimized summary adjacency, avoiding both per-target network replication and flow-based density estimation. Its Gaussian likelihood provides a simpler objective for learning dependencies expressed through either conditional moment, while Mixed Physics separately evaluates mean-only, variance-only, and combined effects. To our knowledge, Mixed Physics is the first benchmark to do this. Appendix \ref{related} provides detailed comparisons.

\section{Mask2Cause Architecture}

\begin{figure}[t]
    \centering
    \includegraphics[width=\linewidth]{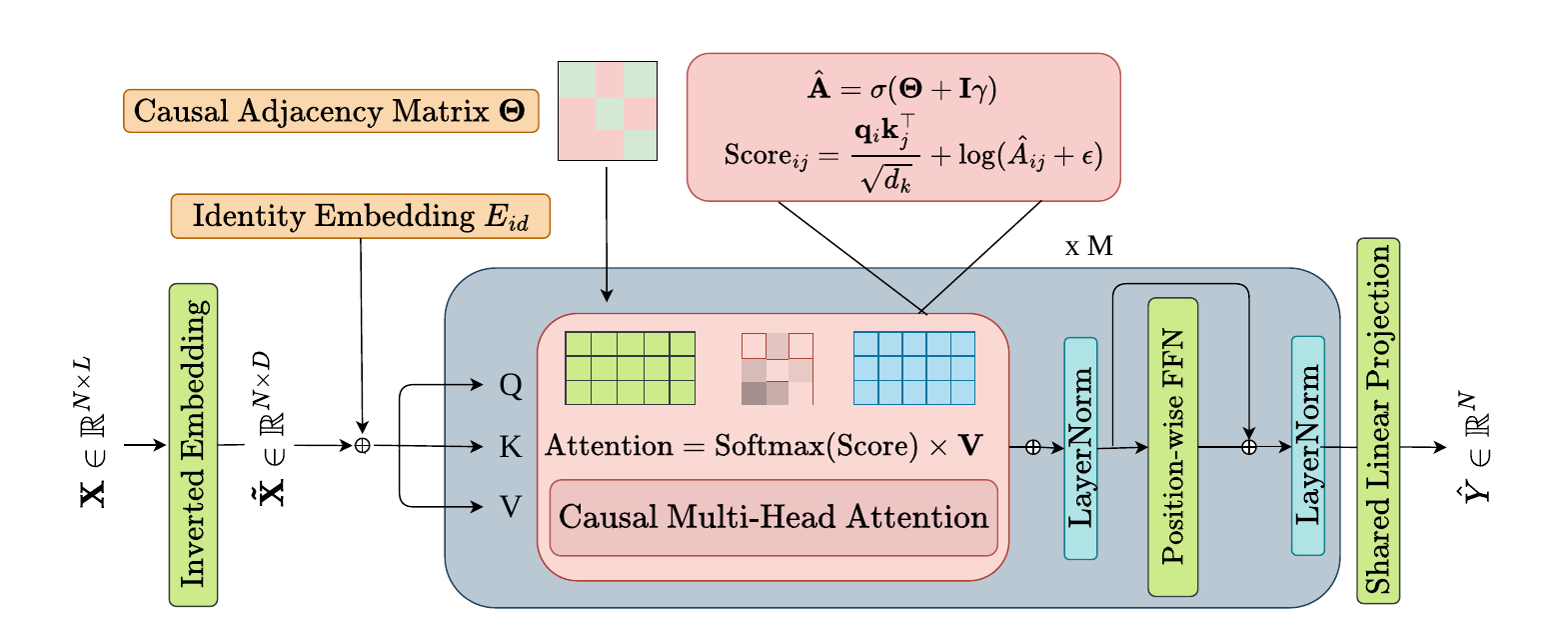}
    \caption{\textbf{The Mask2Cause Architecture.} The model maps a multivariate history into variable-specific tokens. A Transformer encoder, constrained by a learnable adjacency matrix, processes these tokens to predict the next-step state, discovering the causal graph through the forecasting objective.}
    \label{fig:causal_graph}
\end{figure}

\subsection{Problem Formulation and Assumptions}
\label{sec:problem_formulation}
We observe $X\in\mathbb{R}^{T\times N}$ with state $x_t=[x_t^1,\ldots,x_t^N]^\top$. Let $x_{<t}=(x_1,\ldots,x_{t-1})$ denote its history and $x_{<t}^{-j}$ the history excluding variable $j$. We seek a stationary summary graph $A\in\{0,1\}^{N\times N}$, where $A_{ij}=1$ denotes a predictive dependency $X^j\to X^i$ at one or more positive lags. The summary graph may contain feedback cycles; strict temporal precedence makes the time-unrolled graph acyclic. We learn this graph jointly with a forecasting model.

We ground the causal interpretation of the learned graph under the following assumptions commonly used in causal discovery, detailed in Appendix \ref{assume_app}: \label{assume} \textbf{(1) Strict Positivity:} The joint distribution is strictly positive. \textbf{(2) Causal Sufficiency:} There are no unobserved confounders. \textbf{(3) Strict Temporal Precedence:} Causal influences take time to propagate. \textbf{(4) Causal Faithfulness:} True causal effects do not perfectly cancel to produce statistical independencies. \textbf{(5) Stationarity and Finite Markov Order:} The process is stationary, with a time-invariant causal graph and Markov order at most $L$. Appendix \ref{violations} examines hidden confounding and switching graphs: confounding inflates scores among confounded children on VAR but has little effect on Lorenz-96; switching yields approximately the union of regime graphs, supporting screening. M2C approximates $x_{<t}$ with $x_{t-L:t-1}$.

M2C combines variable-history embeddings, adjacency-constrained attention, and shared prediction heads. We jointly optimize the continuous adjacency $\hat A\in[0,1]^{N\times N}$ and forecasting weights, then threshold it to obtain the summary graph (Appendix \ref{thresh}).

\subsection{Inverted Variable Embedding}
Standard time-series Transformers embed each time step as a token, mixing variables and making their individual influences difficult to isolate \cite{wu2021autoformer, zhou2021informer}. Instead, we project each variable's history window into a $d$-dimensional token:
\begin{equation}
    \mathbf{e}_i = \text{Dropout}(\mathbf{W}_{emb} \mathbf{x}^{i}_{t-L:t-1} + \mathbf{b}_{emb})\,,
\end{equation}
where $\mathbf{W}_{emb}\in\mathbb{R}^{d\times L}$ is learned. To distinguish between variables, we add a learnable variable identity embedding $\mathbf{E}_{id} \in \mathbb{R}^{N \times d}$ to the projected tokens:
\begin{equation}
    \mathbf{Z}^{(0)} = [\mathbf{e}_1, \dots, \mathbf{e}_N]^\top + \mathbf{E}_{id}\,.
\end{equation}
Each row of $\mathbf{Z}^{(0)}\in\mathbb{R}^{N\times d}$ encodes one variable's recent history and identity. Representing each variable's history as a token, produces an $N\times N$ attention matrix whose entries correspond to directed interactions between variables, allowing the learned adjacency to constrain information flow directly.

\subsection{Causal Masked Encoder}
The core of our approach is the modification of the self-attention mechanism to enforce causal structure. In a standard Transformer, the attention score $S_{ij}$ represents the relevance of token $j$ to token $i$. We explicitly constrain these scores using a global, learnable adjacency parameter $\mathbf{\Theta} \in \mathbb{R}^{N \times N}$. A sigmoid yields the continuous adjacency, with diagonal forcing preserving access to each variable's own history:
\begin{equation}
    \hat{\mathbf{A}} = \sigma(\mathbf{\Theta} + \mathbf{I} \cdot \gamma)\,.
\end{equation}
Here $\mathbf{I}$ is the identity and a large constant $\gamma$ (e.g., 100) makes diagonal entries approach 1. Although self-dependencies can be learned (Appendix \ref{sec:sensitivity}), this prior preserves the own-history access assumed by causally conditioned DI (Section \ref{sec:theoretical_grounding}). One adjacency is shared across all $M$ layers and attention heads, representing a global graph rather than layer-specific dependencies.

Given token queries $\mathbf{Q}$, keys $\mathbf{K}$ and values $\mathbf{V}$, we add a logarithmic adjacency mask to attention scores:
\begin{equation}
    \text{Score}_{ij} = \frac{\mathbf{q}_i \mathbf{k}_j^\top}{\sqrt{d_k}} + \log(\hat{A}_{ij} + \delta)\,,
    \label{eq:score}
\end{equation}
where $\delta>0$ ensures numerical stability. Near-unit gates leave attention essentially unchanged; near-zero gates suppress the source's contribution, with the positive floor making suppression approximate. Without this floor, a zero gate blocks attention. The head output is
\begin{equation}
    \text{Attention}(\mathbf{Q}, \mathbf{K}, \mathbf{V}, \hat{\mathbf{A}}) = \text{Softmax}(\text{Score}) \mathbf{V}\,.
\end{equation}
We stack $M$ encoder blocks with masked attention, per-token layer normalization, and position-wise feed-forward networks implemented as $1\times1$ convolutions. These blocks learn nonlinear representations while sharing the graph constraint. Because normalization and feed-forward operations act independently on each token, masked attention is the sole route for cross-variable information flow.

\subsection{Prediction and Optimization}
\label{sec:optimization}
The encoder output $\mathbf{Z}^{(M)}\in\mathbb{R}^{N\times d}$ feeds one of two prediction configurations:

\textbf{1. Heteroscedastic Model (Variance-Aware).}
We model a conditional Gaussian $\mathcal{N}(\mu,\sigma^2)$ with time-varying moments predicted by two linear heads: $\hat{\mu}_{t}=\mathbf{Z}^{(M)}\mathbf{w}_\mu+\mathbf{b}_\mu$ and $\hat{\sigma}^2_{t}=\text{Softplus}(\mathbf{Z}^{(M)}\mathbf{w}_\sigma+\mathbf{b}_\sigma)+\delta$. Negative log-likelihood (NLL) weights prediction errors by their estimated uncertainty:
\begin{equation}
    \mathcal{L}_{\text{NLL}} = \frac{1}{|\mathcal{B}|N} \sum_{t \in \mathcal{B}} \sum_{i=1}^N \left( \frac{1}{2} \log (\hat{\sigma}^i_{t})^2 + \frac{(x_{t}^i - \hat{\mu}_{t}^i)^2}{2(\hat{\sigma}^i_{t})^2} \right)\,,
    \label{eq:nll}
\end{equation}
where $\mathcal{B}$ contains the mini-batch time indices and $N$ is the variable count.

\textbf{2. Homoscedastic Model (Mean-Only).}
For additive noise with input-independent variance (e.g., standard VAR), we omit variance estimation and use one linear head, $\hat{\mathbf{x}}_{t}=\mathbf{Z}^{(M)}\mathbf{w}_{\text{out}}+\mathbf{b}_{\text{out}}$. With unit variance, NLL is equivalent, up to scaling and an additive constant, to mean squared error (MSE):
\begin{equation}
    \mathcal{L}_{\text{MSE}} = \frac{1}{|\mathcal{B}|N} \sum_{t \in \mathcal{B}} ||\mathbf{x}_{t} - \hat{\mathbf{x}}_{t}||^2_2\,.
    \label{eq:mse}
\end{equation}
\textbf{Joint Optimization.}
We combine forecasting loss $\mathcal{L}_{\text{pred}}\in\{\mathcal{L}_{\text{NLL}},\mathcal{L}_{\text{MSE}}\}$ with off-diagonal adjacency sparsity:
\begin{equation}
    \mathcal{L} = \mathcal{L}_{\text{pred}} + \lambda \cdot \frac{1}{N(N-1)} \sum_{i \neq j} \hat{A}_{ij}\,.
\end{equation}
Prediction and sparsity jointly optimize the same adjacency whose entries become the reported edge scores. In methods that penalize input weights or per-window coefficient tensors, graph extraction instead requires choices such as absolute values, aggregation over time and lags, and threshold selection. M2C avoids this intermediate attribution step: one graph parameter gates attention across layers and heads, receives the sparsity penalty, and is thresholded for evaluation.

\section{Theoretical Motivation}
\label{sec:theoretical_grounding}

\subsection{From Functional Causality to Directed Information}
\label{sec:functional_to_di}

Neural Granger methods learn a predictor $g_i(\mathbf{x}_{<t})$
of $x_t^i$ and identify noncausality through functional invariance
\citep{tank2021neural}: changing $\mathbf{x}_{<t}^{j}$ while
holding the remaining histories fixed does not change the
prediction. Under squared-error loss, the population-optimal
predictor is the conditional mean. A source that changes only
the target's conditional variance therefore leaves this
prediction unchanged, even though it affects the distribution
of possible outcomes (Appendix~\ref{proof:variance_only_example} gives a concrete example).

To capture these dependencies, we broaden the prediction from
a single value to a conditional distribution. Directed
Information (DI) quantifies the reduction in uncertainty about
a target's next value obtained from a source's past, beyond
what the remaining histories already provide.
The Directed Information Graph (DIG) contains an edge
$X^j \to X^i$ if and only if the DI from $X^j$ to $X^i$,
causally conditioned on the other variables
$\mathbf{X}^{-\{i,j\}}$, is strictly positive.
For $j \neq i$, this quantity is the expected cumulative
KL divergence between the true full and source-excluded
conditional distributions:
\begin{equation}
\begin{split}
&I(X^j \to X^i \Vert \mathbf{X}^{-\{i,j\}}) =
\sum_{t=1}^{T}
\mathbb{E}\!\left[
D_{\mathrm{KL}}\!\left(
p(x_t^i \mid \mathbf{x}_{<t})
\,\Vert\,
p(x_t^i \mid x_{<t}^{i},
                  \mathbf{x}_{<t}^{-\{i,j\}})
\right)
\right].
\end{split}
\label{eq:directed_information}
\end{equation}

Under history-independent, zero-mean additive noise and finite first moments, functional
noncausality and distributional noncausality coincide almost surely, making
the functional criterion a special case of the DIG formulation
(Appendix \ref{proof1}).
Under positivity and strict causality, the DIG also equals
the unique minimal generative model graph
\citep[Theorem~1]{quinn2015directed}: its parent sets are the
smallest sets of histories preserving the full predictive
distributions.

\subsection{Translating Directed Information into Model Design}
\label{sec:di_to_model}

\paragraph{Proposition (Optimal log-loss reduction).}
Let $\mathcal{L}(p,x)=-\log p(x)$.
For ideal predictors of the true full and source-excluded
conditional densities, assuming finite expected losses,
\begin{equation}
\begin{split}
&\mathbb{E}\!\left[
\sum_{t=1}^{T}
\bigl(
\mathcal{L}(p_{-j,t},x_t^i)
-
\mathcal{L}(p_{\mathrm{full},t},x_t^i)
\bigr)
\right] =
I(X^j \to X^i \Vert \mathbf{X}^{-\{i,j\}}),
\end{split}
\label{eq:di_logloss}
\end{equation}
where $p_{\mathrm{full},t}=p(\cdot\mid\mathbf{x}_{<t})$
and $p_{-j,t}=p(\cdot\mid\mathbf{x}_{<t}^{-j})$
\citep[Proposition~1]{quinn2015directed}.

For a candidate parent set $S$, let $\mathcal{L}_{S}$ denote
the cumulative negative log-likelihood of the true conditional
densities $p(x_t^i\mid x_{<t}^i,\mathbf{x}_{<t}^{S})$,
and let $\mathcal{L}_{\mathrm{own}}$ denote the corresponding
loss using only $p(x_t^i\mid x_{<t}^i)$.
The same identity gives
$I(\mathbf{X}^{S}\to X^i)
=\mathbb{E}[\mathcal{L}_{\mathrm{own}}-\mathcal{L}_{S}]$:
the information supplied by $S$ equals the reduction in
optimal expected log-loss obtained by adding its histories
to the target's own past.
The own-history term is constant with respect to $S$;
maximizing the captured predictive information therefore
amounts to minimizing negative log-likelihood.
This motivates our combination of likelihood prediction
and an $L_1$ penalty on the adjacency: prediction rewards
informative histories, while the sparsity penalty discourages unnecessary connections.

M2C approximates the ideal density predictors with a parameterized network and finite windows. Variable-specific tokens and masked attention connect graph learning to source-history access. With $M=1$, a zero gate suppresses a source. Consequently, the network's empirical increase in predictive loss points to a positive DI for the excluded parent. Furthermore, we establish in Appendix \ref{robust} that the approximation continues to hold in practice for multiple Transformer blocks ($M \ge 2$), allowing for the use of deep architectures that leverage the capacity of Transformers in capturing complex relations. By simultaneously imposing an $L_1$ penalty on the adjacency matrix, the optimization process is incentivized to retain only those parents $X^j$ whose inclusion reduces this loss, or equivalently, those that contribute non negligible Directed Information, thereby incentivizing recovery of the DIG.

Gaussian NLL provides a prediction objective sensitive to both
mean and variance dependencies. Under finite second moments,
positive conditional variances and suitable integrability,
minimizing its population expectation over unrestricted mean
and variance predictors yields the true conditional moments,
even for non-Gaussian data. Consequently, including a source's
history strictly reduces optimal Gaussian NLL if and only if
it changes either conditional moment on a set of positive
probability; optimal MSE decreases only when the conditional
mean changes (Appendix~\ref{proof:gaussian_moments}).
When the true full-history conditional distribution is Gaussian,
its mean and variance determine the entire distribution.
In this case, a source has positive DI if and only if including
its history strictly reduces optimal Gaussian NLL, even when
the source-excluded conditional distribution is non-Gaussian
(corollary in Appendix~\ref{proof:gaussian_moments}).

Alternative likelihoods retain the log-loss
motivation; Appendix \ref{heavytail} compares Gaussian and
Student-t likelihoods empirically. Appendix~\ref{proof:fitted_density} shows how prediction errors affect the connection between fitted likelihood models and DI.

\textbf{Diagonal Forcing as Structural Prior.} Because the target's own history is already conditioned on in Eq.~\ref{eq:directed_information}, its self-directed information under this definition is zero (Appendix \ref{proof:zero_self_info}), without implying that its past lacks predictive value. Diagonal forcing preserves this own-history access and matches the self-dependencies ($W_{ii}\neq0$) in standard benchmark graphs.

\section{Experiments}
\label{experiments}

Our experiments test whether M2C remains competitive on mean-driven benchmarks and recovers variance-driven dependencies missed by squared-error prediction (Section \ref{sec:theoretical_grounding}). We cover sparse linear systems, chaotic dynamics, gene regulatory networks and series generated from traffic, air-quality and medical recordings, including high-dimensional settings where standard configurations no longer saturate. Mixed Physics separates mean- and variance-driven edges with known ground truth.

\subsection{Baselines and Metrics}
We compare against component-wise, shared, attention-based and statistical baselines: cMLP and cLSTM \cite{tank2021neural}, TCDF \cite{nauta2019causal}, CUTS\cite{cheng2023cuts}, CUTS+ \cite{cheng2024cuts+}, Causalformer \cite{kong2024causalformer}, SRU and eSRU \cite{Khanna2020Economy}, PCMCI \cite{runge2019detecting}, NGM \cite{bellotneural}, LCCM \cite{de2020latent}, and the more recent AERCA \cite{han2025aerca} and UnCLe \cite{bi2025unclescalabledynamiccausal}, as well as RHINO \cite{gong2022rhino}. For KANGCI \cite{liu2025kolmogorov}, LLM-GC \cite{liu2026llm} and CSAM \cite{liu2025}, we were unable to obtain a runnable implementation. Appendix \ref{published} separately compares their published numbers with M2C evaluated on the reported configurations. Following standard practice \cite{tank2021neural}, we report AUROC: the probability that a true edge ranks above a non-edge, obtained by sweeping thresholds over continuous adjacency scores $\hat A_{ij}\in[0,1]$. Appendix \ref{metrics} reports AUPRC, SHD and F1.

\textbf{Scoring convention.} On every benchmark drawn from prior work we follow the scoring protocol established by its source paper, which is what allows the published baseline numbers to be quoted directly. Mixed Physics is our own benchmark and carries no such constraint, so we score it off-diagonal, removing self-loops from the metric and reporting only the inter-variable structure the benchmark was built to test.
\subsection{Datasets}
\subsubsection{Standard Benchmarks}
We use dataset implementations from the official SRU \cite{Khanna2020Economy} and CausalTime \cite{chengcausaltime} repositories. 

We assess performance on synthetic systems including the Vector Autoregressive ({VAR}) processes to test sparse linear recovery, and the chaotic Lorenz-96 system \cite{karimi2010extensive} (forcing constants $F=10, 40$) to evaluate nonlinear continuous-time dynamics. To validate efficacy on realistic, high-dimensional proxies, we employ the DREAM3 gene regulatory networks \cite{prill2010towards} and the CausalTime benchmark suite \cite{chengcausaltime}. Both are governed by real-world dynamics. The DREAM3 topologies are sub-networks extracted from the experimentally derived \textit{E. coli} and \textit{S. cerevisiae} transcriptional maps, so the structure being recovered is a real regulatory network, and CausalTime generates its series from deep models fitted to real air-quality, traffic and medical recordings, so the dynamics inherit the statistical properties of the source data. These are the standard means of obtaining real causal structure in domains where interventional ground truth is unobtainable. DREAM3, which features extremely short trajectories ($T=21$), is purely used to evaluate how the model fares against baselines in highly data-scarce domains. We also evaluate lag-only ER type graphs from RHINO's synthetic benchmark with history-dependent noise \citep{gong2022rhino}, using shared datasets and off-diagonal summary-graph AUROC after aggregating lagged edges.

\subsubsection{Mixed Physics (Heteroscedastic)}
Standard benchmarks (VAR, Lorenz) model interactions strictly via the conditional mean. To evaluate
dependencies expressed through conditional variance, we introduce
Mixed Physics, a synthetic benchmark with separately controlled
mean and variance effects.

\textbf{Generation Mechanism.} We simulate a system where disjoint sets of cross-variable parent variables control the mean and the variance of the target. Mean self-dependencies are retained in every
configuration. For variable $i$, the state evolves according to:
\begin{equation}
    x_t^i = \underbrace{\sum_{j} \mathbf{W}^{\mu}_{ij} x_{t-L}^j}_{\text{Causality in Mean}} + \underbrace{\sqrt{\beta + \sum_{k} \mathbf{W}^{\sigma}_{ik} (x_{t-L}^k)^2}}_{\text{Causality in Variance}} \cdot \eta_t^i
    \label{eq:mixedphysics}
\end{equation}
where the innovations $\eta_t^i$ are independent standard Gaussian variables across targets and time, independent of the past, $L$ is a fixed time lag, and
$\beta>0$ is a baseline variance. The matrices $\mathbf{W}^{\mu}$ and $\mathbf{W}^{\sigma}$ represent the causal strengths for mean and volatility, respectively. A causal edge $j \to i$ is either a \textit{Mean Edge} ($\mathbf{W}^{\mu}_{ij} \neq 0, \mathbf{W}^{\sigma}_{ij} = 0$) or a \textit{Variance Edge} ($\mathbf{W}^{\mu}_{ij} = 0, \mathbf{W}^{\sigma}_{ij} > 0$). This distinction renders Variance-type edges invisible to standard MSE-based models. To analyze the sensitivity of our NLL objective to heteroscedastic dynamics, we vary the ratio of Variance-to-Mean edges across three regimes: 50:50, 75: 25, 100:0 (Pure Heteroscedastic control).

\textbf{Second-moment stability.}
Unconstrained coefficients can produce diverging second moments,
while artificially bounding the state can distort variance-driven
dependencies. We therefore rescale the coefficients and verify
stability without clipping the state.
Let $\mathbf{S}_t=\mathbb{E}[\mathbf{x}_t\mathbf{x}_t^\top]$.
Independence of the innovations gives the exact recursion
\begin{equation}
\begin{aligned}
    \mathbf{S}_t
    =\mathcal{T}(\mathbf{S}_{t-L})+\beta\mathbf{I},
    \qquad
    \mathcal{T}(\mathbf{S})
    =\mathbf{W}^{\mu}\mathbf{S}(\mathbf{W}^{\mu})^\top
    +\operatorname{Diag}\!\left(
        \mathbf{W}^{\sigma}\operatorname{diag}(\mathbf{S})
      \right),
\end{aligned}
\label{eq:mixedphysics_second_moment}
\end{equation}
where $\operatorname{diag}$ extracts a matrix's diagonal
and $\operatorname{Diag}$ constructs a diagonal matrix.
This expression retains cross moments between mean parents.
If $\rho(\mathcal{T})<1$, where $\rho$ denotes the spectral
radius of the linear operator, its repeated applications
decay and the accumulated innovation contributions converge.
The second moments therefore approach the finite fixed point
\begin{equation}
    \mathbf{S}_*
    =\beta\sum_{r=0}^{\infty}\mathcal{T}^{r}(\mathbf{I}),
    \qquad
    \mathbf{S}_*=\mathcal{T}(\mathbf{S}_*)+\beta\mathbf{I}.
\end{equation}
Rescaling $\mathbf{W}^{\mu}$ by $\sqrt{k}$ and
$\mathbf{W}^{\sigma}$ by $k$ scales $\mathcal{T}$ by $k$
without changing edge positions or mean/variance assignments.
After rescaling, the generated
configurations have spectral radii $\approx 0.9$ satisfying the stability condition.
Appendix~\ref{mixed physics appendix} provides the coefficient
sampling, rescaling procedure, and detailed derivation.

Synthetic data provides known ground truth for variance-driven
dependencies. In real data, financial connectedness indices
and neural coherence measures provide estimated rather than
known structure, so agreement with these measures would not
by itself establish causal correctness. Appendix \ref{realhetero} provides evidence of forecasting utility on real world dataset ETTh1: pruning to M2C-NLL's graph improves forecasting over M2C-MSE's graph. Appendix \ref{A} details all datasets and ground truths.



\begin{table}[ht]
\centering
\small
\caption{\textbf{Average AUROC on Lorenz--96 and VAR systems.}
Baselines are taken from \cite{Khanna2020Economy}, except those marked with $\dagger$ which were reproduced using the official implementations.}

    \label{tab: lorvar}
\resizebox{\textwidth}{!}{
\begin{tabular}{lcccccc}
\toprule
\textbf{Model} &
\multicolumn{4}{c}{\textbf{Lorenz--96}} &
\multicolumn{2}{c}{\textbf{VAR}} \\
\cmidrule(lr){2-5} \cmidrule(lr){6-7}
 & \multicolumn{2}{c}{$F=10$} & \multicolumn{2}{c}{$F=40$} & $T=500$ & $T=1000$ \\
\cmidrule(lr){2-3} \cmidrule(lr){4-5}
 & $T=250$ & $T=500$ & $T=250$ & $T=500$ &  &  \\
\midrule
cMLP(NGC)  
& $0.93 \pm 0.02$ & $0.96 \pm 0.03$
& $0.85 \pm 0.08$ & $0.94 \pm 0.03$
& $0.94 \pm 0.03$ & $0.93 \pm 0.02$ \\

cLSTM(NGC)
& $0.90 \pm 0.02$ & $0.95 \pm 0.05$
& $0.78 \pm 0.09$ & $0.90 \pm 0.05$
& $0.79 \pm 0.12$ & $0.80 \pm 0.09$ \\

TCDF 
& $0.70 \pm 0.01$ & $0.72 \pm 0.04$
& $0.62 \pm 0.01$ & $0.68 \pm 0.04$
& $0.77 \pm 0.07$ & $0.78 \pm 0.04$ \\

SRU  
& $0.84 \pm 0.03$ & $0.90 \pm 0.02$
& $\mathbf{1.00 \pm 0.00}$ & $\mathbf{1.00 \pm 0.00}$
& $0.82 \pm 0.06$ & $0.91 \pm 0.04$ \\

eSRU 
& $0.95 \pm 0.02$ & $0.98 \pm 0.01$
& $0.99 \pm 0.00$ & $\mathbf{1.00 \pm 0.00}$
& $0.93 \pm 0.05$ & $0.98 \pm 0.01$ \\

Causalformer$^\dagger$ 
& $0.68 \pm 0.09$ & $0.75 \pm 0.04$
& $0.85 \pm 0.05$ & $0.89 \pm 0.10$
& $0.66 \pm 0.07$ & $0.79 \pm 0.07$ \\

CUTS+$^\dagger$
& $0.77 \pm 0.06$ & $0.84 \pm 0.14$
& $0.94 \pm 0.02$ & $\mathbf{1.00 \pm 0.00}$
& $0.97 \pm 0.02$ & $0.99 \pm 0.005$ \\

AERCA$^\dagger$
& $0.68 \pm 0.01$ & $\mathbf{1.00 \pm 0.00}$
& $0.81 \pm 0.04$ & $0.97 \pm 0.01$
& $0.79 \pm 0.01$ & $0.91 \pm 0.02$ \\

UnCLe$^\dagger$
& $0.92 \pm 0.05$ & $0.87 \pm 0.02$
& $0.95 \pm 0.02$ & $0.93 \pm 0.01$
& $0.83 \pm 0.01$ & $0.93 \pm 0.05$ \\

\midrule
\textbf{M2C(MSE)}  
& $\mathbf{0.99 \pm 0.01}$ & $\mathbf{1.00 \pm 0.00}$
& $0.99 \pm 0.01$ & $\mathbf{1.00 \pm 0.00}$
& $\mathbf{1.00 \pm 0.00}$ & $\mathbf{1.00 \pm 0.00}$ \\

\textbf{M2C(NLL)}  
& $\mathbf{0.99 \pm 0.02}$ & $\mathbf{1.00 \pm 0.00}$
& $\mathbf{1.00 \pm 0.00}$ & $\mathbf{1.00 \pm 0.00}$
& $\mathbf{1.00 \pm 0.00}$ & $\mathbf{ 1.00 \pm 0.00}$ \\
\bottomrule
\end{tabular}
}
\end{table}


\begin{table*}[t]
\centering
\begin{minipage}[t]{0.45\linewidth}
    \vspace{0pt}
    \centering
    \small
    \setlength{\tabcolsep}{2.5pt}
    \caption{\textbf{Average AUROC on CausalTime dataset.} Baselines are taken from \cite{chengcausaltime}; $^\dagger$reproduced via official implementations.}
    \label{tab:causaltime}
    \resizebox{\linewidth}{!}{
    \begin{tabular}{lccc}
    \toprule
    \textbf{Method} & \textbf{Traffic} & \textbf{AQI} & \textbf{Medical} \\
    \midrule
    PCMCI  & 0.54 $\pm$ 0.07 & 0.53 $\pm$ 0.07 & 0.70 $\pm$ 0.01 \\
    NGC    & 0.60 $\pm$ 0.01 & 0.72 $\pm$ 0.01 & 0.57 $\pm$ 0.01 \\
    TCDF   & 0.50 $\pm$ 0.00 & 0.41 $\pm$ 0.02 & 0.63 $\pm$ 0.04 \\
    CUTS   & 0.62 $\pm$ 0.02 & 0.60 $\pm$ 0.00 & 0.37 $\pm$ 0.03 \\
    CUTS+  & 0.62 $\pm$ 0.08 & \textbf{0.89 $\pm$ 0.02} & 0.82 $\pm$ 0.02 \\
    LCCM   & 0.55 $\pm$ 0.03 & 0.86 $\pm$ 0.07 & 0.80 $\pm$ 0.02 \\
    eSRU   & 0.60 $\pm$ 0.02 & 0.82 $\pm$ 0.03 & 0.76 $\pm$ 0.04 \\
    SCGL   & 0.59 $\pm$ 0.06 & 0.49 $\pm$ 0.05 & 0.50 $\pm$ 0.02 \\
    AERCA$^\dagger$  & $0.65 \pm 0.01$ & $0.50 \pm 0.02$ & $0.59 \pm 0.01$ \\
    UnCLe$^\dagger$  & $0.60 \pm 0.01$ & $0.66 \pm 0.02$ & $0.64 \pm 0.04$ \\
\midrule
    \textbf{M2C(MSE)} & 0.68 $\pm$ 0.04 & \textbf{0.89 $\pm$ 0.03} & \textbf{0.90 $\pm$ 0.05} \\
    \textbf{M2C(NLL)} & \textbf{0.71 $\pm$ 0.02} & 0.84 $\pm$ 0.06 & 0.86 $\pm$ 0.03 \\
    \bottomrule
    \end{tabular}
    }
\end{minipage}\hfill
\begin{minipage}[t]{0.53\linewidth}
    \vspace{0pt}
    \centering
    \small
    \setlength{\tabcolsep}{2.5pt}
    \caption{\textbf{AUROC on DREAM3 dataset.} Baselines are taken from \cite{Khanna2020Economy}; $^\dagger$reproduced via official implementations.}
    \label{tab:dream}
    \resizebox{\linewidth}{!}{
    \begin{tabular}{lccccc}
    \toprule
    \textbf{Model} & \textbf{E.coli-1} & \textbf{E.coli-2} & \textbf{Yeast-1} & \textbf{Yeast-2} & \textbf{Yeast-3} \\
    \midrule
    cMLP(NGC)   & 0.644 & 0.568 & 0.585 & 0.506 & 0.528 \\
    cLSTM(NGC)  & 0.629 & 0.609 & 0.579 & 0.519 & 0.555 \\
    TCDF  & 0.614 & 0.647 & 0.581 & 0.556 & \textbf{0.557} \\
    SRU   & 0.657 & 0.666 & 0.617 & 0.575 & 0.550 \\
    eSRU  & 0.660 & 0.629 & 0.627 & 0.557 & 0.550 \\
    Causalformer $^\dagger$ & 0.420 & 0.261 & \textbf{0.637} & 0.483 & 0.484 \\
    CUTS+ $^\dagger$ & 0.585 & 0.563 & 0.547 & 0.532 & 0.509 \\
    AERCA $^\dagger$ & 0.668 & 0.616 & 0.621 & 0.567 & 0.465 \\
    UnCLe $^\dagger$ & 0.402 & 0.556 & 0.458 & 0.535 & \textbf{0.558} \\
\midrule
    \textbf{M2C(MSE)}   & 0.643 & 0.672 & \textbf{0.637} & 0.563 & 0.530 \\
    \textbf{M2C(NLL)}   & \textbf{0.672} & \textbf{0.703} & 0.605 & \textbf{0.578} & \textbf{0.557} \\
    \bottomrule
    \end{tabular}
    }
\end{minipage}
\end{table*}

\begin{table}[t]
\centering
\small
\caption{\textbf{Off-diagonal AUROC on Mixed Physics and RHINO synthetic data.} Mixed Physics results use five seeds; the mean-based reference assumes perfect ranking of mean edges and chance ranking of variance-only edges. RHINO settings use lag-only graphs.}
\label{tab:heteroscedastic}
\label{tab:mixed_physics}
\label{tab:rhino_synthetic}
\resizebox{\linewidth}{!}{%
\begin{tabular}{lccccccc}
\toprule
& \multicolumn{3}{c}{\textbf{Mixed Physics (variance:mean)}} & \multicolumn{4}{c}{\textbf{RHINO synthetic noise (History-dependent ER)}} \\
\cmidrule(lr){2-4}\cmidrule(lr){5-8}
\textbf{Method} & \textbf{50:50} & \textbf{75:25} & \textbf{100:0} & \textbf{N = 5} & \textbf{N = 10} & \textbf{N = 20} & \textbf{N = 40} \\
\midrule
\textit{Mean-based reference} & \textit{0.741} & \textit{0.630} & \textit{0.500} & -- & -- & -- & -- \\
\midrule
RHINO$^\dagger$ & $0.981 \pm 0.000$ & $0.982 \pm 0.000$ & $\mathbf{1.000 \pm 0.000}$ & $0.798 \pm 0.109$ & $0.839 \pm 0.060$ & $\mathbf{0.890 \pm 0.032}$& $0.901 \pm 0.024$ \\
TS-LSNM$^\dagger$ & $0.903 \pm 0.002$ & $0.863 \pm 0.008$ & $0.737 \pm 0.003$ & $0.687 \pm 0.111$ & $0.754 \pm 0.076$ & $0.752 \pm 0.039$ & $0.742 \pm 0.049$ \\
\textbf{M2C(MSE)} & $0.764 \pm 0.047$ & $0.633 \pm 0.049$ & $0.521 \pm 0.047$ & $0.674 \pm 0.068$ & $0.754 \pm 0.076$& $0.633 \pm 0.028$ & $0.592 \pm 0.039$ \\
\textbf{M2C(NLL)} & $\mathbf{0.998 \pm 0.003}$ & $\mathbf{0.997 \pm 0.005}$ & $\mathbf{1.000 \pm 0.000}$ & $\mathbf{0.885 \pm 0.072}$& $\mathbf{0.864 \pm 0.045}$ & $0.888 \pm 0.019$ & $\mathbf{0.914 \pm 0.016}$ \\
\bottomrule
\end{tabular}}
\end{table}

\subsection{Causal Discovery Results}

On Lorenz-96 and VAR (Table~\ref{tab: lorvar}), both M2C variants achieve near-perfect recovery. Several baselines also saturate at $N=10$; at $N=100$, M2C achieves 0.87--0.95 AUROC against 0.51--0.85 for the compared baselines (Appendix \ref{scaling}). M2C also leads on Traffic and Medical in CausalTime (Table~\ref{tab:causaltime}). On DREAM3 (Table~\ref{tab:dream}), M2C-NLL ranks first on three of five networks despite trajectories of only $T=21$. Its AUROC ranges from 0.557 to 0.703, while several baselines fall below chance on individual networks.

Mixed Physics (Table~\ref{tab:mixed_physics}) provides an idealized reference for the mean-prediction limitation in Section \ref{sec:functional_to_di}. Perfect ranking of mean edges and chance ranking of variance-only edges predicts AUROCs of 0.741, 0.630 and 0.500 from the respective edge counts (13/14, 7/20 and 0/27; derivation in Appendix~\ref{proof:auroc_reference}). M2C-MSE scores 0.764, 0.633 and 0.521, each within half a standard deviation of this reference. On the pure-variance Mixed Physics, M2C-NLL achieves 1.000 off-diagonal AUROC across five seeds. Given that Mixed Physics is a novel benchmark that separates mean- and variance-driven parents,
we also investigate whether M2C can distinguish them.
Appendix~\ref{app:typed_parent_recovery} shows that a two-graph extension
achieves strong parent-type separation with near-perfect AUROC. Appendix \ref{guide} discusses objective selection; Appendix \ref{C} reports sensitivity and hyperparameter studies.


On RHINO's synthetic benchmark (Table~\ref{tab:rhino_synthetic}), M2C-NLL exceeds M2C-MSE by 0.110--0.322 AUROC and has the highest mean at $N=5,10,40$; at $N=20$, it scores 0.888 versus RHINO's 0.890.

\textbf{Computational Analysis.} Mask2Cause scales linearly with history length and quadratically with variable count, with lower computational costs than the compared baselines in large-system settings. At $N=10,L=5$, it requires 2.05M FLOPs and approximately 0.10M parameters (Appendix \ref{D}).

\subsection{Downstream Utility}
M2C's graphs support forecasting and root cause localisation. On DREAM3, graph-guided pruning removes 59--99\% of forecasting-model parameters while reducing MSE by 31--41\% (Appendix \ref{forecasting}). On MSDS, M2C-NLL attains mean localisation accuracies of 0.446 at rank 1 and 0.995 within the top 5, supporting cloud microservice diagnosis (Appendix \ref{rca}).

\begin{table}[htbp]
\centering
\begin{minipage}[t]{0.42\linewidth}
\vspace{0pt}
\centering
\includegraphics[width=\linewidth]{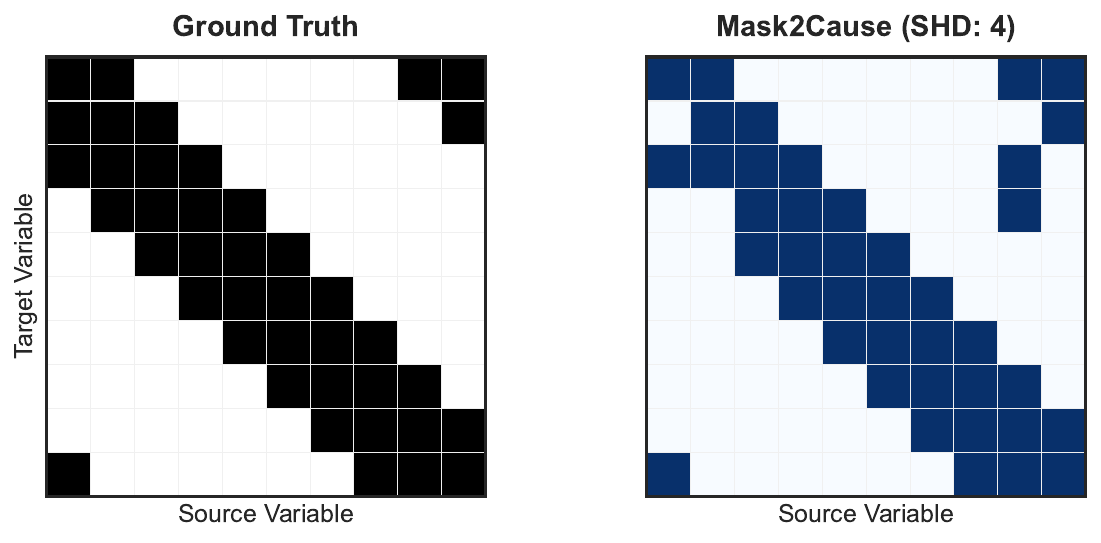}
\captionof{figure}{Ground truth and predicted causal graph for Lorenz-96 ($F=40$, $T=250$, AUROC = 0.98).}
\label{fig:lorenz_comparison}
\end{minipage}\hfill
\begin{minipage}[t]{0.55\linewidth}
\vspace{0pt}
\setlength{\tabcolsep}{4pt}

    \centering
    \small
    \caption{Mean ablation AUROC: M2C-MSE on CausalTime and M2C-NLL on Mixed Physics.}
    \label{tab:ablation}
    \resizebox{\linewidth}{!}{
        \begin{tabular}{lcc}
            \toprule
            \textbf{Variant} & \textbf{M2C-MSE} & \textbf{M2C-NLL} \\
             & \textbf{CausalTime} & \textbf{Mixed Physics} \\
            \midrule
            Multiplicative Mask & $0.74 \pm 0.06$ & $0.78 \pm 0.03$ \\
            Additive Bias Mask & $0.73 \pm 0.05$ & $0.75 \pm 0.01$ \\
            Layer-wise Masks & $0.66 \pm 0.04$ & $0.77 \pm 0.04$ \\
            Decoupled Heads & $0.61 \pm 0.05$ & $0.86 \pm 0.01$ \\
            Residual Pred & $0.78 \pm 0.11$ & $0.96 \pm 0.01$ \\
            \midrule
            \textbf{Full Model} & $\mathbf{0.82 \pm 0.09}$ & $\mathbf{0.99 \pm 0.01}$ \\
            \bottomrule
        \end{tabular}
        }
        \end{minipage}
\end{table}

\section{Ablation Study}
We ablate masking, the global adjacency, shared projection heads and prediction targets in both M2C variants (Table \ref{tab:ablation}). We retain variable tokenization because timestep attention lacks the $N\times N$ variable-pair structure required by $\hat{\mathbf{A}}$.

(1) \textbf{Alternative masking mechanisms:} We replace the log-barrier of Eq.~\ref{eq:score} with two variants that keep a learnable $\hat{\mathbf{A}}$ and end-to-end extraction and differ only in how it enters attention: an \textit{unnormalized multiplicative} gate $(\hat{\mathbf{A}} \odot \text{Softmax}(\mathbf{S}))\mathbf{V}$, and a \textit{bounded additive bias} $\alpha \hat{A}_{ij}$ with $\alpha$ tuned per dataset. The multiplicative gate loses 0.08 AUROC on CausalTime and 0.21 on Mixed Physics, and the additive bias 0.09 and 0.24. A renormalized multiplicative gate is absent because it is algebraically identical to ours. Appendix \ref{masking} derives that equivalence and shows that the log gate approximately cancels sigmoid attenuation in the pre-softmax score derivative when $\delta \ll \hat{A}_{ij} \ll 1$.\\
(2) \textbf{Layer-wise Masks (No Global $\hat{\mathbf{A}}$):} Each layer learns its own mask; averaging them after training to obtain the graph costs 0.16 AUROC on CausalTime and 0.22 on Mixed Physics.\\
(3) \textbf{Decoupled Heads:} We replace the shared linear projection with a separate head per variable. The shared projection performs better, suggesting that learning common dynamics (``shared physics'') reduces the parameter cost and overfitting risk of variable-specific functions.\\
(4) \textbf{Residual Prediction:} Predicting the raw state $\mathbf{x}_{t+1}$ rather than the increment $\Delta \mathbf{x}_{t+1}$ yields marginally better structural recovery.

\section{Conclusion}

Mask2Cause treats variables as tokens and constrains attention through a learned adjacency, recovering graphs in one shared model without component-wise networks or separate post-hoc extraction. Directed Information encompasses the mean-based criterion under history-independent additive noise; variance-only parents escape the population mean criterion while carrying positive DI. M2C detects these edges using log likelihood instead of squared error. M2C achieves competitive recovery across linear, chaotic, biological and real-data-derived benchmarks, with substantial computational savings in large-system settings. On the Mixed Physics benchmark we release, it recovers variance-driven structure that mean-based methods miss. The discovered graphs also transfer downstream: on DREAM3, graph-guided pruning removes 59--99\% of forecasting-model parameters while reducing MSE by 31--41\%, and on MSDS the graphs support accurate root cause localisation.
\endgroup

\bibliographystyle{unsrtnat}
\bibliography{references}

@article{karimi2010extensive,
  title={Extensive chaos in the Lorenz-96 model},
  author={Karimi, Alireza and Paul, Mark R},
  journal={Chaos: An interdisciplinary journal of nonlinear science},
  volume={20},
  number={4},
  year={2010},
  publisher={AIP Publishing}
}

@article{tank2021neural,
  title={Neural granger causality},
  author={Tank, Alex and Covert, Ian and Foti, Nicholas and Shojaie, Ali and Fox, Emily B},
  journal={IEEE Transactions on Pattern Analysis and Machine Intelligence},
  volume={44},
  number={8},
  pages={4267--4279},
  year={2021},
  publisher={IEEE}
}

@article{nauta2019causal,
  title={Causal discovery with attention-based convolutional neural networks},
  author={Nauta, Meike and Bucur, Doina and Seifert, Christin},
  journal={Machine Learning and Knowledge Extraction},
  volume={1},
  number={1},
  pages={19},
  year={2019},
  publisher={MDPI}
}

@inproceedings{cheng2023cuts,
title={{CUTS}: Neural Causal Discovery from Irregular Time-Series Data},
author={Yuxiao Cheng and Runzhao Yang and Tingxiong Xiao and Zongren Li and Jinli Suo and Kunlun He and Qionghai Dai},
booktitle={The Eleventh International Conference on Learning Representations },
year={2023}
}

@inproceedings{cheng2024cuts+,
  title={Cuts+: High-dimensional causal discovery from irregular time-series},
  author={Cheng, Yuxiao and Li, Lianglong and Xiao, Tingxiong and Li, Zongren and Suo, Jinli and He, Kunlun and Dai, Qionghai},
  booktitle={Proceedings of the AAAI Conference on Artificial Intelligence},
  volume={38},
  pages={11525--11533},
  year={2024}
}

@article{kong2024causalformer,
  title={Causalformer: An interpretable transformer for temporal causal discovery},
  author={Kong, Lingbai and Li, Wengen and Yang, Hanchen and Zhang, Yichao and Guan, Jihong and Zhou, Shuigeng},
  journal={IEEE Transactions on Knowledge and Data Engineering},
  year={2024},
  publisher={IEEE}
}

@inproceedings{
Khanna2020Economy,
title={Economy Statistical Recurrent Units For Inferring Nonlinear Granger Causality},
author={Saurabh Khanna and Vincent Y. F. Tan},
booktitle={International Conference on Learning Representations},
year={2020},
url={https://openreview.net/forum?id=SyxV9ANFDH}
}

@inproceedings{pamfil2020dynotears,
  title={Dynotears: Structure learning from time-series data},
  author={Pamfil, Roxana and Sriwattanaworachai, Nisara and Desai, Shaan and Pilgerstorfer, Philip and Georgatzis, Konstantinos and Beaumont, Paul and Aragam, Bryon},
  booktitle={International conference on artificial intelligence and statistics},
  pages={1595--1605},
  year={2020},
  organization={Pmlr}
}

@article{runge2019detecting,
  title={Detecting and quantifying causal associations in large nonlinear time series datasets},
  author={Runge, Jakob and Nowack, Peer and Kretschmer, Marlene and Flaxman, Seth and Sejdinovic, Dino},
  journal={Science advances},
  volume={5},
  number={11},
  pages={eaau4996},
  year={2019},
  publisher={American Association for the Advancement of Science}
}

@inproceedings{liuitransformer,
  title={itransformer: Inverted transformers are effective for time series forecasting},
  author={Liu, Yong and Hu, Tengge and Zhang, Haoran and Wu, Haixu and Wang, Shiyu and Ma, Lintao and Long, Mingsheng},
  booktitle={International conference on learning representations},
  volume={2024},
  pages={11116--11140},
  year={2024}
}

@inproceedings{chengcausaltime,
  title={Causaltime: Realistically generated time-series for benchmarking of causal discovery},
  author={Cheng, Yuxiao and Wang, Ziqian and Xiao, Tingxiong and Zhong, Qin and Suo, Jinli and He, Kunlun},
  booktitle={International Conference on Learning Representations},
  volume={2024},
  pages={2923--2945},
  year={2024}
}

@article{zhou2024jacobian,
  title={Jacobian regularizer-based neural granger causality},
  author={Zhou, Wanqi and Bai, Shuanghao and Yu, Shujian and Zhao, Qibin and Chen, Badong},
  journal={arXiv preprint arXiv:2405.08779},
  year={2024}
}

@article{granger1969investigating,
  title={Investigating causal relations by econometric models and cross-spectral methods},
  author={Granger, Clive WJ},
  journal={Econometrica: journal of the Econometric Society},
  pages={424--438},
  year={1969},
  publisher={JSTOR}
}

@article{quinn2015directed,
  title={Directed information graphs},
  author={Quinn, Christopher J and Kiyavash, Negar and Coleman, Todd P},
  journal={IEEE Transactions on information theory},
  volume={61},
  number={12},
  pages={6887--6909},
  year={2015},
  publisher={IEEE}
}

@article{diebold2009measuring,
  title={Measuring financial asset return and volatility spillovers, with application to global equity markets},
  author={Diebold, Francis X and Yilmaz, Kamil},
  journal={The Economic Journal},
  volume={119},
  number={534},
  pages={158--171},
  year={2009},
  publisher={Oxford University Press Oxford, UK}
}

@inproceedings{zhou2021informer,
  title={Informer: Beyond efficient transformer for long sequence time-series forecasting},
  author={Zhou, Haoyi and Zhang, Shanghang and Peng, Jieqi and Zhang, Shuai and Li, Jianxin and Xiong, Hui and Zhang, Wancai},
  booktitle={Proceedings of the AAAI conference on artificial intelligence},
  volume={35},
  number={12},
  pages={11106--11115},
  year={2021}
}

@article{wu2021autoformer,
  title={Autoformer: Decomposition transformers with auto-correlation for long-term series forecasting},
  author={Wu, Haixu and Xu, Jiehui and Wang, Jianmin and Long, Mingsheng},
  journal={Advances in neural information processing systems},
  volume={34},
  pages={22419--22430},
  year={2021}
}

@inproceedings{de2020latent,
  title={Latent convergent cross mapping},
  author={De Brouwer, Edward and Arany, Adam and Simm, Jaak and Moreau, Yves},
  booktitle={International Conference on Learning Representations},
  year={2020}
}

@article{bellotneural,
  title={Neural graphical modelling in continuous-time: consistency guarantees and algorithms},
  author={Bellot, Alexis and Branson, Kim and van der Schaar, Mihaela},
  journal={arXiv preprint arXiv:2105.02522},
  year={2021}
}

@article{bi2025unclescalabledynamiccausal,
  title={Uncle: Towards scalable dynamic causal discovery in non-linear temporal systems},
  author={Bi, Tingzhu and Pan, Yicheng and Jiang, Xinrui and Sun, Huize and Ma, Meng and Wang, Ping},
  journal={Advances in Neural Information Processing Systems},
  volume={38},
  pages={96158--96187},
  year={2026}
}

@inproceedings{liu2025,
  title={Causal discovery with inverted self-attention for multivariate time series},
  author={Liu, Yusen and Wang, Yong and Yin, Yifan and Zhu, Tianqing and Liu, Xiufeng and Huo, Huan},
  booktitle={Pacific-Asia Conference on Knowledge Discovery and Data Mining},
  pages={167--179},
  year={2025},
  organization={Springer}
}

@article{prill2010towards,
  title={Towards a rigorous assessment of systems biology models: the DREAM3 challenges},
  author={Prill, Robert J and Marbach, Daniel and Saez-Rodriguez, Julio and Sorger, Peter K and Alexopoulos, Leonidas G and Xue, Xiaowei and Clarke, Neil D and Altan-Bonnet, Gregoire and Stolovitzky, Gustavo},
  journal={PloS one},
  volume={5},
  number={2},
  pages={e9202},
  year={2010},
  publisher={Public Library of Science San Francisco, USA}
}

@article{hyvarinen2010estimation,
  title={Estimation of a structural vector autoregression model using non-Gaussianity.},
  author={Hyv{\"a}rinen, Aapo and Zhang, Kun and Shimizu, Shohei and Hoyer, Patrik O},
  journal={Journal of Machine Learning Research},
  volume={11},
  number={5},
  year={2010}
}

@article{peters2013causal,
  title={Causal inference on time series using restricted structural equation models},
  author={Peters, Jonas and Janzing, Dominik and Sch{\"o}lkopf, Bernhard},
  journal={Advances in neural information processing systems},
  volume={26},
  year={2013}
}

@inproceedings{lozano2009grouped,
  title={Grouped graphical Granger modeling methods for temporal causal modeling},
  author={Lozano, Aurelie C and Abe, Naoki and Liu, Yan and Rosset, Saharon},
  booktitle={Proceedings of the 15th ACM SIGKDD international conference on Knowledge discovery and data mining},
  pages={577--586},
  year={2009}
}

@book{lutkepohl2005new,
  title={New introduction to multiple time series analysis},
  author={L{\"u}tkepohl, Helmut and others},
  volume={1},
  year={2005},
  publisher={springer Berlin}
}

@inproceedings{han2025aerca,
  title={Root cause analysis of anomalies in multivariate time series through granger causal discovery},
  author={Han, Xiao and Absar, Saima and Zhang, Lu and Yuan, Shuhan},
  booktitle={International Conference on Learning Representations},
  volume={2025},
  pages={45137--45153},
  year={2025}
}

@article{liu2025kolmogorov,
  title={Kolmogorov-Arnold Networks for time series Granger causality inference},
  author={Liu, Meiliang and Xu, Yunfang and Li, Zijin and Si, Zhengye and Yang, Xiaoxiao and Yang, Xinyue and Zhao, Zhiwen},
  journal={arXiv preprint arXiv:2501.08958},
  year={2025}
}

@inproceedings{liu2026llm,
  title={Llm-gc: Advancing granger causal discovery from time series with multimodel language modeling},
  author={Liu, Bo and Li, Hongyan and Hong, Shenda},
  booktitle={Proceedings of the Nineteenth ACM International Conference on Web Search and Data Mining},
  pages={387--395},
  year={2026}
}

@article{reisach2021beware,
  title={Beware of the simulated dag! causal discovery benchmarks may be easy to game},
  author={Reisach, Alexander and Seiler, Christof and Weichwald, Sebastian},
  journal={Advances in Neural Information Processing Systems},
  volume={34},
  pages={27772--27784},
  year={2021}
}

@inproceedings{nedelkoski2020multi,
  title={Multi-source distributed system data for ai-powered analytics},
  author={Nedelkoski, Sasho and Bogatinovski, Jasmin and Mandapati, Ajay Kumar and Becker, Soeren and Cardoso, Jorge and Kao, Odej},
  booktitle={European conference on service-oriented and cloud computing},
  pages={161--176},
  year={2020},
  organization={Springer}
}

@article{gemici2021causality,
  title={Causality-in-mean and causality-in-variance among Bitcoin, Litecoin, and Ethereum},
  author={Gemici, Eray and Polat, M{\"u}sl{\"u}m},
  journal={Studies in Economics and Finance},
  volume={38},
  number={4},
  pages={861--872},
  year={2021},
  publisher={Emerald Publishing Limited}
}

@article{sosnik2016noise,
  title={Noise modulation in retinoic acid signaling sharpens segmental boundaries of gene expression in the embryonic zebrafish hindbrain},
  author={Sosnik, Julian and Zheng, Likun and Rackauckas, Christopher V and Digman, Michelle and Gratton, Enrico and Nie, Qing and Schilling, Thomas F},
  journal={Elife},
  volume={5},
  pages={e14034},
  year={2016},
  publisher={eLife Sciences Publications, Ltd}
}

@article{cheung1996causality,
  title={A causality-in-variance test and its application to financial market prices},
  author={Cheung, Yin-Wong and Ng, Lilian K},
  journal={Journal of econometrics},
  volume={72},
  number={1-2},
  pages={33--48},
  year={1996},
  publisher={Elsevier}
}

@article{hafner2008testing,
  title={Testing for causality in variance using multivariate GARCH models},
  author={Hafner, Christian M and Herwartz, Helmut},
  journal={Annales d'Economie et de Statistique},
  pages={215--241},
  year={2008},
  publisher={JSTOR}
}

@article{luo2011granger,
  title={Granger causality with signal-dependent noise},
  author={Luo, Qiang and Ge, Tian and Feng, Jianfeng},
  journal={Neuroimage},
  volume={57},
  number={4},
  pages={1422--1429},
  year={2011},
  publisher={Elsevier}
}

@inproceedings{kikuchi2023structure,
  title={Structure learning for groups of variables in nonlinear time-series data with location-scale noise},
  author={Kikuchi, Genta and Shimizu, Shohei},
  booktitle={Causal Analysis Workshop Series},
  pages={20--39},
  year={2023},
  organization={PMLR}
}

@article{gong2022rhino,
  title={Rhino: Deep causal temporal relationship learning with history-dependent noise},
  author={Gong, Wenbo and Jennings, Joel and Zhang, Cheng and Pawlowski, Nick},
  journal={arXiv preprint arXiv:2210.14706},
  year={2022}
}

@article{rahmani2026flow,
  title={Flow based approach for Dynamic Temporal Causal models with non-Gaussian or Heteroscedastic Noises},
  author={Rahmani, Abdellah and Frossard, Pascal},
  journal={Advances in Neural Information Processing Systems},
  volume={38},
  pages={8300--8349},
  year={2026}
}

\newpage
\appendix
\section{Notation Table}
\label{app:notation}

Table \ref{tab:notation} summarizes the primary mathematical notation used throughout the main text and appendices. 

\begingroup
\small
\renewcommand{\arraystretch}{1.15}
\begin{longtable}{@{}p{3.5cm}p{10cm}@{}}
\caption{Summary of notation}\label{tab:notation}\\
\toprule
\textbf{Symbol} & \textbf{Description} \\
\midrule
\endfirsthead
\toprule
\textbf{Symbol} & \textbf{Description} \\
\midrule
\endhead
\multicolumn{2}{l}{\textit{System and Data}} \\
\midrule
$N$& Number of variables in the multivariate system \\
$T$ & Total number of observed time steps \\
$X$ & The full multivariate time series matrix $\in \mathbb{R}^{T \times N}$ \\
$x_t$ & The state of all variables at time $t$, $x_t \in \mathbb{R}^N$ \\
$x_t^i$ & The scalar value of variable $i$ at time $t$ \\
$L$ & Look-back window length (historical lag) \\
$x_{<t}$ & Complete history of the system through time $t-1$\\
$x_{<t}^{-j}$ & Complete history through time $t-1$, excluding variable $j$\\
$x_{t-L:t-1}$, $x^i_{t-L:t-1}$ & The $L$ most recent observations of the system and of variable $i$, respectively \\

\midrule
\multicolumn{2}{l}{\textit{Causal Graph \& Information Theory}} \\
\midrule
$G$ & The underlying directed Granger-causal graph \\
$A$ & Ground-truth binary adjacency matrix $\in \{0, 1\}^{N \times N}$; $A_{ij}=1$ denotes $j\to i$ \\
$\hat{A}$ & Learned continuous adjacency matrix (attention gates and reported edge scores) $\in [0, 1]^{N \times N}$ \\
$\Theta$ & Global learnable adjacency parameter matrix $\in \mathbb{R}^{N \times N}$ \\
$I(X^j \rightarrow X^i \Vert X^{-\{i,j\}})$ & Causally Conditioned Directed Information from $X^j$ to $X^i$ \\
$P(i)$ & Predicted set of causal parents for target variable $i$ \\
$S$ & Candidate set of source variables for target $i$, excluding $i$; the target's own history is always included separately \\
$p_{\mathrm{full},t}$ & True conditional density of $x_t^i$ given all variables' histories: $p_{\mathrm{full},t}(u)=p(x_t^i=u\mid x_{<t})$ \\
$p_{-j,t}$ & True conditional density of $x_t^i$ given every observed history except source $j$'s: $p_{-j,t}(u)=p(x_t^i=u\mid x_{<t}^{-j})$. This still conditions on the target's own history ($j\ne i$) \\
$\mathcal L_{\mathrm{own}}$ & Cumulative log-loss of the ideal predictor using only target $i$'s own past: $-\sum_{t=1}^T\log p(x_t^i\mid x_{<t}^i)$. Each density is the true conditional density evaluated at the observed target value \\
$\mathcal L_S$ & Cumulative log-loss of the ideal predictor using target $i$'s own past and the histories of sources in $S$: $-\sum_{t=1}^T\log p(x_t^i\mid x_{<t}^i,x_{<t}^S)$. Its expectation is the minimum expected cumulative log-loss with access to these histories \\
\midrule
\multicolumn{2}{l}{\textit{Architecture and Optimization}} \\
\midrule
$d$ & Latent embedding dimension for variable tokens \\
$d_k$ & Query/key dimension per attention head \\
$\mathbf e_i$ & Projected history token for variable $i$, before adding its identity embedding \\
$\mathbf W_{\mathrm{emb}},\mathbf b_{\mathrm{emb}}$ & Shared history-projection weights and bias; $\mathbf W_{\mathrm{emb}}\in\mathbb R^{d\times L}$ \\
$E_{id}$ & Variable identity embedding matrix $\in \mathbb{R}^{N \times d}$ \\
$M$ & Number of Transformer encoder layers \\
$Z^{(m)}$ & Latent token representation after the $m$-th encoder layer \\
$Q, K, V$ & Query, Key, and Value matrices used in the masked attention mechanism \\
$S_{ij},\mathrm{Score}_{ij}$ & Attention logit from source token $j$ to target token $i$, before and after adding the adjacency mask \\
$\mathbf I,\sigma(\cdot)$ & Identity matrix and elementwise sigmoid function \\
$\gamma$ & Diagonal forcing constant for self-loops \\
$\delta$ & Positive numerical floor, used in the logarithmic attention mask and the predicted variance \\
$\lambda$ & Sparsity penalty hyperparameter for the $L_1$ regularization \\
$\hat{\mu}_{t}^i$ & Predicted conditional mean for variable $i$ at time $t$ \\
$(\hat{\sigma}_{t}^i)^2$ & Predicted conditional variance for variable $i$ at time $t$ \\
$\mathbf w_\mu,\mathbf b_\mu;\ \mathbf w_\sigma,\mathbf b_\sigma$ & Shared linear-head parameters for the conditional mean and the pre-Softplus variance output \\
$\mathbf w_{\mathrm{out}},\mathbf b_{\mathrm{out}}$ & Shared linear-head parameters for the mean-only prediction \\
$\mathcal B$ & Set of target-time indices in the minibatch \\
$\mathcal{L}_{NLL}$, $\mathcal{L}_{MSE}$ & Negative Log-Likelihood and Mean Squared Error objective functions \\
$\mathcal L_{\mathrm{pred}},\mathcal L$ & Selected prediction loss (NLL or MSE), and the total objective including the adjacency sparsity penalty \\
\bottomrule
\end{longtable}
\endgroup

\section{Detailed Review of Related Works}
\label{related}

\textbf{Causal Discovery Approaches.} 
Causal discovery in time series is formalized through the following families of approaches:
\begin{itemize}
    \item \textbf{Constraint-based approaches} depend upon conditional independence tests to construct a skeletal graph, which is subsequently oriented using temporal priority and faithfulness rules. Methods such as PCMCI (\cite{runge2019detecting}) reduce the conditioning sets through preliminary selection, but computational cost and statistical power remain sensitive to the number of variables, maximum lag, graph density, and conditioning dimension. Pairwise candidate links grow quadratically in the variable count; exhaustive searches over conditioning subsets can grow exponentially.
    \item \textbf{Score-based approaches} like DYNOTEARS (\cite{pamfil2020dynotears}) formulate causal discovery as a continuous optimization problem by minimizing a regularized loss function that balances data likelihood against network complexity. The number of coefficients grows quadratically with the variable count and linearly with the maximum lag, while enforcing instantaneous acyclicity adds optimization overhead. These costs can become substantial in high dimensions; the linear parameterization also limits the nonlinear dependencies that DYNOTEARS can represent.
    \item \textbf{Noise-based approaches} achieve structural identifiability by exploiting statistical asymmetries in residual noise distributions under strict semi-parametric conditions, such as non-Gaussian linear or nonlinear additive noise models.
    \item \textbf{Granger causality} infers a causal link $X^p \to X^q$ if the inclusion of $X^p$'s history strictly reduces the predictive error of an autoregressive model for $X^q$. 
\end{itemize}


\textbf{Linear Granger Models.} These causal discovery models assumed linear dynamics and utilized Vector Autoregressive (VAR) models of order $P$ (\cite{lozano2009grouped}, \cite{lutkepohl2005new}). In this regime, the system dynamics for $N$ variables are governed by:
\begin{equation}
    x_t = \sum_{k=1}^{P} \mathbf{A}^{(k)} x_{t-k} + \epsilon_t
\end{equation}
where $x_t \in \mathbb{R}^N$ is the observation vector at time $t$, $\mathbf{A}^{(k)} \in \mathbb{R}^{N \times N}$ are the lag-specific coefficient matrices, and $\epsilon_t$ is the noise term. The causal structure is determined by aggregating these coefficients into an adjacency matrix $\mathbf{A}$, where a variable $X^j$ is identified as a Granger-cause of $X^i$ if there exists at least one $k \in \{1, \dots, P\}$ such that $A_{ij}^{(k)} \neq 0$. To handle high-dimensional data, these methods often employ group-sparse regularization to recover the underlying causal graph.

\textbf{Non-linear Component-wise Models.} These include models such as eSRU (\cite{Khanna2020Economy}), TCDF (\cite{nauta2019causal}) and NGC (\cite{tank2021neural}). They extract causal information by training a separate model for each variable and then inferring the causal structure from the learned model through various techniques. For instance Neural Granger Causality (cMLP) analyzes weights of the first hidden layer to interpret inter-variable influence. TCDF, on the other hand, utilizes Temporal Convolutional Networks and applies a post-hoc attention interpretation method to its learned convolutional kernels to identify causal relationships. As mentioned before, these incur heavy computational overhead for high-dimensional data.

\textbf{Non-linear Shared Weight Models.} Models like CUTS+ (\cite{cheng2024cuts+}) utilize message-passing Graph Neural Networks (GNNs) for causal discovery in irregularly sampled time-series. Learning a large graph while performing graph-conditioned prediction creates computational challenges. CUTS+ addresses these scalability bottlenecks through shared message-passing weights and a Coarse-to-Fine discovery strategy that progressively refines the graph. Other shared-weight methods, such as JRNGC (\cite{zhou2024jacobian}) and UnCLe (\cite{bi2025unclescalabledynamiccausal}), utilize standard sequence forecasters rather than GNNs. However, they rely heavily on post-hoc analysis - such as evaluating the expected input-output Jacobian matrix over the full dataset (JRNGC) or measuring prediction error spikes after chronological data permutation (UnCLe) - to extract the causal graph after forecaster's convergence. This decoupling of representation learning from graph extraction causes the underlying models to optimize strictly for unconstrained prediction rather than structural sparsity, increasing the risk of overfitting to spurious correlations.

\textbf{Transformer-based Models.} CausalFormer (\cite{kong2024causalformer}) and CSAM (\cite{liu2025}) perform post-hoc analysis in the form of Regression Relevance Propagation and statistical verification modules respectively. Hence, they inherit the risk of overfitting to spurious correlations as well. In CSAM the attention scores stay the primary object, and edges are read off them after training by a separate global causal algorithm and a verification module, whereas in Mask2Cause the adjacency matrix is itself a parameter that gates attention and receives the sparsity penalty, so the tensor we optimize is the one we report. Standard Transformer architectures face another issue - they tokenize the time-series data across time steps which mixes information across variables. This makes it challenging to isolate causal influences among variables. iTransformer (\cite{liuitransformer}) provides a way to circumvent this by introducing inverted embedding, i.e. inverting the embedding axis to treat variables rather than time steps as tokens. This helps in preserving identifiability required for causal inference while also retaining the parameter efficiency of a Transformer architecture.

\textbf{Recent Neural Granger Methods (2025--26).} AERCA (\cite{han2025aerca}) treats anomalies as interventions on exogenous variables, jointly learning Granger structure and the normal-condition distribution of those exogenous variables, and reads its graph out of per-window coefficient tensors by taking an absolute value, a median over time, a maximum over lags and a quantile threshold after training. UnCLe (\cite{bi2025unclescalabledynamiccausal}) pairs an Uncoupler and a Recoupler network and estimates dynamic causal influence from datapoint-wise prediction errors induced by temporal perturbations, which makes it one of the few methods in this comparison that targets a time-resolved rather than a single summary graph. KANGCI (\cite{liu2025kolmogorov}) replaces the MLP predictors of the component-wise family with Kolmogorov-Arnold Networks, extracting base weights from the KAN layers under sparsity and ridge penalties, and adds a time-reversed Granger criterion to suppress spurious links. LLM-GC (\cite{liu2026llm}) is a multimodal approach that supplements the temporal signal with semantic priors retrieved from a large language model, which places it outside the purely observational setting the remaining methods share. All four estimate edges through quantities that depend on the data only through the conditional mean, which is the property our heteroscedastic experiments isolate.

\textbf{Neural Discovery with Heteroscedastic Noise.} TS-LSNM \citep{kikuchi2023structure} learns neural location and scale functions for each target and extracts lagged and contemporaneous relationships from their weights, with optional constraints on variable groups. M2C instead shares its attention encoder across targets, avoiding per-target network replication, and uses one sparsity-regularized adjacency to control both mean and variance prediction.

RHINO \citep{gong2022rhino} combines nonlinear temporal mechanisms, conditional normalizing flows for history-dependent noise, and variational inference over lagged and contemporaneous graphs. M2C directly learns the stationary summary graph through an adjacency-constrained attention encoder. Its Gaussian mean and variance heads provide a straightforward likelihood objective without a conditional flow or variational graph posterior, and the same adjacency gates prediction and supplies the reported edge scores.

FANTOM \citep{rahmani2026flow} uses conditional normalizing flows and variational Bayesian expectation-maximization to learn temporal causal models with non-Gaussian or heteroscedastic noise, including regime-specific mechanisms and graphs. For stationary summary-graph discovery, M2C uses a single shared encoder and adjacency without regime segmentation or alternating variational inference. This design accommodates squared-error and Gaussian-likelihood training within the same architecture, enabling controlled comparisons of mean- and variance-driven discovery.



Table \ref{tab:method_comparison} gives a brief comparative analysis of the methodologies used in closely related works against \textbf{Mask2Cause}. 
\begingroup
\small
\setlength{\tabcolsep}{4pt}
\renewcommand{\arraystretch}{1.15}
\begin{longtable}{@{}>{\raggedright\arraybackslash}p{2.3cm}>{\raggedright\arraybackslash}p{2.3cm}>{\raggedright\arraybackslash}p{4.7cm}>{\raggedright\arraybackslash}p{3.5cm}@{}}
\caption{Comparison of prediction models and graph-learning mechanisms in related temporal causal discovery methods.}\label{tab:method_comparison}\\
\toprule
\textbf{Method} & \textbf{Parameter sharing} & \textbf{Graph learning and extraction} & \textbf{Prediction model} \\
\midrule
\endfirsthead
\toprule
\textbf{Method} & \textbf{Parameter sharing} & \textbf{Graph learning and extraction} & \textbf{Prediction model} \\
\midrule
\endhead
NGC \citep{tank2021neural} & Target-specific networks & Sparse input-weight groups; graph derived from their norms & MLP or LSTM mean predictor \\
\addlinespace
CUTS+ \citep{cheng2024cuts+} & Shared message-passing weights & Learned graph with coarse-to-fine discovery & Graph-based prediction for irregular series \\
\addlinespace
JRNGC \citep{zhou2024jacobian} & Shared predictor & Jacobian regularization during training; Jacobian-based edge scores & Joint mean prediction \\
\addlinespace
UnCLe \citep{bi2025unclescalabledynamiccausal} & Shared temporal networks & Sparse dependency matrices; aggregation for static graphs or temporal perturbation for dynamic graphs & Reconstruction and mean prediction \\
\addlinespace
CausalFormer \citep{kong2024causalformer} & Joint Transformer & Regression relevance propagation through the trained model & Variable-history embeddings, causal convolutions and attention \\
\addlinespace
CSAM \citep{liu2025} & Joint attention model & Sparse attention, global graph extraction and causal verification & Inverted variable-history tokens \\
\addlinespace
TS-LSNM \citep{kikuchi2023structure} & Target-specific location/scale networks & Weight-derived mean and variance adjacencies, combined across mechanisms & Neural location--scale model \\
\addlinespace
RHINO \citep{gong2022rhino} & Shared neural mechanisms & Variational inference over instantaneous and lag-specific edges & Nonlinear mechanisms with conditional-flow noise \\
\addlinespace
FANTOM \citep{rahmani2026flow} & Neural mechanisms within regimes & Variational graph inference with Bayesian EM for regime assignments & Regime-specific mechanisms with conditional-flow noise \\
\addlinespace
\textbf{M2C (ours)} & \textbf{Shared attention encoder and heads} & \textbf{One learned adjacency gates attention, receives sparsity regularization, and supplies edge scores} & \textbf{Mean-only or Gaussian mean/variance prediction} \\
\bottomrule
\end{longtable}
\endgroup

\section{Theoretical Assumptions for Causal Identifiability}
\label{assume_app}

Our framework relies on five standard assumptions for our model to discover discover the Ground Truth and for it to be unique.

\textbf{1. Strict Positivity (Non-Determinism):} 
We assume the true joint probability density of the system is strictly positive over the entire state space $\mathcal{X}$:
\begin{equation}
    P(\mathbf{x}) > 0 \quad \forall \mathbf{x} \in \mathcal{X}
\end{equation}
\textit{Necessity:} This condition guarantees that no variable is a strictly deterministic, noiseless function of another. In purely deterministic systems, information becomes redundant, creating unresolvable causal symmetries. For example, suppose variable $Z$ is a deterministic copy of $X$ ($Z_t = X_{t-1}$), and $Y$ depends on $Z$ with some noise ($Y_t = Z_{t-1} + \epsilon_t$). Because $Z_{t-1}$ is mathematically indistinguishable from $X_{t-2}$, substituting $X$ for $Z$ yields the exact same predictive loss. Consequently, an optimizer cannot distinguish between the causal chain $X \rightarrow Z \rightarrow Y$ and an alternative graph containing the direct edge $X \rightarrow Y$. This symmetry renders the minimal generative causal graph non-unique. Positivity ensures that every variable possesses at least a marginal amount of independent noise, breaking these symmetries and ensuring that the minimal generative causal graph $\mathcal{G}$ is unique.

\textbf{2. Causal Sufficiency:} 
We assume there are no unobserved (hidden) confounding variables that simultaneously influence two or more observed variables in our system $\mathbf{X}$. 

\textit{Necessity:} If a hidden confounder exists, the network's optimizer will observe a spurious statistical correlation between the variables and hallucinate a direct causal edge to minimize the forecasting loss. Sufficiency ensures all inferred edges represent true causal mechanisms within the observed system.

\textbf{3. Strict Temporal Precedence:} 
We assume that causal influences strictly take time to propagate, precluding instantaneous (intra-step) causal effects. Each variable
is generated from the system's past and its own innovation,
with contemporaneous innovations mutually independent conditional
on the complete past.

\textit{Necessity:} Under this assumption, the state of variable $i$ at time $t$ is strictly determined by historical states and independent noise, rather than the concurrent states of other variables. This justifies our model's autoregressive design, where predictions for $\mathbf{x}^i_t$ are conditioned exclusively on the strictly past window $\mathbf{x}_{<t}$. The absence of instantaneous effects and the conditional
independence of contemporaneous innovations imply the
strictly causal factorization
    \[
        p(x_t \mid x_{<t})
        =
        \prod_{i=1}^{N} p(x_t^i \mid x_{<t}).
    \]
    Under positivity and this factorization, the directed information
    graph equals the unique minimal generative model graph
    \citep[Theorem~1]{quinn2015directed}.

\textbf{4. Causal Faithfulness:}
We assume the observed distribution is faithful to the time-unrolled causal graph: every conditional independence in the distribution is implied by the graph.
\textit{Necessity:} This excludes parameter cancellations that conceal structural dependencies, supporting the interpretation of predictive relationships in terms of the underlying causal mechanisms.

\textbf{5. Stationarity and Finite Markov Order:} 
We assume the process is strictly stationary, with time-invariant graph topology, and that the conditional transition probabilities satisfy a finite-order Markov property bounded by our look-back window $L$:
\begin{equation}
    P(\mathbf{x}_t \mid \mathbf{x}_{0:t-1}) = P(\mathbf{x}_t \mid \mathbf{x}_{t-L:t-1})
\end{equation}
\textit{Necessity:} This guarantees that the complete causal footprint required to predict the next state is fully contained within the finite context window of the Transformer. Our model also assumes a stationary causal graph while trying to infer it. 

\section{Behaviour Under Assumption Violations}
\label{violations}

Section 2.1 states the assumptions under which Mask2Cause is formulated. Causal sufficiency and a stationary graph are the two that real systems most often break, so this appendix characterises what happens when they do. Both suites are built as offshoots of our existing benchmarks so that the numbers stay comparable to the main results.

\subsection{Hidden Confounding}

Causal sufficiency is a classical assumption of the entire Granger family, so the question is how resilient M2C is when its assumption is violated. We add 10 observed variables plus one hidden AR(1) hub driving 4 of them. The model never sees the hub and is scored against the true $10 \times 10$ sub-graph. We measure whether it hallucinates edges specifically among the confounded children. The hub-to-children coupling strength was calibrated per system so the resulting hub-child correlation matched the same order of magnitude as the system's own true parent-child correlations.

\begin{table}[htbp]
\centering
\small
\caption{\textbf{Hidden confounding.} Inflation ratio compares the scores of false edges among confounded children against other non-edges.}
\label{tab:confounding}
\begin{tabular}{llcc}
\toprule
\textbf{Family} & \textbf{Objective} & \textbf{AUROC} & \textbf{Inflation ratio (children vs other non-edges)} \\
\midrule
VAR & MSE & 0.795 & 1.22 \\
VAR & NLL & 0.822 & 1.10 \\
Lorenz-96 & MSE & 0.981 & 1.01 \\
Lorenz-96 & NLL & 0.992 & 1.01 \\
\bottomrule
\end{tabular}
\end{table}

The outcome is asymmetric across the datasets. On VAR, confounded children's false pairwise edges score 10--22\% above other non-edges, which is the classical hidden-confounder signature, and AUROC degrades from 1.00 to roughly 0.80. On Lorenz-96 the effect is essentially absent (around 1\%), plausibly because the nonlinear dynamics leave less of the hub's signature recoverable from pairwise structure.

\subsection{Regime Switching}

We take a fixed base graph and add one edge that is active only in the second half of the series. The learned $\hat{\mathbf{A}}$ is then scored against regime 1 and against regime 2 separately.

\begin{table}[htbp]
\centering
\small
\caption{\textbf{Regime switching.} An intermittently active edge is compared against edges that are true throughout.}
\label{tab:regime}
\resizebox{\textwidth}{!}{
\begin{tabular}{llccc}
\toprule
\textbf{Family} & \textbf{Objective} & \textbf{New edge score} & \textbf{Mean ``always-on'' true edge} & \textbf{Regime-1 / Regime-2 AUROC} \\
\midrule
VAR & MSE & 0.277 & 0.291 & 0.986 / 0.985 \\
VAR & NLL & 0.708 & 0.755 & 0.973 / 0.967 \\
Lorenz-96 & MSE & 0.210 & 0.256 & 0.943 / 0.930 \\
Lorenz-96 & NLL & 0.709 & 0.748 & 0.993 / 0.980 \\
\bottomrule
\end{tabular}
}
\end{table}

The model outputs approximately the union of the two regimes. In all four configurations the intermittent edge reaches 82--95\% of the mean score of permanently true edges, and AUROC remains high against both regimes. A practitioner therefore gets a superset of the true structure, which is the safe failure mode for screening. Recovering the switch points themselves is a different estimation problem. A possible mitigation is a sliding-window variant that re-estimates $\hat{\mathbf{A}}$ over overlapping segments, with change-point detection on $\|\hat{\mathbf{A}}_t - \hat{\mathbf{A}}_{t-1}\|$. We leave dynamic causal discovery to future work.

\subsection{On the Scope of the Stationary-Graph Setting}

We do not claim to solve all variants of causal discovery, and we state these assumptions in Section 2.1 and their limitations in Appendix \ref{limit}. But stationary-graph discovery is not a toy restriction. Granger causality is classically defined for stationary processes, and almost all of the models we compare against report a single time-invariant causal graph, with UnCLe \citep{bi2025unclescalabledynamiccausal} being the notable exception. The setting also has ample real-world application. It is the operative regime for gene regulatory networks, where our DREAM3 ground truth is a real \textit{E. coli} and yeast network with a fixed topology, and equally for industrial process monitoring, sensor and power networks, and climate teleconnections.

\section{Separating Mean and Variance Parents}
\label{app:typed_parent_recovery}

Mixed Physics provides separate ground-truth labels for mean- and
variance-driven dependencies. We investigate whether an extension of
M2C-NLL can recover these parent types by learning two global adjacency
matrices, $\widehat{A}^{\mu}$ and $\widehat{A}^{v}$, for conditional mean
and variance prediction, respectively.

\paragraph{Two-graph architecture.}
The two prediction paths share variable-local embeddings and Transformer
encoder weights, but maintain separate intermediate representations.
Each path applies its own adjacency matrix throughout all encoder layers
and attention heads, before its respective prediction head. Thus,
cross-variable representations are computed separately for mean and
variance prediction. This introduces 100 additional parameters for the
ten-variable benchmark (26,122 versus 26,022 parameters with window
length 3). We train both paths jointly using Gaussian NLL and a sparsity
penalty equal to $\lambda$ times the average of the two off-diagonal
graph penalties. No parent-type labels or disjointness constraints are
used during training.

\paragraph{Experimental setup.}
The regenerated datasets store $W_{\mu}$ and $W_{\sigma}$, whose nonzero
entries identify mean and variance parents in target-by-source
orientation. The 50:50, 75:25, and 100:0 settings contain respectively
13/14, 7/20, and 0/27 off-diagonal mean/variance edges; the setting names
refer to the variance-to-mean allocation ratio, subject to rounding.
All datasets also contain ten mean self-loops. We run the model on hyper-parameters selected for the standard single graph M2C-NLL on Mixed Physics without additional tuning.

\paragraph{Parent recovery.}
We evaluate each learned adjacency matrix separately against its
corresponding binary ground truth over all off-diagonal entries.
Mean-parent AUROC therefore treats variance-only edges as negatives,
and vice versa. These AUROCs use continuous scores without thresholding.
For a separate parent-type diagnostic, we label each known true edge
according to whichever adjacency matrix assigns it the larger score.
This diagnostic is conditional on an edge existing and does not require
that the edge first be discovered.

\begin{table}[t]
    \centering
    \caption{Parent recovery with two-graph M2C-NLL at the fixed final
    training epoch. Type accuracy is conditional on the true edge set.
    Mean-parent AUROC is undefined for 100:0 because no off-diagonal
    mean edges exist. (single seed run)}
    \label{tab:typed_parent_recovery}
    \begin{tabular}{lcccc}
        \toprule
        Setting & Mean AUROC & Variance AUROC & Union AUROC & Type accuracy \\
        \midrule
        50:50 & 0.9690 & 0.9962 & 0.9806 & 25/27 (92.6\%) \\
        75:25 & 1.0000 & 0.9943 & 0.9971 & 27/27 (100.0\%) \\
        100:0 & --- & 1.0000 & 1.0000 & 27/27 (100.0\%) \\
        \bottomrule
    \end{tabular}
\end{table}

The model correctly labels all 13 mean edges and 12 of 14 variance
edges in 50:50, and all seven mean edges and 20 variance edges in
75:25. In the variance-only 100:0 setting, all 27 true edges receive
larger variance than mean scores, and variance-edge AUROC is 1.0000.

\paragraph{Separately thresholded graphs.}
For visualization and binary per-graph metrics, we threshold the matrices
independently to match their respective ground-truth densities.
An edge may appear in both predicted graphs.

\begin{table}[t]
    \centering
    \caption{Binary recovery with separate density-matched thresholds.
    SHD is the number of false positives plus false negatives.
    Mean F1 is undefined for 100:0, where the true mean graph is empty.}
    \label{tab:typed_parent_binary}
    \begin{tabular}{lcccc}
        \toprule
        Setting & Mean SHD & Mean F1 & Variance SHD & Variance F1 \\
        \midrule
        50:50 & 6 & 0.7692 & 4 & 0.8571 \\
        75:25 & 0 & 1.0000 & 2 & 0.9500 \\
        100:0 & 0 & --- & 0 & 1.0000 \\
        \bottomrule
    \end{tabular}
\end{table}

\begin{figure}[t]
    \centering
    
            \includegraphics[width=0.9\linewidth]{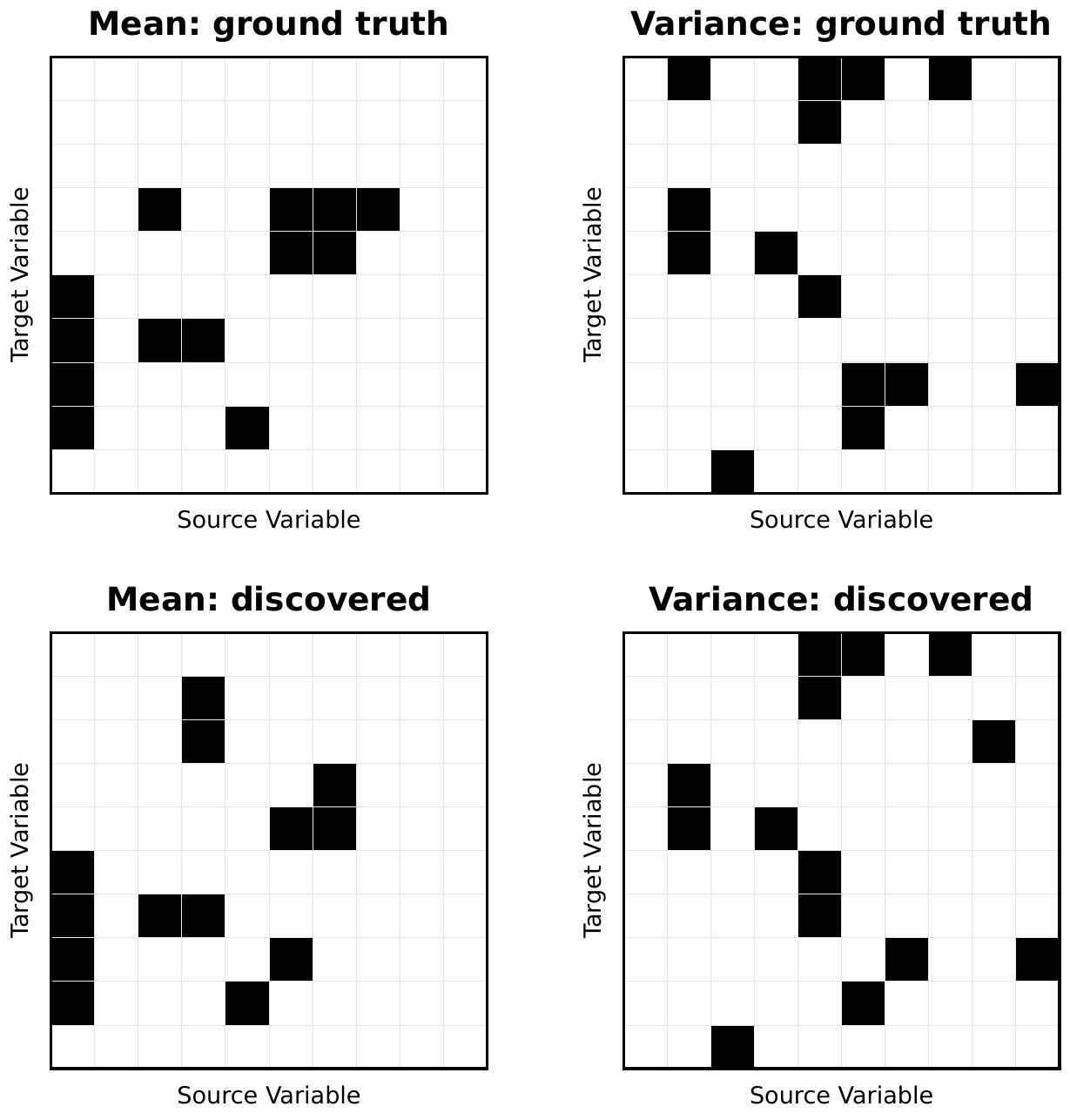}

    \caption{Mean and variance parent matrices for Mixed Physics 50:50.
    Rows denote targets and columns denote sources. Black cells indicate
    edges. Self-loops are excluded. The discovered matrices retain the
    top 13 mean scores and top 14 variance scores independently.}
    \label{fig:typed_parent_matrices}
\end{figure}

\paragraph{Overall graph recovery.}
Using the elementwise maximum of the two matrices as a union score,
the extension achieves AUROCs of 0.9806, 0.9971, and 1.0000 across
the three settings. Matched single-graph controls achieve 0.9877,
0.9976, and 1.0000, respectively. The extension therefore provides
additional parent-type information.

\section{Thresholding}
\label{thresh}

The primary output of the Mask2Cause architecture is a continuous adjacency probability matrix $\hat{\mathbf{A}} \in [0,1]^{N \times N}$. While continuous evaluation metrics such as AUROC and AUPRC comprehensively evaluate the model's ranking capability across all possible thresholds, downstream applications and discrete metrics (e.g., Structural Hamming Distance, F1 ) require a binarized causal graph. 

\textbf{Density-Matched Thresholding for Benchmarking.}
For standardized benchmark evaluations where the ground truth is accessible, we follow common practice by selecting a threshold that matches the edge density of the predicted graph to the known density of the ground-truth graph (expected edge density). For example, if the true VAR system possesses a known edge density of 30\%, we threshold $\hat{\mathbf{A}}$ to retain exactly the top 30\% of predicted edges.

\textbf{Unsupervised Thresholding for Real-World Application.}
In real-world settings where the expected density is unknown, a threshold must be inferred directly from the model's output. Mask2Cause is particularly well-suited for this scenario because it does not produce ambiguous, uniformly distributed scores. Instead, the combination of our log-gated attention mechanism and the $L_1$ sparsity penalty actively suppresses "lukewarm" connections. 

True causal edges successfully resist the regularization penalty to maintain high activation weights, while non-causal edges are heavily penalized. Consequently, the optimized continuous matrix $\hat{\mathbf{A}}$ exhibits a stark, highly separable bimodal distribution. For example, the average weights for true causal edges versus non-causal edges naturally diverge to $0.514$ vs. $0.164$ on the Lorenz-96 system, and $0.491$ vs. $0.208$ on the VAR system. 

Due to this strong bimodality, practitioners can reliably determine the binarization threshold via simple unsupervised 1D clustering. The clustering algorithm automatically identifies the optimal separation valley between the two score distributions, allowing for accurate discrete graph extraction without expected density.

\textbf{How much is lost without ground-truth density.} Table \ref{tab:thresh_sensitivity} quantifies the cost of discarding the density oracle. Every dataset is binarized by the unsupervised rule above, applied uniformly with no per-dataset tuning and no access to ground-truth density, alongside the density-matched thresholds used elsewhere in the paper. AUROC is threshold-free and therefore identical under both rules, so only the discretization differs and we report SHD. The model is M2C-MSE.

\begin{table}[htbp]
\centering
\small
\caption{\textbf{Sensitivity to the thresholding rule.} SHD under density-matched thresholding against the unsupervised clustering rule of this appendix. AUROC is unaffected by the choice.}
\label{tab:thresh_sensitivity}
\begin{tabular}{lccc}
\toprule
\textbf{Dataset} & \textbf{AUROC} & \textbf{SHD (density-matched)} & \textbf{SHD (unsupervised)} \\
\midrule
Lorenz-96 ($F=10$, $T=500$) & 0.997 & 2 & 2 \\
Lorenz-96 ($F=40$, $T=500$) & 1.000 & 0 & 3 \\
Lorenz-96 ($F=10$, $T=250$) & 0.995 & 6 & 10 \\
Lorenz-96 ($F=40$, $T=250$) & 0.999 & 2 & 4 \\
VAR ($T=1000$) & 1.000 & 0 & 0 \\
VAR ($T=500$) & 1.000 & 0 & 0 \\
Medical & 0.903 & 74 & 81 \\
AQI & 0.898 & 188 & 198 \\
Traffic & 0.681 & 54 & 57 \\
\bottomrule
\end{tabular}
\end{table}

Removing the density oracle never costs more than 4\% of the candidate entries, and VAR is recovered exactly under both rules. This is the behaviour the bimodality argument predicts: the separation valley is wide enough that simple clustering locates it reliably.

\section{Robustness of Causal Masking in Deep Architectures ($M \ge 2$)}
\label{robust}

As introduced in Section \ref{sec:theoretical_grounding}, our optimization objective is motivated by the optimal log-loss identity in Eq.~\eqref{eq:di_logloss}, which assumes the underlying predicting process is an ideal, perfectly minimizing estimator. To properly connect our model to this proposition, we must assume that setting $\hat{A}_{ij} = 0$ excludes variable $j$ from the prediction of variable $i$. For a single-layer architecture ($M=1$), this assumption effectively holds (for $\delta$ small in equation \ref{eq:score}), as the mask explicitly blocks the only available pathway, guaranteeing that $j$ cannot cause $i$ in the model's prediction.

However, leveraging the full capacity of Transformer architectures requires stacking encoder layers ($M \ge 2$). In a multi-layer setting, a theoretical discrepancy arises between the ideal estimator assumption and the structural mask. Specifically, even if the direct edge is penalized and masked ($\hat{A}_{ij} = 0$), the computational graph of a deep Transformer could theoretically route the necessary information of $j$ to $i$ through an intermediate conduit variable $k$, provided that $\hat{A}_{ik} > 0$ and $\hat{A}_{kj} > 0$. This multi-hop capacity of the neural network threatens to violate the strict variable exclusion required by the theoretical proposition, potentially allowing the model to minimize the NLL for variable $i$ while successfully evading the $L_1$ penalty on the direct edge $\hat{A}_{ij}$.

While this ``multi-hop leakage'' is mathematically possible in an unbounded optimization landscape, it does not manifest practically in Mask2Cause. Firstly, it is important to note that a potential ``multi-hop leakage'' is mostly a concern under specific configurations of the Ground Truth: namely, graphs containing Parallel Causal Pathways. If the ground truth contains a direct edge $j \rightarrow i$ alongside an alternative causal chain (e.g., a triangle $j \rightarrow k \rightarrow i$, or a longer path $j \rightarrow k_1 \dots \rightarrow k_n\rightarrow i$), the model could theoretically evade the $L_1$ penalty on the direct edge by routing $j$'s signal through the intermediaries.  Even in such cases, the network is subject to architectural bottlenecks that make indirect routing practically unfeasible. After the usual hyperparameter tuning of the sparsity hyperparameter $\lambda$, we see that the massive NLL penalty incurred by attempting indirect routing automatically made to outweigh the $L_1$ savings of dropping the direct edge.

\subsection{Architectural and Optimization Bottlenecks Suppressing Leakage}

The architecture actively suppresses theoretical information leakage via two distinct mechanisms:

\textbf{1. Representational Contention and Capacity Limits.} 
To use variable $k$ as a lossless conduit for $j$, the network must allocate specific attention heads and latent subspace dimensions within $k$'s token strictly for $j$'s signal. However, under the Directed Information framework, if $j$ is a true direct causal parent of $i$ (i.e., $I(X^j \rightarrow X^i | X^{-\{i,j\}}) > 0$), then $j$ contains unique information in prediction of $i$ that cannot be derived from $k$'s history.
Forcing $k$ to embed this additional information means $j$ and $k$ must compete for finite representational bandwidth, reducing the parameter capacity available to model $k$'s own complex trajectory. Because the network jointly optimizes the forecasting loss for all variables simultaneously, the degradation in $k$'s predictive precision causes a massive spike in the forecasting error (NLL/MSE). Consequently, the optimizer is incentivized to retain the true direct edge $\hat{A}_{ij}$, as the NLL degradation incurred by the conduit node drastically outweighs the $\lambda \|\hat{A}_{ij}\|_1$ savings.

Routing $j$'s signal through $k$ only becomes advantageous if the network possesses disproportionately high representational capacity and sparsity penalty ($\lambda$). As demonstrated in our hyperparameter sensitivity analysis \ref{sec:sensitivity} on Lorenz-96 (a dataset that does contain Parallel Causal Pathways), the optimal space for $\lambda$ and $d_{model}$ is remarkably broad (indicating that standard hyperparameter tuning likely isolates the regime where indirect routing is heavily penalized).

\textbf{2. Shared-Weight Disentanglement.} 
Mask2Cause applies a universally shared Feed-Forward Network and final projection head across all variable tokens. This constraint in the architecture allows us to argue that node-specific routing strategies
harder to learn, even when spare representational capacity
is available. Let us assume that node $k$ could partition its latent vector, dedicating one subspace to its own dynamics and an orthogonal "null space" to $j$'s smuggled signal, thereby protecting its own prediction from interference. Such routing requires the shared transformations to preserve
information useful for downstream targets while supporting
the intermediate token's own prediction. Because these transformations are shared across all variables,
a routing strategy must also remain compatible with their
prediction tasks. Any resulting interference discourages indirect routing
when its forecasting cost exceeds the sparsity savings
from dropping the direct edge.

\subsection{Empirical Validation}

To empirically test whether the optimization landscape prefers direct causal edges over multi-hop routing, we conducted a masking ablation study. We locked the adjacency matrices and compared the baseline forecasting MSE of the true causal mask against an ablated mask. 

For the Lorenz-96 system, the underlying ODE dictates that the true parents of node $i$ are $i-2, i-1, i$, and $i+1$. We ablated the mask by explicitly removing the direct edge $i-2 \to i$, forcing the model to attempt to route the necessary information of $i-2$ exclusively through the intermediate parent $i-1$. The original VAR dataset does not have this potential weakness to test. So, we conducted a similar ablation on a custom VAR dataset parameterized such that $x_{i-2} \to x_i$ and $x_{i-1} \to x_i$.

\begin{table}[h]
\centering
\caption{MSE degradation under forced multi-hop routing (ablated mask)}
\label{tab:mask_ablation}
\resizebox{\linewidth}{!}{
\begin{tabular}{lccc}
\toprule
\textbf{Dataset} & \textbf{True Mask (MSE)} & \textbf{Ablated Mask (MSE)} & \textbf{Difference ($\Delta_{\text{MSE}}$)} \\
\midrule
Lorenz-96, $F=40, T=500$ & $0.028651$ & $0.076617$ & $+0.047966$ \\
Lorenz-96, $F=40, T=250$ & $0.045148$ & $0.100554$ & $+0.055406$ \\
Lorenz-96, $F=10, T=500$ & $0.011014$ & $0.016673$ & $+0.005659$ \\
Lorenz-96, $F=10, T=250$ & $0.009012$ & $0.013292$ & $+0.004280$ \\
\midrule
VAR, $T=1000$            & $0.499738$ & $0.503702$ & $+0.003964$ \\
VAR, $T=500$             & $0.521975$ & $0.533000$ & $+0.011025$ \\
\bottomrule
\end{tabular}
}
\end{table}

In our baseline Lorenz-96 configuration, the effective $L_1$ sparsity penalty ($\lambda$) per node is $0.0002$. Dropping the 10 direct edges across the system saves only $0.002$ in the total $L_1$ penalty. As shown in Table \ref{tab:mask_ablation}, the MSE degradation ($\Delta_{\text{MSE}}$) caused by forcing multi-hop routing entirely dwarfs the $\lambda$ savings. The optimizer prefers direct edges.

\textbf{Robustness to Encoder Depth.} 
Because the architectural bottlenecks naturally regularize the flow of information, Mask2Cause remains remarkably robust even at extreme depths where multi-hop routing would theoretically be easiest. As shown in Table \ref{tab:depth_robustness}, scaling the number of encoder layers ($M$) up to $20$ yields little to no degradation in structural recovery (in fact, we see an improvement in evaluation for Lorenz-96). We default to $M=2$ in our main experiments solely to minimize parameter count, computational overhead, and the risk of overfitting in data-scarce regimes.

\begin{table}[h]
\centering
\caption{Causal discovery performance vs. encoder depth ($M$) on Lorenz-96 ($F=40, T= 500$) and VAR ($T=500$). (single seed).}
\label{tab:depth_robustness}
\begin{tabular}{lcccc}
\toprule
\textbf{Dataset} & \textbf{Encoders ($M$)} & \textbf{AUROC} & \textbf{AUPRC} & \textbf{SHD} \\
\midrule
Lorenz-96 & 1  & 0.997 & 0.996 & 4 \\
Lorenz-96 & 2  & 0.998 & 0.998 & 2 \\
Lorenz-96 & 4  & 1.000 & 1.000 & 0 \\
Lorenz-96 & 10 & 1.000 & 1.000 & 0 \\
Lorenz-96 & 20 & 1.000 & 1.000 & 0 \\
\midrule
VAR & 1  & 1.000 & 1.000 & 0 \\
VAR & 2  & 1.000 & 1.000 & 0 \\
VAR & 4  & 1.000 & 1.000 & 0 \\
VAR & 10 & 0.999 & 0.997 & 2 \\
VAR & 20 & 0.999 & 0.998 & 2 \\
\bottomrule
\end{tabular}
\end{table}

\section{Alternative Masking Mechanisms}
\label{masking}

This appendix expands the masking ablation of Section 5. All variants below retain a learnable adjacency matrix and end-to-end graph extraction, and differ only in how $\hat{\mathbf{A}}$ enters the attention computation. Everything else, including the encoder, the sparsity penalty and the diagonal forcing, is held fixed.

\subsection{Which Alternatives Are Actually Distinct}

The composition of Eq.~\ref{eq:score} and the subsequent softmax is exactly equivalent to a multiplicative gate applied to the attention weights with renormalization:
\begin{equation}
    \text{Softmax}(\mathbf{S} + \log(\hat{\mathbf{A}} + \delta)) \propto \exp(\mathbf{S}) \odot (\hat{\mathbf{A}} + \delta)\,.
\end{equation}
A renormalized multiplicative mask is therefore not an ablation of our mechanism, since it is our mechanism written differently. Two alternatives remain genuinely distinct.

\textbf{(A) Unnormalized multiplicative masking.} We gate the attention weights without restoring the simplex constraint:
\begin{equation}
    \text{Attention}(\mathbf{Q}, \mathbf{K}, \mathbf{V}, \hat{\mathbf{A}}) = \left( \hat{\mathbf{A}} \odot \text{Softmax}(\mathbf{S}^{\text{raw}}) \right) \mathbf{V}\,.
\end{equation}

\textbf{(B) Bounded additive bias.} We replace the logarithmic term of Eq.~\ref{eq:score} with a linear penalty of tunable strength $\alpha$, tuned per dataset with the best result reported:
\begin{equation}
    \text{Score}_{ij} = \frac{\mathbf{q}_i \mathbf{k}_j^\top}{\sqrt{d_k}} + \alpha \hat{A}_{ij}\,.
\end{equation}

Table \ref{tab:ablation} reports AUROC for both on CausalTime (M2C-MSE) and Mixed Physics (M2C-NLL).

\subsection{Why the Alternatives Fall Short}

\textbf{(A) couples structure to signal magnitude.} Once rows of the attention matrix no longer sum to one, the norm of a token's output scales with the number and strength of its retained parents. Layer normalization partially compensates for this, which dilutes the gating effect and, more damagingly, gives the optimizer an incentive to inflate $\hat{\mathbf{A}}$ purely to preserve activation scale, working directly against the $L_1$ penalty.

\textbf{(B) makes the criterion depend on a tuned hyperparameter.} The term $\alpha \hat{A}_{ij}$ acts as a bounded boost, not a penalty. A pruned edge receives no suppression, only the absence of a boost, and retains a fraction $e^{-\alpha}$ of the weight of a retained edge. The effective pruning floor is therefore a tuned quantity that also governs optimization stability.

\textbf{The log gate preserves score sensitivity above the numerical floor.} For $\hat{A}_{ij} \gg \delta$, the logarithm approximately cancels the sigmoid factor in the pre-softmax score derivative:
\begin{equation}
    \frac{\partial\, \text{Score}_{ij}}{\partial \Theta_{ij}} = \frac{\sigma'(\Theta_{ij})}{\hat{A}_{ij} + \delta} = \frac{\hat{A}_{ij}(1 - \hat{A}_{ij})}{\hat{A}_{ij} + \delta} \approx 1 - \hat{A}_{ij}\,,
\end{equation}
Thus, when $\delta \ll \hat{A}_{ij} \ll 1$, this score derivative is approximately 1. The direct multiplicative gate in (A) and additive bias in (B) instead retain the sigmoid factor $\hat{A}_{ij}(1-\hat{A}_{ij})$ in their derivatives with respect to $\Theta_{ij}$. This comparison concerns the local gating operations: the softmax and downstream network can still attenuate the loss gradient, so it does not guarantee recovery of suppressed edges. For fixed $\delta>0$, the log-gated score derivative also tends to zero as $\hat{A}_{ij}\to0$.

\section{Gaussian and Heavy-Tailed Likelihoods}
\label{heavytail}

Appendix \ref{limit} notes that the Gaussian NLL objective can be swapped for a heavier-tailed likelihood when the data contains extreme outliers. This appendix demonstrates that substitution. The architecture and the $\mu$ and $\sigma$ heads are unchanged, and only the loss differs between Gaussian NLL and Student-t NLL.

The data is a 30\% dense random directed graph over 10 variables, where each edge is randomly assigned as either mean-causal or variance-causal. The setup follows Mixed Physics, except that the noise is changed from Gaussian to Student-t with 5 degrees of freedom, rescaled to unit variance so that it plays the same structural role as the original unit-variance Gaussian noise.

\begin{table}[htbp]
\centering
\small
\caption{\textbf{Gaussian versus Student-t likelihood} on heavy-tailed data and on a Gaussian control.}
\label{tab:studentt}
\begin{tabular}{llcc}
\toprule
\textbf{Data} & \textbf{Loss} & \textbf{AUROC} & \textbf{F1} \\
\midrule
Heavy-tailed 50:50 & Gaussian NLL & $0.641 \pm 0.043$ & 0.495 \\
Heavy-tailed 50:50 & \textbf{Student-t NLL} & $\mathbf{0.906 \pm 0.012}$ & $\mathbf{0.815}$ \\
Gaussian control 50:50 & \textbf{Gaussian NLL} & $\mathbf{0.998 \pm 0.003}$ & $\mathbf{0.970}$ \\
Gaussian control 50:50 & Student-t NLL & $0.668 \pm 0.039$ & 0.539 \\
\bottomrule
\end{tabular}
\end{table}

The trade-off is symmetric. Under genuine outliers, Student-t NLL beats Gaussian NLL, and under correctly specified Gaussian noise the Gaussian loss retains the expected edge. This supplies the practical selection rule referred to in Appendix \ref{guide}: switch to the heavier-tailed likelihood when the data shows high kurtosis.

\section{Datasets}
\label{A}

To evaluate the efficacy of our approach in identifying causal structures, we utilize the following standard and synthetic benchmarks with corresponding ground truth causal graphs:
\begin{itemize}
    \item Synthetic : Vector Autoregressive (VAR) (Linear), Lorenz-96 (Non-Linear)
    \item Real-World : Causal Time \cite{chengcausaltime}, DREAM3
    \item Heteroscedastic : Mixed Physics
\end{itemize}  

\subsection{Vector Autoregressive Dataset (Linear)}
To evaluate the model's performance on linear systems with explicit multi-step memory, we simulate sparse Vector Autoregressive (VAR) processes. A VAR process of order $K$, denoted as VAR($K$), evolves according to the linear recurrence:
\begin{equation}
    \mathbf{x}_{t} = \sum_{k=1}^{K} B^{(k)} \mathbf{x}_{t-k} + \mathbf{\eta}_{t}
\end{equation}
where $\mathbf{x}_t \in \mathbb{R}^p$ is the state vector at time $t$, $B^{(k)} \in \mathbb{R}^{p \times p}$ is the transition matrix representing dependencies at lag $k$, and $\mathbf{\eta}_t \sim \mathcal{N}(0, 0.1^2 I)$ is isotropic Gaussian noise.


The \textbf{Ground Truth Causal Graph} $A \in \{0, 1\}^{p \times p}$ is defined by the existence of a dependency at \textit{any} lag. Specifically, $A_{ij} = 1$ (variable $j$ causes variable $i$) if and only if the coefficient for $j \to i$ is non-zero in at least one lag matrix:
\begin{equation}
    A_{ij} = \mathbb{I}\left( \sum_{k=1}^{K} \left| B_{ij}^{(k)} \right| > 0 \right)
\end{equation}
where $\mathbb{I}(\cdot)$ is the indicator function. This allows us to test if the model can aggregate causal signals distributed across multiple past time lags. Figure \ref{fig:VAR_ground_truth} shows the ground truth causal graph for a 10 variable system ($p=10$). We source the dataset from \cite{Khanna2020Economy}, accessible at \url{https://github.com/sakhanna/SRU_for_GCI/tree/master} (MIT License).
\begin{figure}[h]
    \centering
    \includegraphics[width=0.4\linewidth]{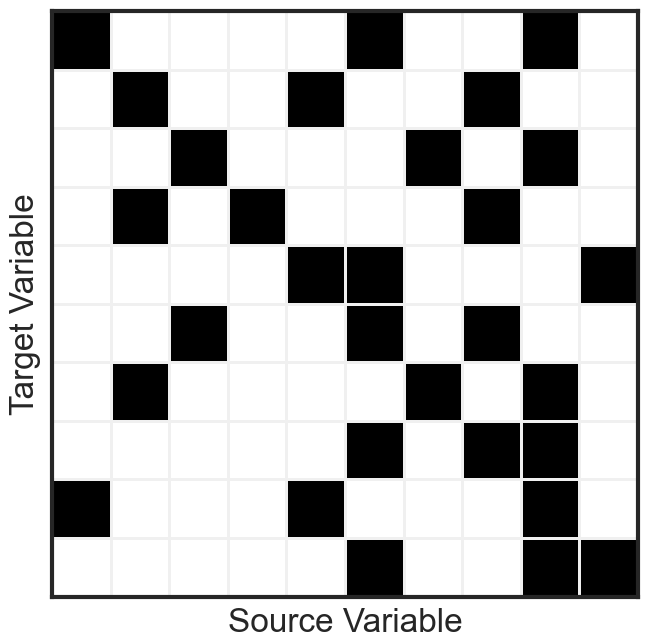}
    \caption{Ground truth for p=10 VAR system}
    \label{fig:VAR_ground_truth}
\end{figure}


\subsection{Lorenz-96 Dataset (Nonlinear)}

The Lorenz-96 model is a continuous-time dynamic system often used to simulate complex atmospheric physics \cite{karimi2010extensive}. It consists of $p$ variables $x_1, \dots, x_p$ arranged in a cyclic dependency structure. The evolution of variable $x_i$ is governed by the system of ordinary differential equations (ODEs):
\begin{equation}
    \frac{d x_{i}}{d t} = (x_{i+1} - x_{i-2}) x_{i-1} - x_{i} + F
    \label{eq:lorenz}
\end{equation}
where indices are taken modulo $p$. The forcing constant $F$ determines the level of nonlinearity and chaos in the system. Following the benchmarks used in eSRU \cite{Khanna2020Economy}, we have $p=10$ and evaluate for $F=10$ and $F=40$ (a more turbulent regime). The high forcing constant in $F=40$ regime amplifies the quadratic interaction terms, making nonlinear dependencies on parent variables significantly more dominant relative to linear decay.



The \textbf{Ground Truth Causal Graph} $A \in \{0, 1\}^{p \times p}$ is defined as the sparse, static adjacency matrix $A \in \{0, 1\}^{p \times p}$ derived from the ODE mechanism: $A_{ij} = 1$ if and only if $j \in \{i-2, i-1, i, i+1\} \pmod p$. This corresponds to a strict \textbf{Lag-1} causal dependency. Figure \ref{fig:LOR_ground_truth} shows the ground truth causal graph for a 10 variable system ($p=10$). We source the dataset from \cite{Khanna2020Economy}, accessible at \url{https://github.com/sakhanna/SRU_for_GCI/tree/master} (MIT License).
\begin{figure}[h]
    \centering
    \includegraphics[width=0.4\linewidth]{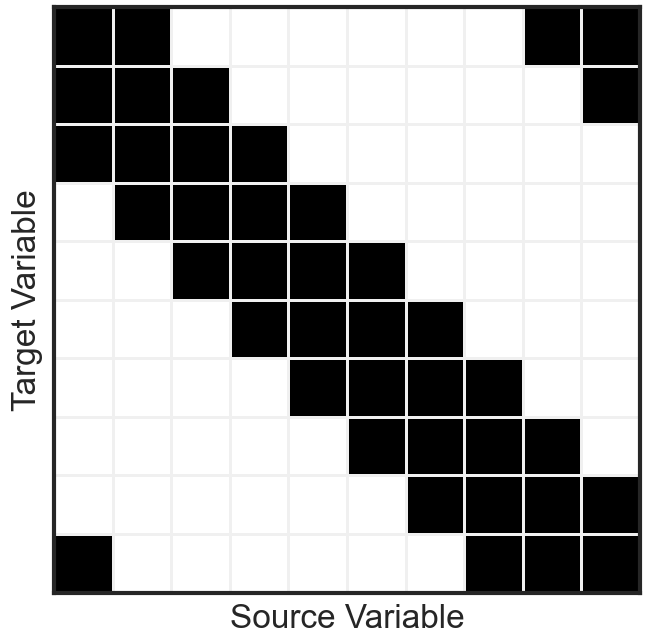}
    \caption{Ground truth causal graphs for p=10 Lorenz-96 system}
    \label{fig:LOR_ground_truth}
\end{figure}


\subsection{CausalTime Datasets}
To rigorously evaluate performance on realistic, high-dimensional data, we utilize the CausalTime benchmark suite \cite{chengcausaltime}. Unlike purely synthetic datasets (VAR, Lorenz) or real-world datasets with unverified ground truths, CausalTime employs a Real-to-Generated pipeline - It fits a deep generative model to real observational data and then generates new samples from the fitted model. This ensures the data retains realistic statistical properties while adhering to a known, mathematically precise causal graph.

\par The underlying generation process models the time series using a Non-Linear Autoregressive (NAR) framework implemented via Causally Disentangled Neural Networks (CDNNs) \cite{chengcausaltime}. The value of variable $i$ at time $t$ is generated as:
\begin{equation}
    \hat{x}_{t,i} = f_{\theta_i}(\text{Parents}(x_{t,i})) + \hat{x}_{t-1, i}^R + \eta_{t,i}
\end{equation}
where $f_{\theta_i}$ represents the causal mechanism, $\hat{x}_{t-1, i}^R$ is a residual term capturing non-causal dynamics, and $\eta_{t,i}$ is noise modeled by a Normalizing Flow. This decomposition allows for the rigorous definition of a ground truth graph based solely on the inputs to $f_{\theta_i}$, while the residual term ensures the trajectory complexity matches real-world recordings.

\textbf{Domains and Configuration:}
We evaluate our method on three distinct domains provided by the benchmark:
\begin{itemize}
    \item \textbf{Traffic ($p=20$):} Based on traffic speed data from sensors in the San Francisco Bay Area. The causal structure is sparse and informed by the physical road network topology.
    \item \textbf{AQI ($p=36$):} Derived from PM2.5 air quality measurements across 36 monitoring stations in China. The underlying graph reflects geographic proximity, where causal links exist between stations within a 40 km radius.
    \item \textbf{Medical ($p=20$):} Constructed from the MIMIC-IV database, tracking vital signs and clinical events for ICU patients. Unlike the spatial datasets, this domain lacks a geometric prior, relying on the complex physiological interactions captured by the CDNN to define the causal structure.
\end{itemize}

For all three datasets, the \textbf{Ground Truth Causal Graph} is the binary adjacency matrix $A$ used to constrain the generation process. We utilize the standard benchmark configuration consisting of 500 generated samples (each of length $T=40$) for each domain.

Figure \ref{fig:CausalTime_ground_truth} shows the ground truth causal graph for Traffic, AQI and Medical datasets. We source the dataset from \cite{chengcausaltime}, accessible at \url{https://www.causaltime.cc/} (MIT License).

\begin{figure}[h]
    \centering
    \includegraphics[width=1\linewidth]{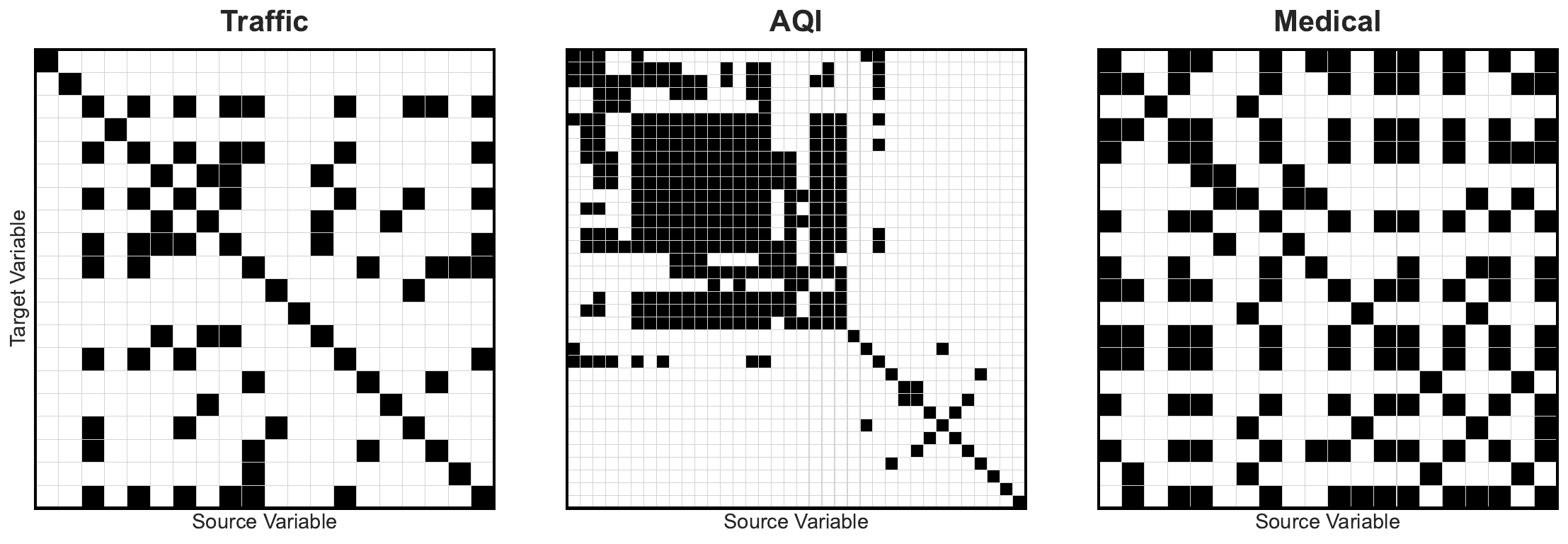}
    \caption{Ground truth causal graphs for CausalTime datasets}
    \label{fig:CausalTime_ground_truth}
\end{figure}
\subsection{DREAM3 Datasets}
To evaluate performance on systems governed by continuous-time dynamics, we utilize the DREAM3 \textit{InSilico} benchmark \cite{prill2010towards}. This suite provides a standardized platform for causal discovery in networks where interactions are defined by biologically-inspired structural equations.\par The trajectories are generated using the \textit{GeneNetWeaver} (GNW) simulator. The dynamics of the $i$-th variable $x_i$ are modeled by a system of ordinary differential equations (ODEs):
\begin{equation}
\frac{dx_i}{dt} = m_i \cdot f_i(x_{\text{Parents}(i)}) - \lambda_i x_i \end{equation}
where $f_i(\cdot)$ is a non-linear input function (modeled via Hill-cube kinetics) representing the causal influence of parent nodes, $m_i$ is the maximum production rate, and $\lambda_i$ is the degradation rate. This formulation ensures that the ground truth graph is mathematically defined by the non-zero partial derivatives of the production term with respect to the input variables.
\textbf{Domains and Configuration:}
The benchmark utilized consists of five distinct networks in the 100-node ($p=100$) category:
\begin{itemize}
\item \textbf{Ecoli1 \& Ecoli2:} Sub-networks derived from the \textit{Escherichia coli} transcriptional map.
\item \textbf{Yeast1, Yeast2, \& Yeast3:} Sub-networks derived from the \textit{Saccharomyces cerevisiae} (Yeast) regulatory map.
\end{itemize}
For all networks, the \textbf{Ground Truth Causal Graph} is the binary adjacency matrix used to define the ODE couplings. The evaluation utilizes time-series data consisting of 46 independent trajectories, each containing $T=21$ time points. Figure \ref{fig:DREAM3_ground_truth} illustrates the high sparsity and topological structure of these causal matrices. We source the dataset from \cite{Khanna2020Economy}, accessible at \url{https://github.com/sakhanna/SRU_for_GCI/tree/master} (MIT License).

\textbf{Note: } 46 independent trajectories of $21$ time points each is an extremely scarce dataset for learning a Causal Graph containing a $100$ nodes. We use this purely as a stress test to see how our model would perform against baselines in such a data scarce regime.

\begin{figure}[h]
\centering
\includegraphics[width=1\linewidth]{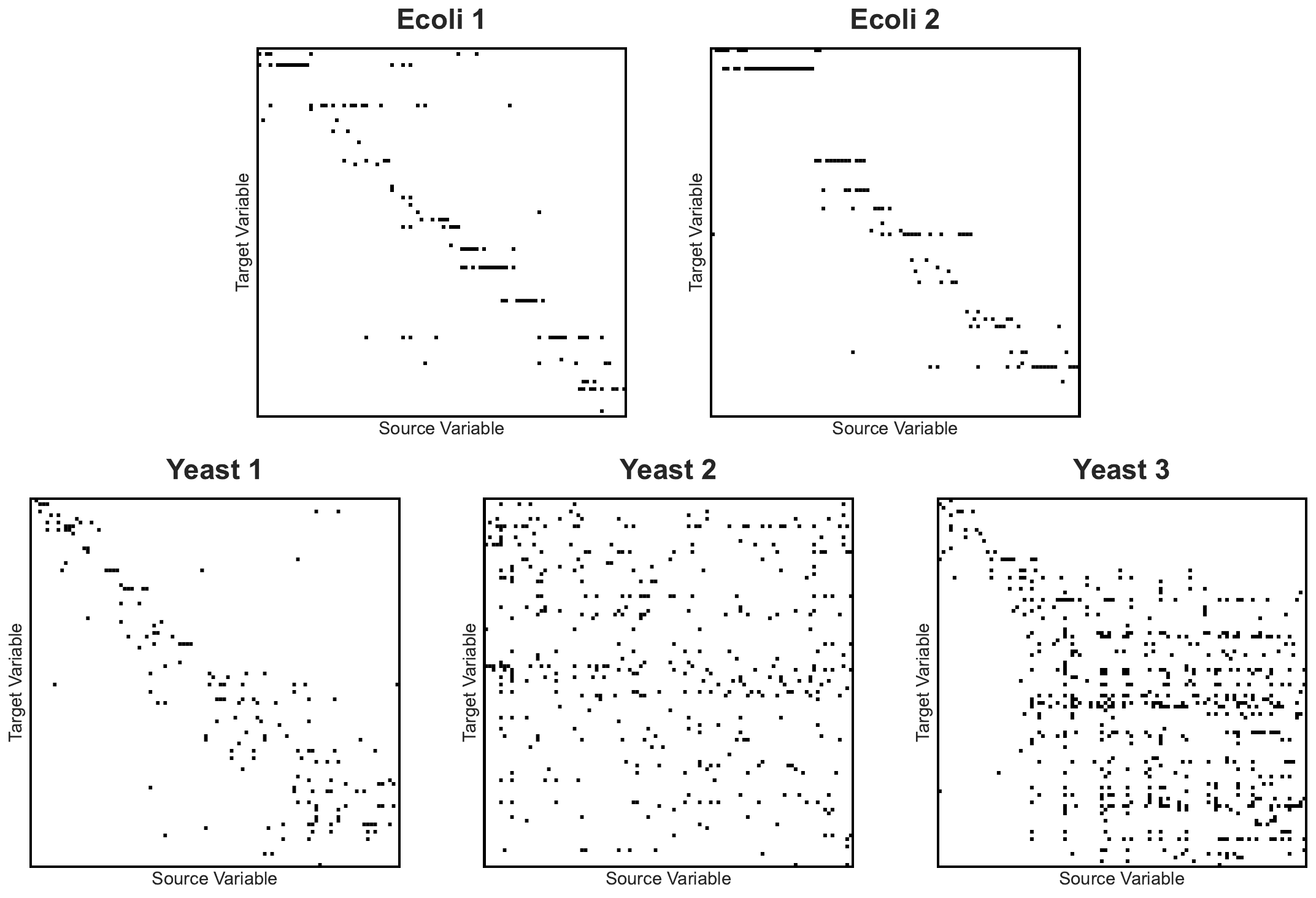}
\caption{Ground truth causal graphs for DREAM3 networks}
\label{fig:DREAM3_ground_truth}
\end{figure}

\subsection{Mixed Physics (Heteroscedastic)}
\label{mixed physics appendix}

We generate Mixed Physics using a heteroscedastic process
with disjoint mean and variance parent sets:
\begin{equation}
    x_t^i
    =
    \underbrace{\sum_{j=1}^{N}W^\mu_{ij}x_{t-L}^j}
        _{\text{Conditional mean}}
    +
    \underbrace{
        \left(\beta+\sum_{k=1}^{N}
        W^\sigma_{ik}(x_{t-L}^k)^2\right)^{1/2}
    }_{\text{Conditional standard deviation}}
    \eta_t^i.
    \label{eq:mixedphysics_appendix}
\end{equation}
The innovations $\eta_t^i$ are independent standard Gaussian
variables across targets and time, independent of the past.
We use $L=1$ and $\beta=0.1$.
The matrices $\mathbf{W}^\mu$ and
$\mathbf{W}^\sigma\geq 0$ have disjoint supports and
parameterize conditional mean and variance effects,
respectively. Mean self-dependencies are retained in every
configuration; variance edges are exclusively off-diagonal.

\textbf{Configurations.}
Each dataset contains $N=10$ variables, with $27$ of the
$N(N-1)=90$ possible off-diagonal edges present, giving
an inter-variable graph density of $30\%$.
We evaluate three approximate variance-to-mean edge ratios:
\begin{itemize}
    \item \textbf{50:50:}
    $14$ variance edges and $13$ mean edges.
    \item \textbf{75:25:}
    $20$ variance edges and $7$ mean edges.
    \item \textbf{100:0:}
    $27$ variance edges and no off-diagonal mean edges.
\end{itemize}
These counts exclude the ten mean self-loops.
Thus, the 100:0 configuration isolates variance-driven
cross-variable dependencies while preserving autoregressive
conditional means. The off-diagonal counts determine the
idealized mean-based AUROC references reported in
Section~\ref{experiments}.
Figure~\ref{fig:MP_ground_truth} shows the graph structures.

\textbf{Coefficient sampling and scaling.}
For each configuration, we randomly partition its
off-diagonal edges into mean and variance edges using the
specified counts. Before rescaling, mean self-coefficients
are sampled uniformly from $[0.3,0.7]$.
Off-diagonal mean coefficients have uniformly sampled
magnitudes in $[0.5,1.0]$ and independently sampled signs.
Variance coefficients are sampled uniformly from $[0.5,1.5]$.
All other coefficients are zero.

Writing the sampled matrices as
$\mathbf{B}_0$ and $\mathbf{C}_0$, we set
\begin{equation}
    \mathbf{M}_0=\mathbf{B}_0\circ\mathbf{B}_0+\mathbf{C}_0,
    \qquad
    k=\frac{0.9}{\rho(\mathbf{M}_0)},
    \qquad
    \mathbf{W}^\mu=\sqrt{k}\,\mathbf{B}_0,
    \qquad
    \mathbf{W}^\sigma=k\,\mathbf{C}_0,
\end{equation}
where $\circ$ denotes elementwise multiplication and
$\rho$ denotes spectral radius.
This sets the overall coefficient scale while preserving
edge positions and edge-type assignments.
We verify second-moment stability of the resulting matrices
using the full operator derived below.

\textbf{Second-moment recursion.}
Let $\mathbf{B}=\mathbf{W}^\mu$,
$\mathbf{C}=\mathbf{W}^\sigma$, and
$\mathbf{S}_t=\mathbb{E}[\mathbf{x}_t\mathbf{x}_t^\top]$.
Conditional on the past, the mean is
$\mathbf{B}\mathbf{x}_{t-1}$ and the covariance is
\[
    \operatorname{Diag}\!\left(
        \beta\mathbf{1}
        +\mathbf{C}(\mathbf{x}_{t-1}\circ\mathbf{x}_{t-1})
    \right).
\]
Here $\operatorname{diag}$ extracts a matrix's diagonal
and $\operatorname{Diag}$ constructs a diagonal matrix
from a vector.
Taking expectations gives the exact recursion
\begin{equation}
\begin{aligned}
    \mathbf{S}_t
    &=\mathcal{T}(\mathbf{S}_{t-1})+\beta\mathbf{I},\\
    \mathcal{T}(\mathbf{S})
    &=\mathbf{B}\mathbf{S}\mathbf{B}^\top
      +\operatorname{Diag}\!\left(
          \mathbf{C}\operatorname{diag}(\mathbf{S})
       \right).
\end{aligned}
\label{eq:mixedphysics_second_moment_appendix}
\end{equation}
The term $\mathbf{B}\mathbf{S}\mathbf{B}^\top$ retains
cross moments between mean parents, so this expression
does not assume that the variables are independent.

\textbf{Stability condition and verification.}
The operator $\mathcal{T}$ is linear and can be represented
by an $N^2\times N^2$ matrix $\mathbf{K}$ satisfying
\[
    \operatorname{vec}(\mathcal{T}(\mathbf{S}))
    =\mathbf{K}\operatorname{vec}(\mathbf{S}).
\]
We define $\rho(\mathcal{T})=\rho(\mathbf{K})$.
Iterating Eq.~\eqref{eq:mixedphysics_second_moment_appendix} yields
\begin{equation}
    \mathbf{S}_t
    =
    \mathcal{T}^{t}(\mathbf{S}_0)
    +\beta\sum_{r=0}^{t-1}\mathcal{T}^{r}(\mathbf{I}).
\end{equation}
If $\rho(\mathcal{T})<1$, the initialization term vanishes
and the series converges. Hence the second moments approach
the finite stationary solution
\begin{equation}
    \mathbf{S}_*
    =
    \beta\sum_{r=0}^{\infty}\mathcal{T}^{r}(\mathbf{I}),
    \qquad
    \mathbf{S}_*=\mathcal{T}(\mathbf{S}_*)+\beta\mathbf{I}.
\end{equation}
The coefficient rescaling above scales the full operator
by $k$, since both terms in $\mathcal{T}$ scale by $k$.
Evaluating this operator for the released configurations
gives spectral radii $0.8864$, $0.9230$, and $0.9000$
for 50:50, 75:25, and 100:0, respectively.
All three therefore satisfy the second-moment stability
condition.

\textbf{Simulation.}
We initialize the state at zero, discard $2{,}000$
burn-in steps, and retain the following $5{,}000$
observations. The released configurations use coefficient
sampling seed $0$ and innovation seed $1000$.
No clipping or bounding of the state is applied.

\textbf{Robustness to variable scaling.} Simulated benchmarks can be gamed when marginal variance increases along the causal order, since variance-ranking alone then recovers much of the graph \citep{reisach2021beware}. Two properties rule this out here. First, every loader in our pipeline standardises each variable before embedding, so multiplying variable $k$ by any constant $c > 0$ scales its mean and standard deviation identically and leaves the model input unchanged. Rescaling each variable of the 100:0 data by a random constant drawn from $[0.1, 10]$ leaves off-diagonal AUROC identical at every seed. Second, the benchmark is far from the regime the concern applies to. Varsortability is 0.667, 0.667 and 0.630 across the three ratios, against the value near 0.94 reported for simulated additive-noise DAGs, and neither figure can account for an off-diagonal AUROC of 0.99. Standardisation is also the remedy prescribed in \citep{reisach2021beware}, and it is applied by default rather than as a corrective.

\begin{figure}[h]
\centering
\includegraphics[width=1\linewidth]{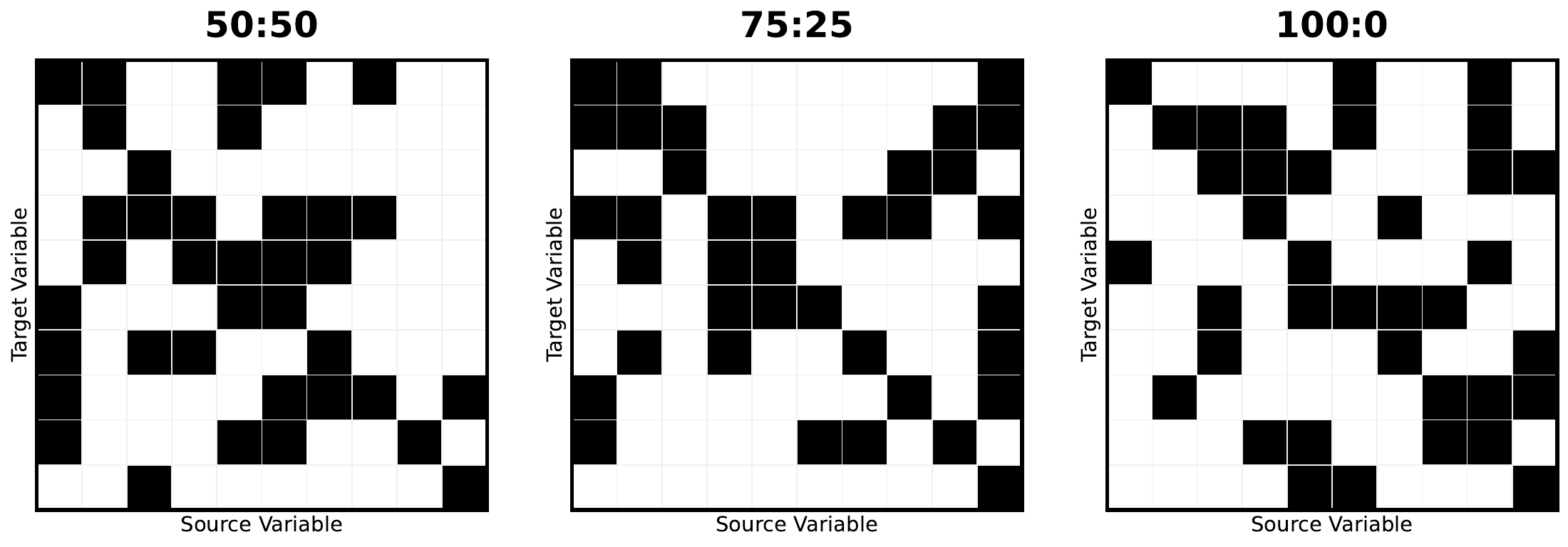}
\caption{Ground truth causal graphs for Mixed Physics system}
\label{fig:MP_ground_truth}
\end{figure}

\subsection{Rhino Erd\H{o}s--R\'enyi Synthetic Benchmark}
\label{app:rhino_er_data}

\textbf{Data source and generation.}
We use the Erd\H{o}s--R\'enyi (ER) synthetic benchmark introduced by
\citet{gong2022rhino}. Each dataset is generated in four steps:
\begin{itemize}
    \item Lag-specific directed graphs are sampled as ER random graphs with $2N$ edges per lag.
    \item Structural causal functions are parameterised as two-hidden-layer MLPs (64 units, ReLU).
    \item In the history-dependent noise setting, lagged-parent values are averaged, passed through a rational-quadratic spline, and used to scale an independent Gaussian innovation (noise level $0.5$).
    \item Trajectories are simulated from random initial states, and an initial burn-in period is discarded.
\end{itemize}
All data are generated using the official Causica implementation at commit
\texttt{b3c79a0}.\footnote{\url{https://github.com/microsoft/causica/tree/b3c79a01f30f44ed36c582ffe2b4522058d82a73}}

\textbf{Evaluation and ground truth.}
From the full benchmark, we select the 20 ER datasets with no instantaneous
effects and history-dependent noise, spanning $N\in\{5,10,20,40\}$ variables
and five generator seeds. We retain only the observational training
trajectories and do not concatenate across trajectory boundaries.
We set the temporal order to $L=2$. We evaluate against a
\emph{summary graph}: an edge $j\to i$ is present whenever $j$ causes $i$
at lag 1, lag 2, or both:
\begin{equation}
    A_{ij}=\mathbb{I}\!\left(A^{(1)}_{ij}+A^{(2)}_{ij}>0\right).
    \label{eq:rhino_er_summary}
\end{equation}
Figure~\ref{fig:rhino_er_ground_truth} shows representative ground-truth
summary graphs for generator seed 1.

\begin{figure*}[t]
    \centering
    \begin{minipage}{0.35\textwidth}
    \small $N=5$
        \centering
        \includegraphics[width=\linewidth]{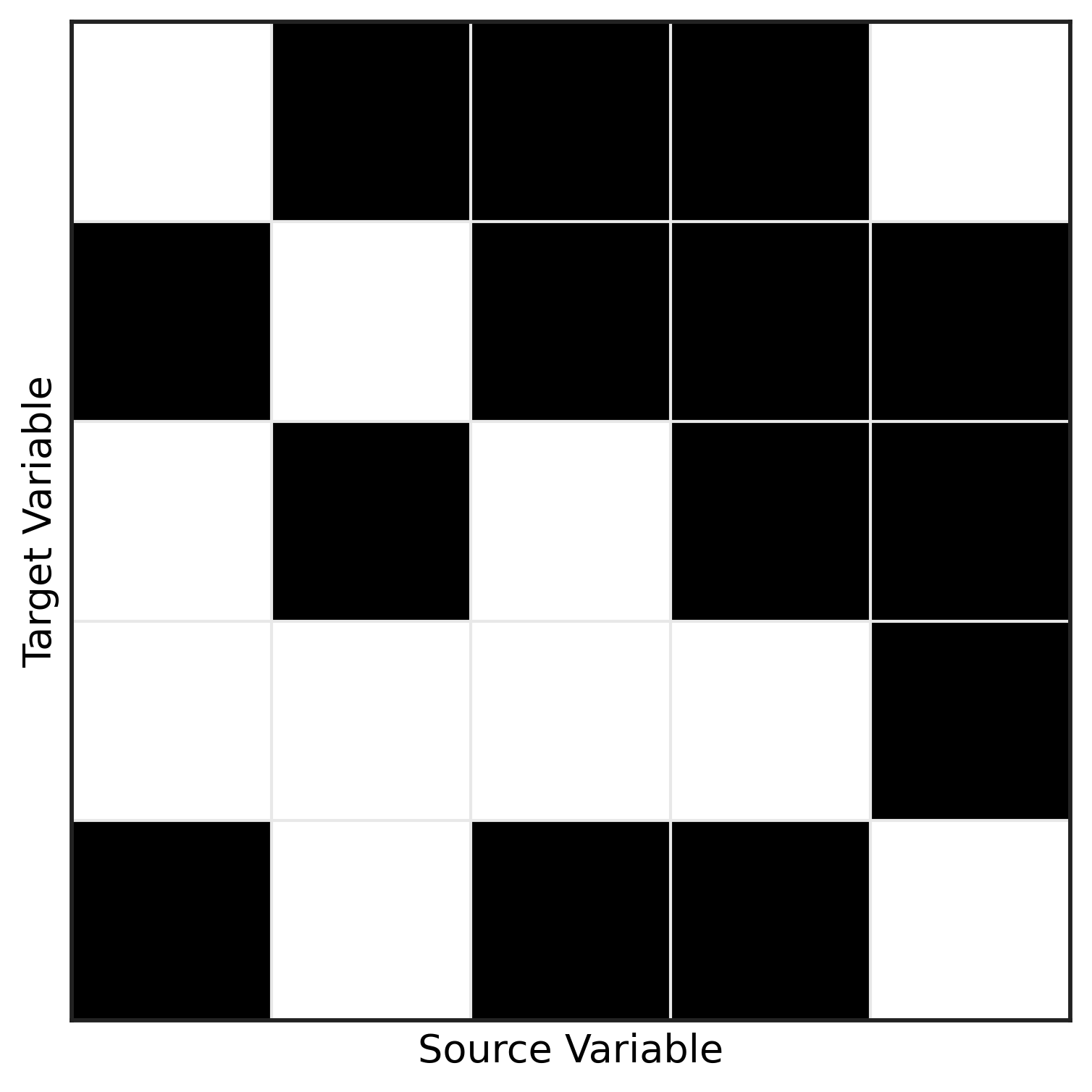}\\[-2mm]
        
    \end{minipage}
    \begin{minipage}{0.35\textwidth}
    \small $N=10$
        \centering
        \includegraphics[width=\linewidth]{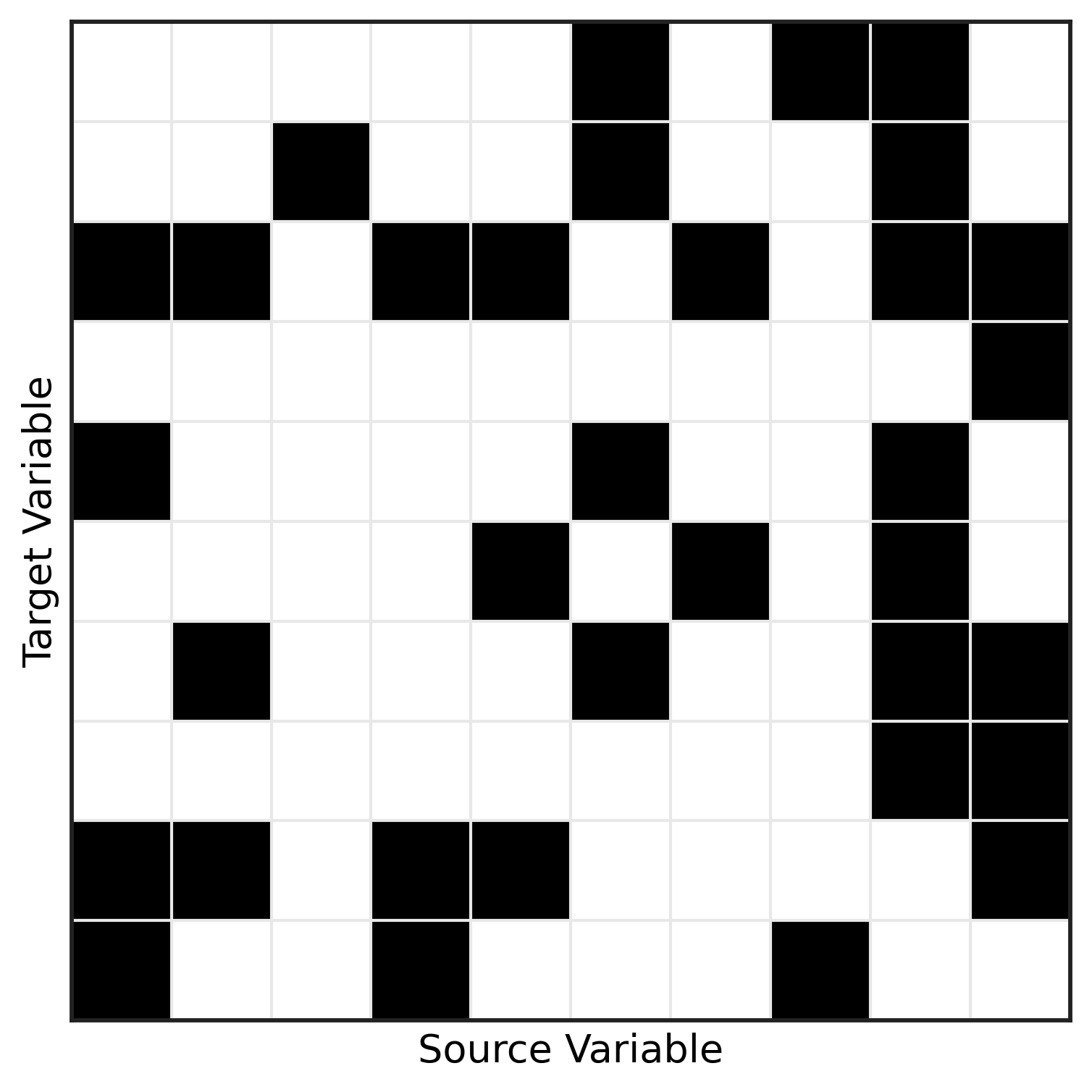}\\[-2mm]
        
    \end{minipage}

    \vspace{3mm}

    \begin{minipage}{0.35\textwidth}
        \centering
        \small $N=20$
        \includegraphics[width=\linewidth]{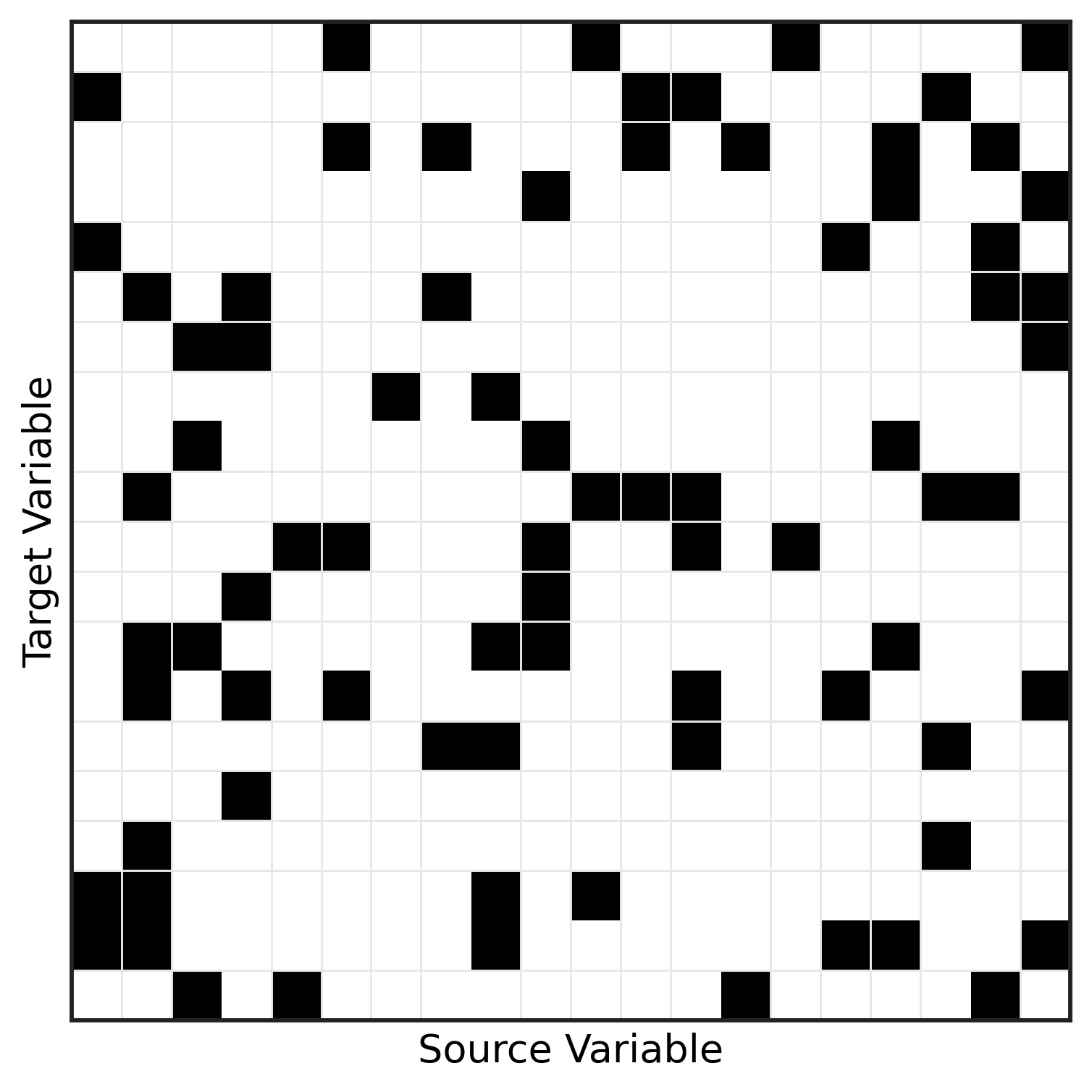}\\[-2mm]
        
    \end{minipage}
    \begin{minipage}{0.35\textwidth}
        \centering
        \small $N=40$
        \includegraphics[width=\linewidth]{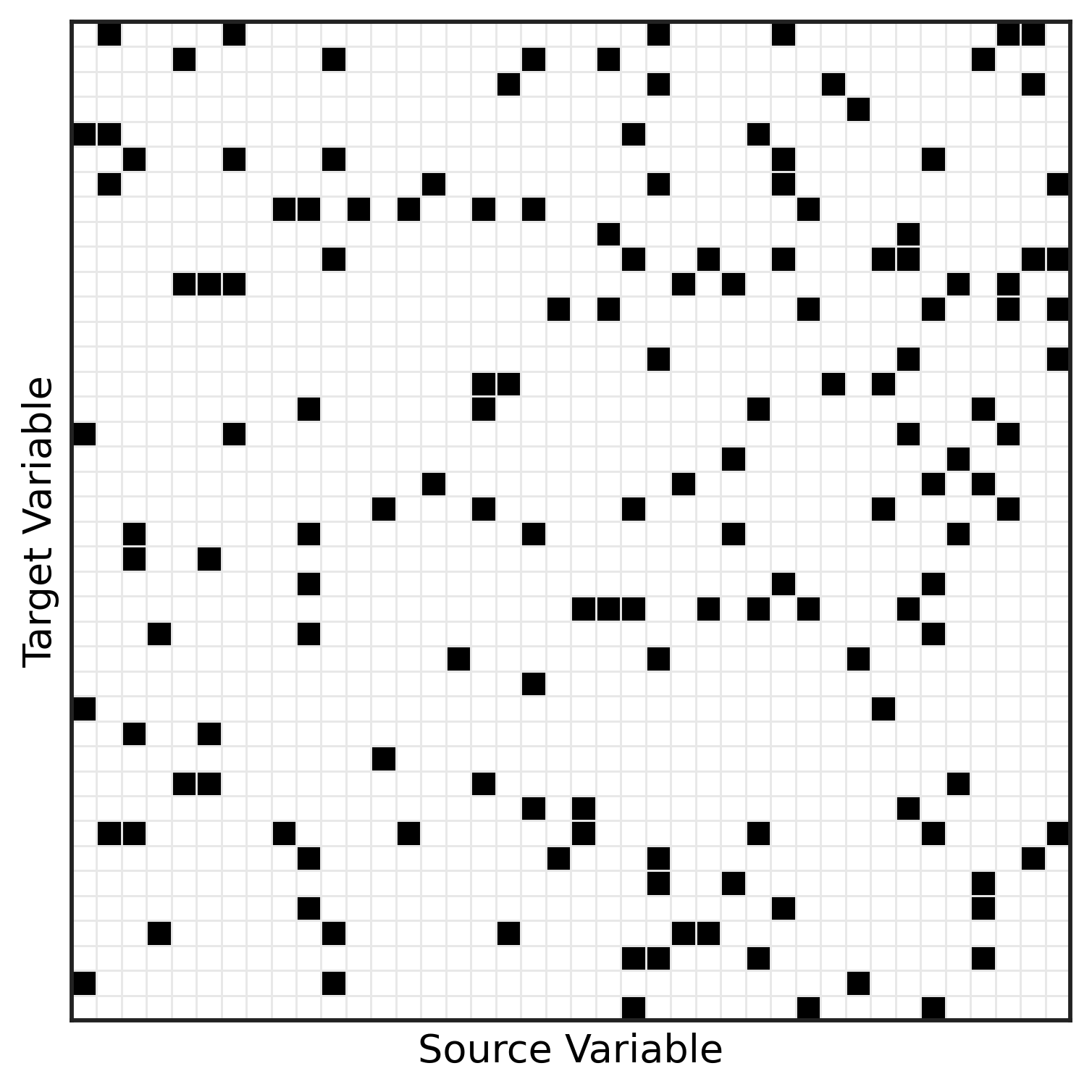}\\[-2mm]
        
    \end{minipage}
    \caption{Ground-truth summary graphs for the Rhino ER benchmark
    (generator seed 1). Rows represent target variables and columns
    represent source variables. A black cell denotes an edge at either lag.}
    \label{fig:rhino_er_ground_truth}
\end{figure*}

\section{Extended Evaluation Metrics: AUPRC, SHD, and F1}
\label{metrics}
While AUROC serves as a standard measure for causal discovery, the Area Under the Precision-Recall Curve (AUPRC) provides a highly sensitive metric for evaluating false positives. A high AUPRC indicates that true edges are concentrated
near the top of the ranking, making it informative
for sparse graphs.

While AUROC and AUPRC are calculated by sweeping a continuous threshold across $\hat{\mathbf{A}}$, evaluating Structural Hamming Distance (SHD) and Precision/Recall requires a discrete graph. To construct this discrete graph, we performed thresholding as described in the Appendix \ref{thresh}

As demonstrated in Table \ref{tab:extended_metrics}, Mask2Cause exhibits exceptionally strong performance across all metrics. On the synthetic benchmarks (VAR, Lorenz-96), the AUPRC nearly perfectly matches the AUROC (frequently achieving a flawless $1.000$), indicating near-perfect edge ranking. Furthermore, the model maintains robust precision-recall scores even on the highly complex, realistic proxies within the CausalTime suite (Medical, AQI, Traffic).

\begin{table}[h]
\centering
\caption{Extended evaluation metrics for Mask2Cause (MSE \& NLL) across synthetic and real-world benchmarks on a single seed}
\label{tab:extended_metrics}
\begin{tabular}{lcccc}
\toprule
\textbf{Dataset Configuration} & \textbf{AUROC} & \textbf{AUPRC} & \textbf{SHD} & \textbf{F1} \\
\midrule
Lorenz-96 ($F=10$), $T=500$, MSE    & 0.9967         & 0.9956         & 2            & 0.975 \\
Lorenz-96 ($F=40$), $T=500$, MSE    & 1.0000         & 1.0000         & 0            & 1.000 \\
Lorenz-96 ($F=10$), $T=250$, MSE    & 0.9950         & 0.9929         & 6            & 0.925 \\
Lorenz-96 ($F=40$), $T=250$, MSE   & 0.9987         & 0.9981         & 2            & 0.975 \\
VAR, $T=1000$, MSE                  & 1.0000         & 1.0000         & 0            & 1.000 \\
VAR, $T=500$, MSE                  & 1.0000         & 1.0000         & 0            & 1.000 \\
\midrule
Medical, MSE                        & 0.9030         & 0.8837         & 74           & 0.758 \\
AQI, MSE                           & 0.8980         & 0.7462         & 188           & 0.735 \\
Traffic, MSE                        & 0.6810         & 0.5355         & 54           & 0.671 \\
\midrule
Mixed Physics (50:50), NLL          & 0.9976         & 0.9943         & 2            & 0.9630 \\
Mixed Physics (75:25), NLL          & 1.0000         & 1.0000         & 0            & 1.0000 \\
Mixed Physics (100:0), NLL          & 1.0000         & 1.0000         & 0            & 1.0000 \\
RHINO synthetic (5), NLL & 0.9405& 0.9735 & 2 & 0.9286\\
RHINO synthetic (10), NLL & 0.9228& 0.9116 & 12 & 0.8235\\
RHINO synthetic (20), NLL & 0.8907& 0.7803 & 50 & 0.6795\\
RHINO synthetic (40), NLL & 0.9030 & 0.7139 & 102 & 0.6710\\
\bottomrule
\end{tabular}
\end{table}

\section{Scalability on High-Dimensional Systems}
\label{scaling}

In the main text experiments (Section 5), we evaluated Mask2Cause on synthetic systems (VAR and Lorenz-96) with $N=10$ variables. This dimension was selected to mirror the established experimental protocols of our primary baselines (e.g., eSRU \cite{Khanna2020Economy} and CUTS \cite{cheng2023cuts}), ensuring a fair and direct comparison. As shown in Table \ref{tab: lorvar}, Mask2Cause achieves near-perfect recovery in this regime.

To rigorously evaluate the model's structural recovery capabilities at scale and explicitly address the ceiling effects observed at smaller dimensions, we extended our evaluation of the Mask2Cause(MSE) variant to $N=20$ and $N=100$ systems (with $T=1000$). 

As detailed in Table \ref{tab:high_dim_scaling}, the metric ceiling effect persists at $N=20$, where the model still easily recovers the underlying dynamics. However, increasing $N$ from 10 to 100 increases the number of candidate causal links quadratically from 100 to 10,000. This explicitly stresses the adjacency-constrained encoder. Despite this massive increase in structural complexity, Mask2Cause remains highly robust, maintaining an AUROC between 0.87 and 0.95 across both linear and chaotic continuous-time dynamics. Table \ref{tab:results_baselines_t1000_p} presents results of evaluation of other competitive baselines (CUTS+, cMLP, cLSTM) on high-dimensional VAR and Lorenz-96 datasets.

\begin{table}[htbp]
    \centering
    \small
    \caption{\textbf{Mask2Cause scaling performance.} Evaluation of the MSE variant on high-dimensional VAR and Lorenz-96 systems ($T=1000$).}
    \label{tab:high_dim_scaling}
    \begin{tabular}{lcccc}
        \toprule
        \textbf{Dataset} & \textbf{System Size ($N$)} & \textbf{AUROC}  & \textbf{AUPRC}  \\
        \midrule
        Lorenz-96 ($F=10$) & 10 & $1.000 \pm 0.000$ & $1.000 \pm 0.000$\\
        Lorenz-96 ($F=40$) & 10 & $1.000 \pm 0.000$ & $1.000\pm0.000$ \\
        VAR & 10 & $1.000 \pm 0.000$ & $1.000 \pm 0.000$ \\
        \midrule
        Lorenz-96 ($F=10$) & 20 & $1.000 \pm 0.000$ & $1.000 \pm 0.000$\\
        Lorenz-96 ($F=40$) & 20 & $1.000 \pm 0.000$ & $0.998 \pm 0.001$ \\
        VAR & 20 & $0.997 \pm 0.001$ & $0.987 \pm 0.004$\\
        \midrule
        Lorenz-96 ($F=10$) & 100 & $0.952 \pm 0.064$ & $0.871 \pm 0.092$\\
        Lorenz-96 ($F=40$) & 100 & $0.883 \pm 0.049$ & $0.772 \pm 0.084$\\
        VAR & 100 & $0.867\pm0.078$ & $0.680 \pm 0.095$\\
        \bottomrule
    \end{tabular}
\end{table}

\begin{table*}[htbp]
\centering
\caption{\textbf{Baseline scaling performance.} Evaluation of competitive baselines on high-dimensional VAR and Lorenz-96 systems ($T=1000$) (mean $\pm$ std)}
\label{tab:results_baselines_t1000_p}
\resizebox{\textwidth}{!}{
\begin{tabular}{ll cccccc}
\toprule
& & \multicolumn{2}{c}{\textbf{LORENZ-96 (F=10)}} & \multicolumn{2}{c}{\textbf{LORENZ-96 (F=40)}} & \multicolumn{2}{c}{\textbf{VAR}} \\
\cmidrule(lr){3-4} \cmidrule(lr){5-6} \cmidrule(lr){7-8}
\textbf{Metric} & \textbf{Model} & $N=20$ & $N=100$ & $N=20$ & $N=100$ & $N=20$ & $N=100$ \\
\midrule
\multirow{4}{*}{\textbf{AUROC}}
& CUTS+ & $0.837 \pm 0.013$ & $0.781 \pm 0.012$ & $0.941 \pm 0.026$ & $0.826 \pm 0.041$ & $0.880 \pm 0.015$ & $0.601 \pm 0.122$ \\
& cMLP  & $0.982 \pm 0.007$ & $0.850 \pm 0.042$ & $0.992 \pm 0.004$ & $0.798 \pm 0.006$ & $0.540 \pm 0.025$ & $0.506 \pm 0.004$ \\
& cLSTM & $0.719 \pm 0.019$ & $0.522 \pm 0.021$ & $0.698 \pm 0.003$ & $0.533 \pm 0.011$ & $0.641 \pm 0.016$ & $0.508 \pm 0.006$ \\

\midrule
\multirow{4}{*}{\textbf{AUPRC}}
& CUTS+ & $0.625 \pm 0.054$ & $0.392 \pm 0.113$ & $0.875 \pm 0.036$ & $0.560 \pm 0.142$ & $0.760 \pm 0.039$ & $0.397 \pm 0.144$ \\
& cMLP  & $0.939 \pm 0.020$ & $0.326 \pm 0.085$ & $0.974 \pm 0.007$ & $0.201 \pm 0.003$ & $0.277 \pm 0.007$ & $0.292 \pm 0.001$ \\
& cLSTM & $0.377 \pm 0.076$ & $0.033 \pm 0.001$ & $0.307 \pm 0.042$ & $0.033 \pm 0.000$ & $0.347 \pm 0.016$ & $0.295 \pm 0.006$ \\

\bottomrule
\end{tabular}
}
\end{table*}

\section{Convergence Behaviour with Sample Size}
\label{convergence}

We prove no consistency or identifiability result for the learned adjacency matrix. To our knowledge no neural Granger method provides finite-sample recovery guarantees under regularization and thresholding, and we do not claim one. What we can offer is empirical evidence on how recovery behaves as the sample size grows.

We sweep the trajectory length on Lorenz-96 and VAR at $N=100$. The high dimension is deliberate, since at $N=10$ the benchmark saturates and any trend would be masked by the ceiling.

\begin{table}[htbp]
\centering
\small
\caption{\textbf{Recovery against trajectory length.} AUROC at $N=100$, where the search space holds $10^4$ candidate links.}
\label{tab:convergence}
\begin{tabular}{lcccc}
\toprule
\textbf{Dataset} & $T=500$ & $T=1000$ & $T=2000$ & $T=5000$ \\
\midrule
Lorenz-96 ($F=10$) & $0.828 \pm 0.083$ & $0.952 \pm 0.064$ & $0.961 \pm 0.045$ & $0.978 \pm 0.024$ \\
Lorenz-96 ($F=40$) & $ 0.810 \pm 0.072$ & $0.883 \pm 0.049$ & $ 0.923 \pm 0.033$ & $0.955 \pm 0.042$\\
VAR & $0.791 \pm 0.089$ & $0.867 \pm 0.078$ & $0.917 \pm 0.041$ & $0.932 \pm 0.040$\\
\bottomrule
\end{tabular}
\end{table}

Recovery improves monotonically with the number of samples and shows no sign of degradation or oscillation, which is consistent with convergence toward the true graph. Appendix \ref{limit} states the limitation.

\section{Computational Complexity Analysis}
\label{D}

We evaluate the efficiency of Mask2Cause using hardware-independent metrics: Parameter Count and Floating Point Operations (FLOPs). We conduct a single-sample inference sweep (Batch Size=1) to quantify scaling behavior with respect to history length $L$ and system size $N$.

\textbf{Scaling w.r.t Sequence Length ($L$).}
While standard temporal Transformers scale quadratically with sequence length ($O(L^2)$), our \textit{Inverted Variable Embedding} strategy decouples the computational cost from the temporal window. Figure \ref{fig:complexity} (Left) confirms that Mask2Cause scales \textbf{linearly with history length}. Remarkably, increasing $L$ from 5 to 2000 results in a negligible increase in computational cost (2.05M to 4.60M FLOPs). This efficiency arises because the sequence dimension is projected into a fixed latent vector $d_{model}$ at the input layer; subsequent attention layers operate solely on the variable dimension, independent of $L$.

\textbf{Scalability w.r.t System Size ($N$).}
Self-attention over variables implies $O(N^2)$ complexity. This is reflected in Figure \ref{fig:complexity} (Right) where we evaluate the computational cost for system sizes up to $N=2000$. This behavior occurs because the quadratic attention term ($N^2 \cdot d_{model}$) dominates the position-wise Feed-Forward Networks (FFN) and linear projections, which scale as $O(N \cdot d_{model}^2)$. 

\textbf{Resource Footprint.}
%
Table~\ref{tab:comparative_complexity_with_nl} demonstrates the efficiency of our architecture while exposing fundamental scaling bottlenecks in baseline models. Because cLSTM and CUTS+ are recurrent architectures and employ temporal weight-sharing, their parameter counts remain constant as the sequence length $L$ varies. In contrast, cMLP flattens the history window, causing a huge increase in parameter count as $L$ scales from $5$ to $2000$. While all models inevitably confront an $\mathcal{O}(N^2)$ complexity bound regarding variable count $N$, their failure modes vary. Neural Granger methods (cMLP and cLSTM) become intractable at scale due to excessive parameter growth. CUTS+ remains parameter-efficient via spatial weight-sharing, but its graph-based message-passing mechanism incurs severe computational overhead, evident in its massive FLOP count at $N=2000, L=2000$. RHINO is compact at $N=10, L=5$, requiring 0.0323M parameters and 1.80M FLOPs. However, at $N=2000, L=5$, it grows to 70.6584M parameters and 8275.23M FLOPs, compared with 4.1005M parameters and 2511.10M FLOPs for Mask2Cause (NLL). Its history-length cost also rises sharply: at $N=10, L=2000$, RHINO requires 555.20M FLOPs versus 4.60M for Mask2Cause. In a separate single-GPU runtime comparison on the history-dependent RHINO ER dataset ($N=10$, seed 1), complete runs excluding tuning took 42.04 minutes for RHINO, compared with 3.88 minutes for Mask2Cause (NLL), 3.67 minutes for Mask2Cause (MSE), 7.28 minutes for CUTS+, 12.88 minutes for cMLP, and 19.66 minutes for cLSTM. Mask2Cause circumvents these limits: it achieves parameter efficiency through shared Transformer weights and computational efficiency via global attention mechanisms. On the standard Lorenz-96 configuration ($N=10, L=5$), Mask2Cause requires only \textbf{0.10M parameters} and \textbf{2.05M FLOPs}, with this relative scaling advantage compounding rapidly at higher values of $N$ and $L$. Figures~\ref{fig:complexity} and~\ref{fig:unified_complexity} support this analysis, confirming the superior scalability of the proposed framework.

\begin{table}[htbp]
\centering
\begin{minipage}{0.48\textwidth}
\centering
\textbf{Configuration: $N=10$, $L=5$} \\ \vspace{0.2cm}
\resizebox{\linewidth}{!}{%
\begin{tabular}{lrr}
\toprule
\textbf{Model} & \textbf{Params (M)} & \textbf{FLOPs (M)} \\
\midrule
cMLP (NGC) & $0.0333$ & $0.13$ \\
cLSTM (NGC) & $0.1952$ & $2.00$ \\
CUTS+ & $0.1623$ & $16.37$ \\
RHINO & $0.0323$ & $1.80$ \\
M2C (MSE) & $0.1005$ & $2.05$ \\
M2C (NLL) & $0.1006$ & $2.05$ \\
\bottomrule
\end{tabular}%
}
\end{minipage}\hfill
\begin{minipage}{0.48\textwidth}
\centering
\textbf{Configuration: $N=2000$, $L=5$} \\ \vspace{0.2cm}
\resizebox{\linewidth}{!}{%
\begin{tabular}{lrr}
\toprule
\textbf{Model} & \textbf{Params (M)} & \textbf{FLOPs (M)} \\
\midrule
cMLP (NGC) & $1280.2580$ & $5120.51$ \\
cLSTM (NGC) & $1057.9220$ & $10589.44$ \\
CUTS+ & $25.2530$ & $499939.84$ \\
RHINO & $70.6584$ & $8275.23$ \\
M2C (MSE) & $4.1004$ & $2510.85$ \\
M2C (NLL) & $4.1005$ & $2511.10$ \\
\bottomrule
\end{tabular}%
}
\end{minipage}

\vspace{0.5cm}

\begin{minipage}{0.48\textwidth}
\centering
\textbf{Configuration: $N=10$, $L=2000$} \\ \vspace{0.2cm}
\resizebox{\linewidth}{!}{%
\begin{tabular}{lrr}
\toprule
\textbf{Model} & \textbf{Params (M)} & \textbf{FLOPs (M)} \\
\midrule
cMLP (NGC) & $12.8013$ & $51.20$ \\
cLSTM (NGC) & $0.1952$ & $801.28$ \\
CUTS+ & $0.1623$ & $6549.68$ \\
RHINO & $1.0298$ & $555.20$ \\
M2C (MSE) & $0.2282$ & $4.60$ \\
M2C (NLL) & $0.2283$ & $4.60$ \\
\bottomrule
\end{tabular}%
}
\end{minipage}\hfill
\begin{minipage}{0.48\textwidth}
\centering
\textbf{Configuration: $N=2000$, $L=2000$} \\ \vspace{0.2cm}
\resizebox{\linewidth}{!}{%
\begin{tabular}{lrr}
\toprule
\textbf{Model} & \textbf{Params (M)} & \textbf{FLOPs (M)} \\
\midrule
cMLP (NGC) & $512000.2580$ & $2048000.51$ \\
cLSTM (NGC) & $1057.9220$ & $4235776.00$ \\
CUTS+ & $25.2530$ & $199975936.00$ \\
RHINO & $24090.4584$ & $2834472.03$ \\
M2C (MSE) & $4.2281$ & $3021.57$ \\
M2C (NLL) & $4.2282$ & $3021.82$ \\
\bottomrule
\end{tabular}%
}
\end{minipage}
\vspace{3mm}
\caption{Complexity analysis: parameters and FLOPs across different configurations.}
\label{tab:comparative_complexity_with_nl}
\end{table}


\begin{figure}[h]
    \centering
    \includegraphics[width=0.8\linewidth]{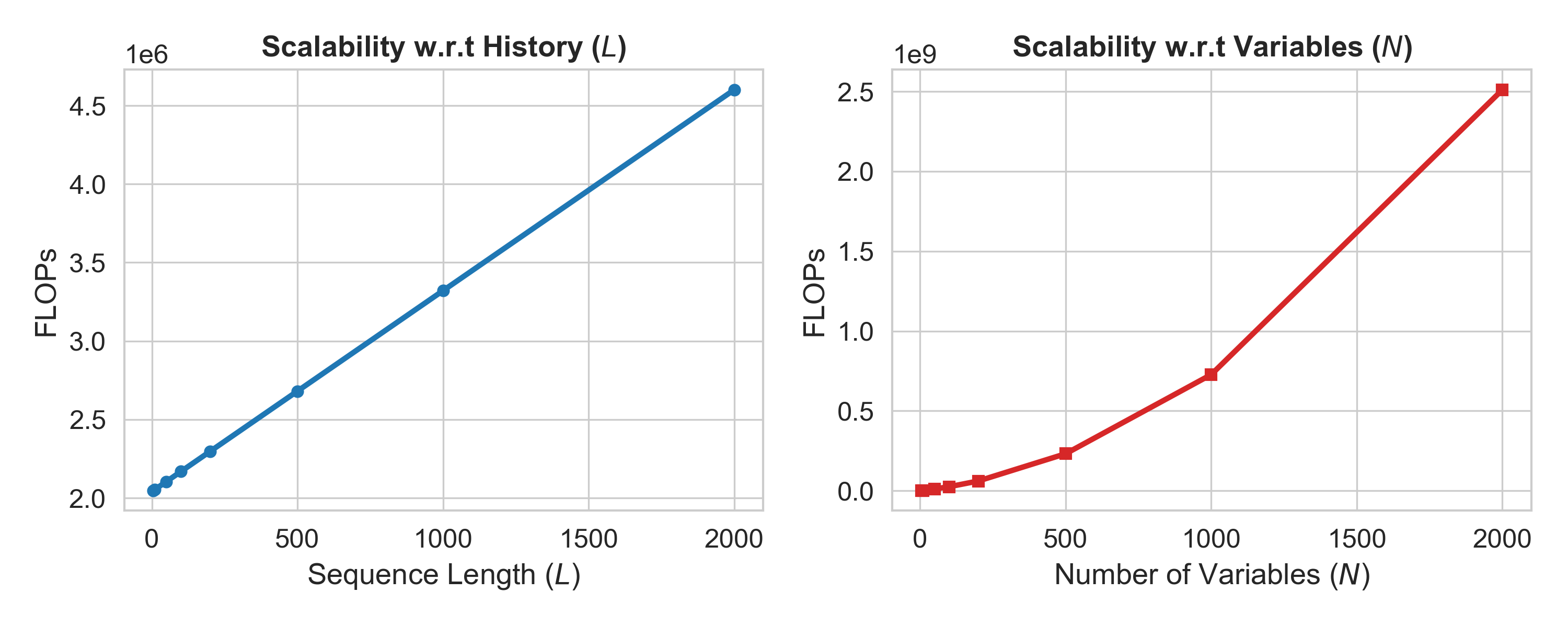} 
    \caption{\textbf{Complexity scaling (MSE).} (Left) FLOPs vs. sequence length $L$. The model is virtually insensitive to the look-back window; increasing history from $L=5$ to $L=2000$ only increases cost from 2.05M to 4.60M FLOPs. (Right) FLOPs vs. variables $N$. The cost grows significantly with system size, spanning from 1.01M ($N=5$) to nearly 2510.85M ($N=2000$), driven by the variable-wise projections.}
    \label{fig:complexity}
\end{figure}

\begin{figure}[h]
    \centering
    \includegraphics[width=0.8\linewidth]{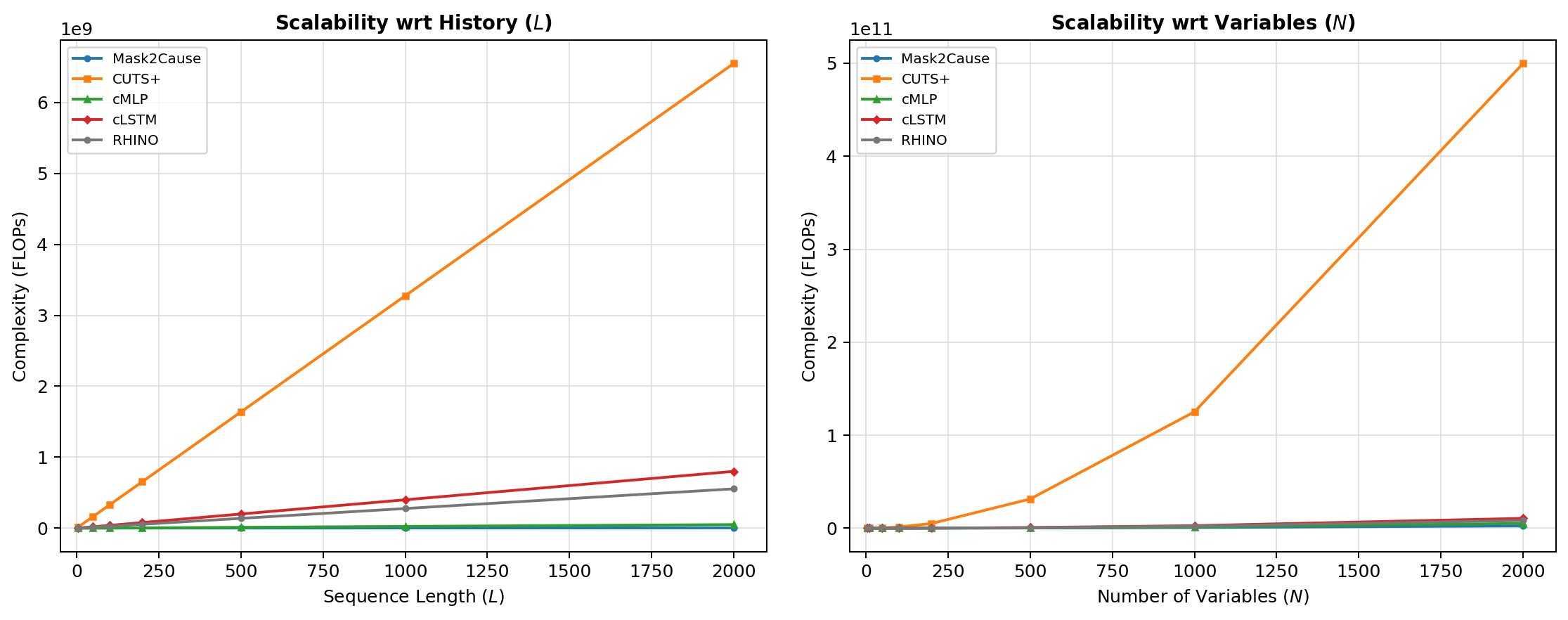} 
    \caption{\textbf{Comparison of complexity scaling.} (Left) FLOPs vs. sequence length $L$ (with $N=10$ fixed). Mask2Cause demonstrates superior scalability with respect to sequence length; increasing history from $L=5$ to $L=2000$ only increases cost from 2.05M to 4.60M FLOPs. In contrast, baselines exhibit significantly steeper linear scaling: cMLP (0.13M to 51.20M), cLSTM (2.00M to 801.28M), CUTS+ (16.37M to 6549.68M), and RHINO (1.80M to 555.20M). (Right) FLOPs vs. variables $N$ (with $L=5$ fixed). While all models exhibit $O(N^2)$ quadratic scaling, Mask2Cause remains the most computationally efficient at large dimensions. Its cost grows from 1.01M ($N=5$) to 2510.85M ($N=2000$), substantially outperforming cMLP (5120.51M), cLSTM (10589.44M), CUTS+ (499939.84M), and RHINO (8275.23M).}
    \label{fig:unified_complexity}
\end{figure}


\subsection{Hardware and Compute Resources}
\label{compute}
All experiments, including hyperparameter tuning and baseline evaluations, were conducted on a standard workstation equipped with an Intel Core Ultra 9 285H processor (32 GB of RAM) and two RTX A5000 GPU.

\section{Implementation Details and Hyperparameter Analysis}
\label{C}
\subsection{Data Sources and Evaluation Setup}
\label{evalset}
We utilize standard dataset implementations sourced directly from the official repository of the eSRU \cite{Khanna2020Economy} framework (for VAR, Lorenz-96, and DREAM3) and the official CausalTime \cite{chengcausaltime} website.

To ensure fair comparison and reproducibility, we adopt specific evaluation strategies tailored to the nature of each benchmark:

\textbf{1. Synthetic Generative Systems (VAR, Lorenz-96)}
We strictly separate hyperparameter tuning from final evaluation to prevent leakage. For each physical configuration (e.g., forcing constant $F$), we use the available 6 independent dataset realizations using distinct random seeds. Seed 0 is used exclusively as a \textit{Calibration Set} for hyperparameter tuning. Seeds 1--5 are reserved as \textit{Test Sets}. We report the mean and standard deviation of the AUROC computed across the 5 independent Test Sets.

\textbf{2. CausalTime (Static Real-World Proxies)}
For the fixed CausalTime datasets, where new samples cannot be generated, we employ a chronological split, using the first 20\% of the data for tuning. To account for the variance inherent in neural network initialization, we train the model 5 times on the same dataset using different random seeds for weight initialization. Results are reported as the mean and standard deviation across these 5 runs.

\textbf{3. DREAM3 (Gene Regulatory Networks)}
Consistent with the evaluation protocol of the baselines we compare against (which utilize the single fixed dataset provided by the challenge), we do not perform multi-seed averaging for this benchmark. We use the first 20\% of the data for tuning. We report the final AUROC from the single best model found after tuning on the validation set, matching the reporting style of the cited baselines.

\textbf{4. Mixed Physics.} Each variance-to-mean regime uses a fixed dataset. We report mean AUROC and standard deviation across five seeded model runs on that same dataset. The seeds vary model initialization and other seed-dependent training randomness; the selected hyperparameters remain fixed across runs. Thus, the reported variability reflects model-run variability, not independently generated datasets.

\textbf{Baseline Configurations.}
For baseline methods that we ran locally (rather than quoting from published literature), we strictly utilized the hyperparameter configurations specified for each respective benchmark in their original publications. If that is not available (for the case of mixed physics), we tune the hyper-parameters ourselves following the guidelines established by the respective papers.

\subsection{Architectural Configurations}
\textbf{Fixed Parameters.}
Through preliminary ablation on the Calibration Sets, we identified a globally robust architecture that performs well across diverse dynamical regimes. We fix the number of encoder layers to $M=2$ and the number of attention heads to $H=4$. Deeper architectures were observed to overfit the noise in sparse causal discovery tasks. The embedding dimension usually $d_{model}$ set to 64 by default, but reduced to 32 or 16 for datasets with extremely short sequences to strictly constrain model capacity.

\textbf{Diagonal Forcing.}
The diagonal forcing parameter controls the prior on self-loops. For autoregressive systems (VAR, Lorenz), a variable's immediate history is its strongest predictor. We force the diagonal logits to $+100$ to ensure the model always attends to itself, stabilizing the learning of cross-variable interactions. For the DREAM3 benchmark, the standard evaluation protocol explicitly excludes self-regulation edges from the ground truth. So, we remove the diagonal component from the predicted adjacency matrix before evaluating. Because of this we are no longer forced to keep the diagonals open, allowing us to include "Diag Force" for tuning, taking values \{-100, 0, 100\}. 

\subsection{Optimal Hyperparameters}
We performed a grid search on the Calibration/Validation sets to determine the optimal learning rate (`lr'), sparsity penalty ($\lambda$), and training duration (`epochs'). The selected configurations for our Homoscedastic (MSE) and Heteroscedastic (NLL) models are detailed below.

\begin{table}[h]
    \centering
    \caption{\textbf{Optimal hyperparameters for M2C-MSE}. CausalTime settings use seed 0; initial graph logits are $+1$ for Traffic and AQI and $-1$ for Medical.}
    \label{tab:params_mse}
    \resizebox{0.9\textwidth}{!}{
        \begin{tabular}{lccccccc}
            \toprule
            \textbf{Dataset} & \textbf{Learning Rate} & \textbf{Batch Size} & \textbf{Seq Len} & \textbf{$d_{model}$} & \textbf{$\lambda$} & \textbf{Diag Force} & \textbf{Epochs} \\
            \midrule
            VAR ($T=500$) & 0.001 & 32 & 3 & 64 & 0.01 & +100 & 10\\
            VAR ($T=1000$) & 0.001 & 32 & 3 & 64 & 0.01 & +100 & 5\\
            Lorenz-96 ($F=10, T = 250$ ) & 0.01 & 32 & 1 & 32 & 0.02 & +100 & 200\\
            Lorenz-96 ($F=10, T = 500$ ) & 0.001 & 32 & 1 & 64 & 0.02 & +100 & 150\\
            Lorenz-96 ($F=40, T = 250$) & 0.01 & 32 & 1 & 64 & 0.02 & +100 & 150\\
            Lorenz-96 ($F=40, T = 500$) & 0.001 & 32 & 1 & 64 & 0.02 & +100 & 150\\
            DREAM3 (E.coli-1) & 0.001 & 32 & 5 & 32 & 0.5 & -100 & 20\\
            DREAM3 (E.coli-2) & 0.001 & 32 & 5 & 32 & 0.005 & +100 & 65\\
            DREAM3 (Yeast-1) & 0.001 & 16 & 5 & 64 & 0.0005 & +100 & 35\\
            DREAM3 (Yeast-2) & 0.001 & 32 & 5 & 32 & 0.001 & 0 & 20\\
            DREAM3 (Yeast-3) & 0.001 & 16 & 5 & 64 & 0.5 & -100 & 20\\
            CausalTime (Traffic) & 0.01 & 32 & 3 & 64 & 0.5 & +100 & 17\\
            CausalTime (AQI) & 0.01 & 32 & 3 & 64 & 1.0 & +100 & 16\\
            CausalTime (Medical) & 0.01 & 32 & 3 & 64 & 1.0 & +100 & 35\\
            MixedPhysics (50:50) & 0.001 & 32 &3  & 32 & 0.01 & +100 &35 \\
            MixedPhysics (75:25) &0.0001  & 64 & 3 & 16 &  0.01& +100 & 5\\
            MixedPhysics (100:0) &0.001  & 16 & 3 & 16 & 0.001 & +100 & 5\\
            RHINO synthetic (5) & 0.001& 32 & 2 & 32 & 0.03 & +100 & 35 \\
            RHINO synthetic (10) & 0.001 & 32 & 2 & 32 & 0.003 & +100&35\\
            RHINO synthetic (20) & 0.001 & 32 & 2 & 32 & 0.03 & +100 & 35 \\
            RHINO synthetic (40) & 0.001 & 32 & 2 & 32 & 0.03 & +100 & 35 \\

            \bottomrule
        \end{tabular}
    }
\end{table}

\begin{table}[h]
    \centering
    \caption{\textbf{Optimal hyperparameters for M2C-NLL}. CausalTime settings use seed 0; and initial graph logits of $+1$.}
    \label{tab:params_nll}
    \resizebox{0.9\textwidth}{!}{
        \begin{tabular}{lccccccc}
            \toprule
            \textbf{Dataset} & \textbf{Learning Rate} & \textbf{Batch Size} & \textbf{Seq Len} & \textbf{$d_{model}$} & \textbf{$\lambda$} & \textbf{Diag Force} & \textbf{Epochs} \\
            \midrule
            VAR ($T=500$) & 0.001 & 32 & 3 & 64 & 0.01 & +100 & 10\\
            VAR ($T=1000$) & 0.001 & 32 & 3 & 64 & 0.01 & +100 & 5\\
            Lorenz-96 ($F=10, T = 250$ ) & 0.01 & 32 & 1 & 32 & 0.02 & +100 & 185\\
            Lorenz-96 ($F=10, T = 500$ ) & 0.001 & 32 & 1 & 64 & 0.02 & +100 & 140\\
            Lorenz-96 ($F=40, T = 250$) & 0.01 & 32 & 1 & 64 & 0.02 & +100 & 150\\
            Lorenz-96 ($F=40, T = 500$) & 0.001 & 32 & 1 & 32 & 0.02 & +100 & 30\\
            DREAM3 (E.coli-1) & 0.001 & 32 & 5 & 64 & 0.01 & -100 & 20\\
            DREAM3 (E.coli-2) & 0.001 & 32 & 5 & 64 & 1.00 & +100 & 3\\
            DREAM3 (Yeast-1) & 0.001 & 32 & 5 & 64 & 0.001 & +100 & 35\\
            DREAM3 (Yeast-2) & 0.0001 & 32 & 5 & 64 & 0.001 & +100 & 15\\
            DREAM3 (Yeast-3) & 0.01 & 16 & 1 & 16 & 0.1 & -100 & 2\\
            CausalTime (Traffic) & 0.01 & 32 & 3 & 64 & 0.001 & +100 & 11\\
            CausalTime (AQI) & 0.003 & 64 & 3 & 64 & 0.01 & +100 & 20\\
            CausalTime (Medical) & 0.01 & 32 & 3 & 64 & 0.5 & +100 & 15\\
            MixedPhysics (50:50) & 0.001 & 32 &3  & 32 & 0.01 & +100 & 75\\
            MixedPhysics (75:25) &0.003  & 32 & 1 & 32 &  0.01& +100 & 35\\
            MixedPhysics (100:0) &0.003  & 32 & 1 & 32 & 0.01 & +100 & 20\\
            RHINO synthetic (5) & 0.001& 128 & 5 & 32 & 0.0005 & +100 & 50 \\
            RHINO synthetic (10) & 0.001 & 32 & 2 & 32 & 0.003 & +100 & 35\\
            RHINO synthetic (20) & 0.003 & 128 & 5 & 64 & 0.003 & 0 & 43 \\
            RHINO synthetic (40) & 0.001 & 128 & 5 & 64 & 0.01 & 0 & 158 \\
            \bottomrule
        \end{tabular}
    }
\end{table}

\subsection{Sensitivity Analysis}
\label{sec:sensitivity}

To assess the robustness of Mask2Cause, we analyze the sensitivity of the AUROC metric to variations in four key hyperparameters: Sparsity Penalty ($\lambda_{L1}$), Sequence Length/Look back window ($L$), Latent Dimension ($d_{model}$), and Diagonal Forcing. We vary one parameter at a time while fixing the others to their optimal settings on the Lorenz-96 ($F=10$) calibration set.

\begin{figure}[h]
    \centering
    \begin{minipage}{0.48\textwidth}
        \centering
        \includegraphics[width=\linewidth]{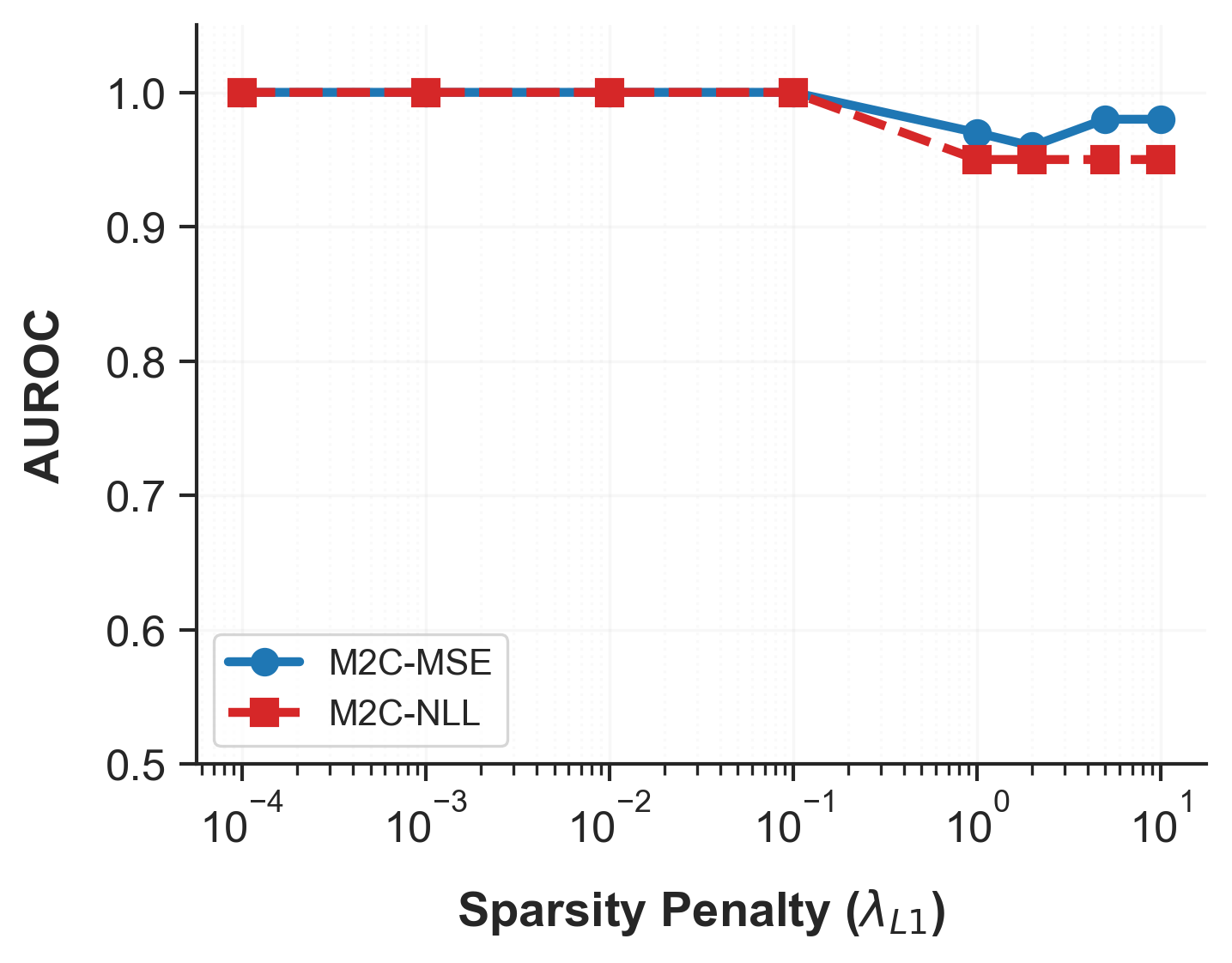} 
        \caption*{(a) Sparsity penalty $\lambda_{L1}$}
    \end{minipage}
    \hfill
    \begin{minipage}{0.48\textwidth}
        \centering
        \includegraphics[width=\linewidth]{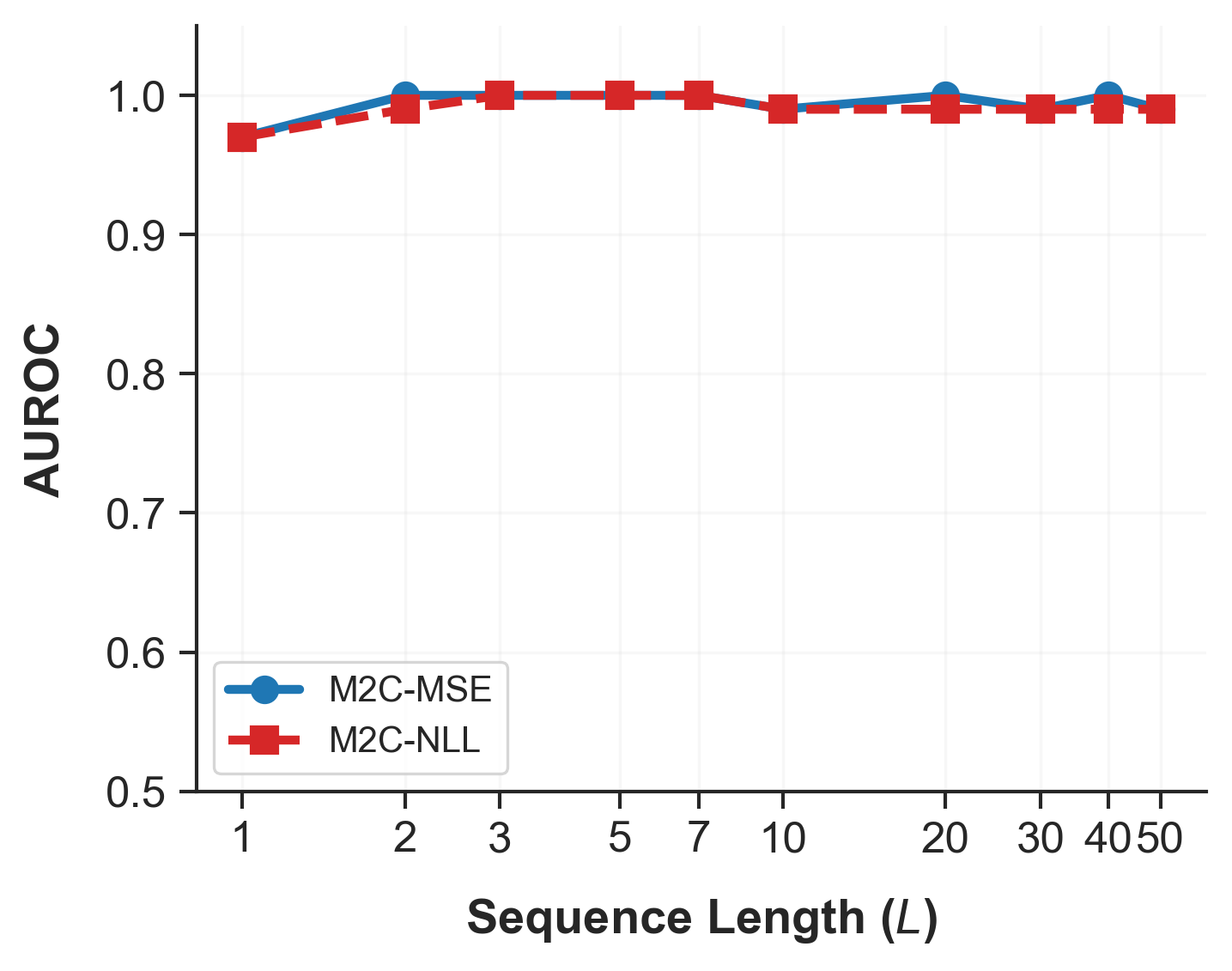} 
        \caption*{(b) Sequence length $L$}
    \end{minipage}
    
    \vspace{1em}
    
    \begin{minipage}{0.48\textwidth}
        \centering
        \includegraphics[width=\linewidth]{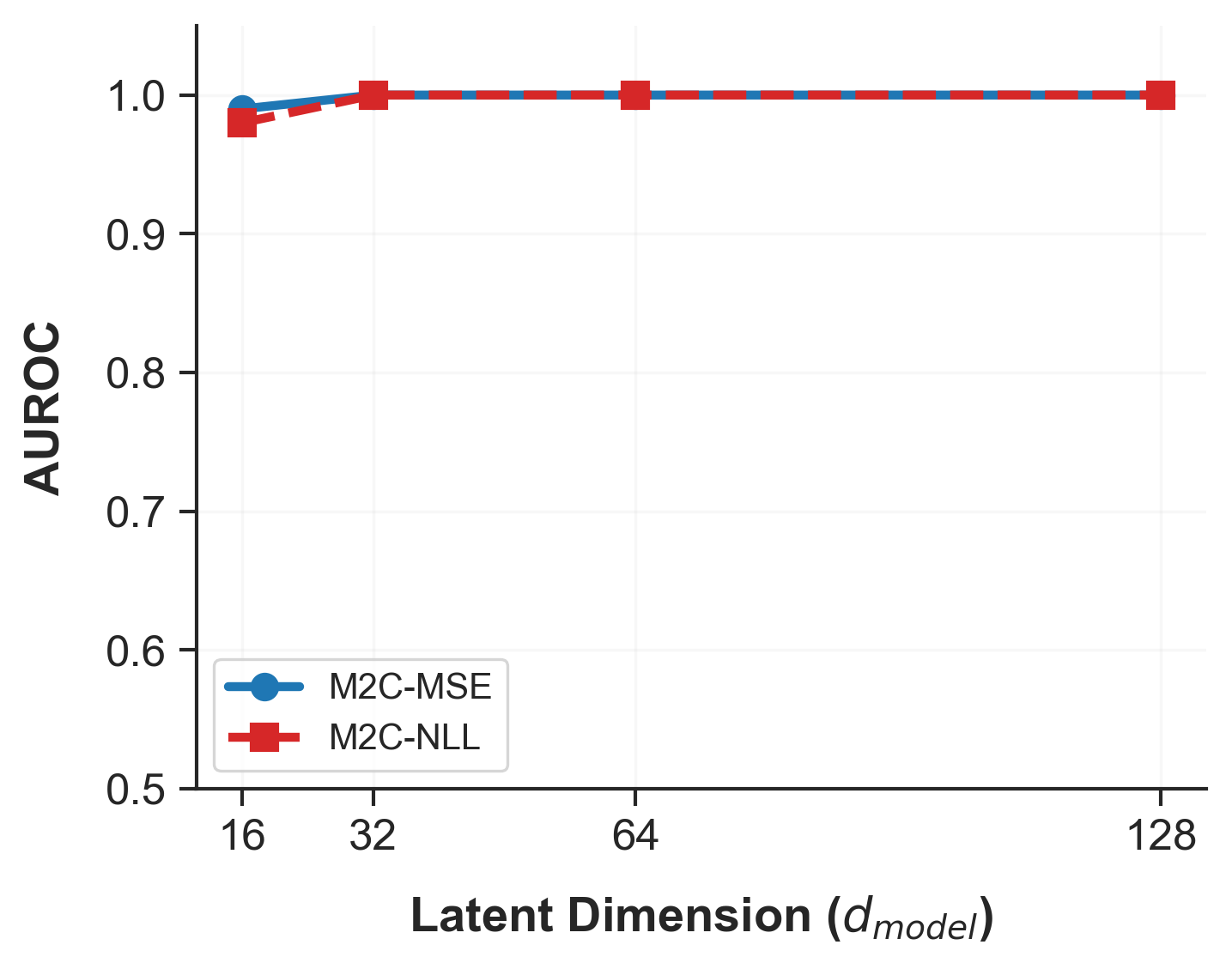} 
        \caption*{(c) Latent dimension $d_{model}$}
    \end{minipage}
    \hfill
    \begin{minipage}{0.48\textwidth}
        \centering
        \includegraphics[width=\linewidth]{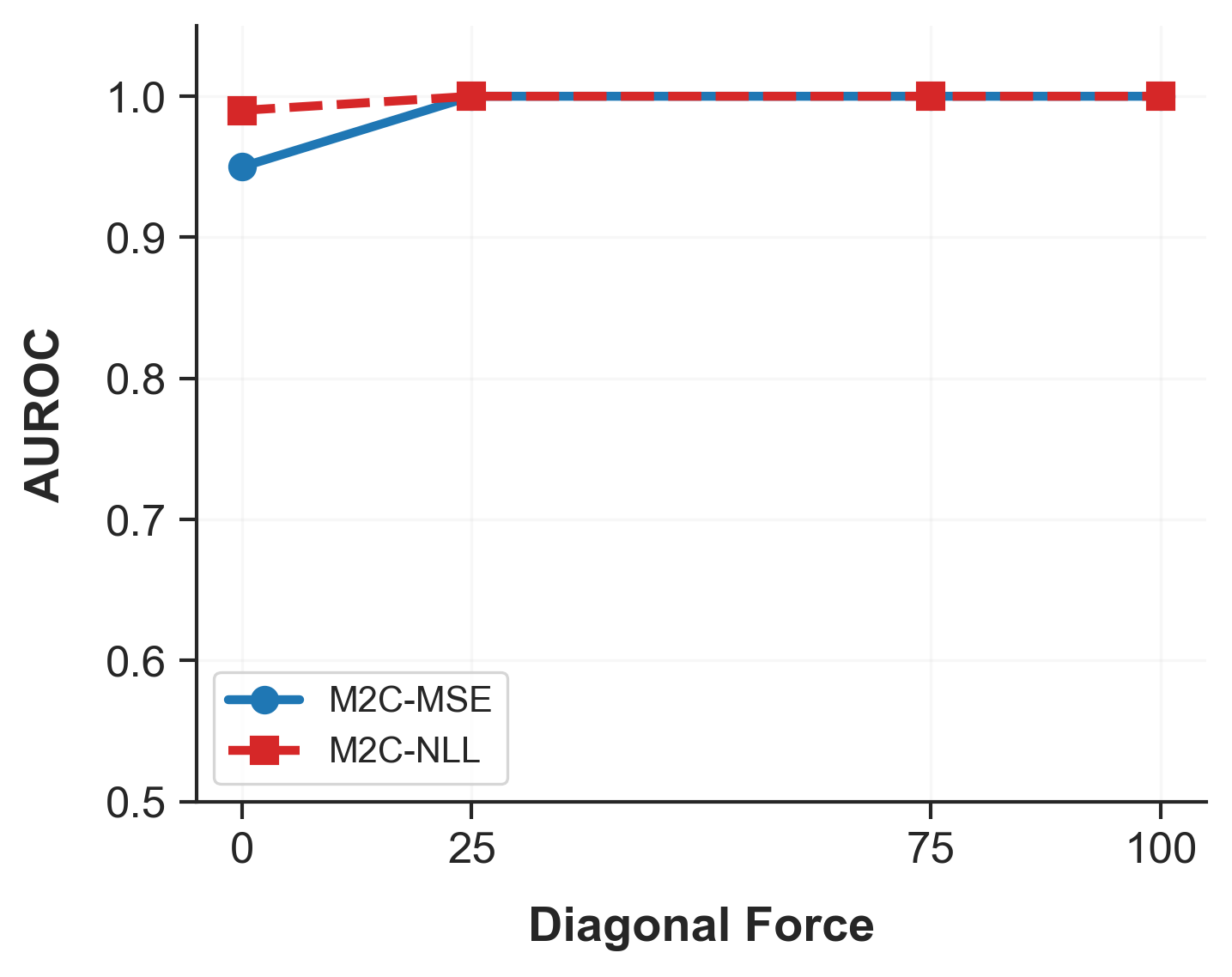} 
        \caption*{(d) Diagonal forcing}
    \end{minipage}
    
    \caption{\textbf{Hyperparameter sensitivity analysis.} The model exhibits high robustness across wide ranges of hyperparameters. Notably, (b) shows that performance remains near-perfect even as sequence length increases to $L=50$, demonstrating the embedding's ability to prioritize relevant recent history.}
    \label{fig:sensitivity}
\end{figure}

\textbf{Observations:}
\begin{itemize}
    \item \textbf{Sparsity Penalty ($\lambda$):} As shown in Panel (a), the model is exceptionally robust to the regularization strength, maintaining perfect recovery (AUROC 1.00) for $\lambda \in [10^{-4}, 10^{-1}]$. A slight degradation ($0.98$) is observed only when $\lambda \ge 1$, where the penalty begins to over-prune weak but true causal signals.
    
    \item \textbf{Sequence Length ($L$):} Panel (b) highlights a critical property of the Inverted Variable Embedding. The true lag of the system is small ($K \approx 3$). Remarkably, even when the input window is extended to $L=50$, the performance does not degrade ($0.99-1.00$). This confirms that the variable-specific tokenization effectively learns to attend to the relevant immediate history (the first few time steps) while suppressing noise from the distant past, without requiring manual lag selection.
    
    \item \textbf{Latent Dimension ($d_{model}$):} As seen in Panel (c), the model is insensitive to the hidden dimension size, achieving AUROC $\ge 0.99$ for $d \in [16, 128]$. This suggests that the causal graph structure is the primary information bottleneck, and even compact representations are sufficient to capture the dynamics.

    \item \textbf{Diagonal Forcing:} Panel (d) demonstrates that our model does not rely on diagonal forcing as a prerequisite for success. Crucially, even in the complete absence of this prior (Force = 0), the model maintains high recovery accuracy (AUROC $\ge$ 0.95 for MSE model, 0.99 for NLL model). This result confirms that diagonal forcing is not a ``cheat code'' or structural crutch; rather, the model is fully capable of autonomously learning autoregressive dependencies from scratch, with the prior serving merely to accelerate convergence in known autoregressive settings.
\end{itemize}

\section{Forecasting using Causal Pruning}
\label{forecasting}

\subsection{Forecasting Accuracy and Parameter Reduction}

We use the causal map discovered by \textbf{Mask2Cause} to perform \textit{Causal Pruning} at the input level and evaluate a forecaster's performance on a single-step-ahead objective.
%
%
For any target variable $X_{i, t+1}$, the model's receptive field (input) is restricted to its predicted causal parents and its own historical value, formally defined as $P(i) = \{X_{j, t} : M_{j,i} = 1\} \cup \{X_{i, t}\}$. This decomposition of forecasting loss minimization into $N$ univariate forecasting tasks ensures that the model isolates direct causal effects and discards spurious associations driven by conditionally independent variables. It also offers an added advantage of reducing the architectural complexity from $O(N)$ to $O(|P(i)|)$ per node. Model-specific implementation details follow below.

\textbf{Empirical Results and Comparative Analysis.} To validate our hypothesis that causal pruning can improve forecasting accuracy by avoiding the influence of spurious correlations and also improve resource optimization, we evaluate the parameter reduction achieved and the gain in the mean square error over the vanilla\footnote{The vanilla versions of forecasting models have access to history of all the variables for one-step-ahead prediction task.} versions of common forecasting models in Table \ref{tab:combined_results_modelwise}. We observe that in the DREAM3 and most Lorenz datasets, causal pruning achieves a better accuracy while reducing the computational overhead significantly. In VAR dataset, the parameter reduction comes at the cost of a slight drop in accuracy for some models.
Our findings show that causal pruning with \textbf{Mask2Cause} outperforms the \textbf{CUTS+} baseline for most datasets and models, indicating it discovers preciser matrices that provide a superior inductive bias (see Appendix \ref{CUTS+ forecasting appendix}) .
We also observe causal pruning performs competitively against alternative feature selection methods (see Appendix \ref{Comparison with other Feature Selection Methods}). Appendix \ref{rca} reports a second downstream task on real data, localising root causes in cloud microservice telemetry from the discovered graph.

\begin{table}[htbp]

    \centering
    \caption{Model-wise comparison of parameter reduction (PR) and MSE reduction (MSE-R) in \%}
    \label{tab:combined_results_modelwise}
    \scriptsize
    \setlength{\tabcolsep}{2pt}
    \renewcommand{\arraystretch}{0.75}
    \resizebox{\linewidth}{!}{%
    \begin{tabular}{lccccccccc}
        \toprule
        & \multicolumn{5}{c}{\textbf{DREAM3} ($N=100$)} & \multicolumn{2}{c}{\textbf{Lorenz}} & \multicolumn{2}{c}{\textbf{VAR}} \\
        \cmidrule(lr){2-6} \cmidrule(lr){7-8} \cmidrule(lr){9-10}
        \textbf{Model / Metric} & \textbf{Ecoli1} & \textbf{Ecoli2} & \textbf{Yeast1} & \textbf{Yeast2} & \textbf{Yeast3} & \textbf{F10} & \textbf{F40} & \textbf{T500} & \textbf{T1000} \\
        \midrule
        \textbf{ARIMAX} \\
        \quad PR (\%) & 99.0 & 99.0 & 99.0 & 99.0 & 98.9 & 60.69 & 58.50 & 72.20 & 70.50 \\
        \quad MSE-R (\%) & +40.01 & +36.70 & +36.35 & +38.11 & +41.15 & +2.34 & +1.20 & -5.11 & -5.62 \\
        \midrule
        \textbf{MLP} \\
        \quad PR (\%) & 97.0 & 97.0 & 96.9 & 97.0 & 96.9 & 50.44 & 48.62 & 60.01 & 58.60 \\
        \quad MSE-R (\%) & +31.67 & +33.81 & +31.91 & +33.74 & +33.13 & +38.48 & +5.44 & -4.41 & -4.06 \\
        \midrule
        \textbf{N-BEATS} \\
        \quad PR (\%) & 59.2 & 59.2 & 59.2 & 59.2 & 59.2 & 7.9 & 7.6 & 4.8 & 4.7 \\
        \quad MSE-R (\%) & +33.26 & +35.25 & +32.12 & +34.10 & +34.05 & +12.93 & -19.47 & +10.72 & +1.13 \\
        \midrule
        \textbf{Linear Reg.} \\
        \quad PR (\%) & 98.0 & 98.0 & 97.9 & 98.0 & 97.9 & 54.8 & 53.2 & 65.6 & 64.0 \\
        \quad MSE-R (\%) & +36.97 & +35.07 & +34.20 & +35.57 & +38.17 & +2.50 & +2.39 & -5.02 & -5.28 \\
        \bottomrule
    \end{tabular}
    }
    \end{table}

\subsection{Model-Specific Pruning Mechanisms}
Causal Pruning offers a benefit of reducing architectural complexity from $O(N)$ to $O(|P(i)|)$ per node, with implementation strategies specific to each model class. For neural architectures such as MLP and N-BEATS, pruning is enforced by slicing the input tensors so that instead of feeding a full $N$-dimensional vector to the first hidden layer, the model only receives indices corresponding to $P(i)$. In the case of ARIMAX, the set $P(i)$ is treated as the collection of exogenous regressors, where masking the exogenous input matrix to include only causal parents allows the model to avoid estimating coefficients for non-causal variables. Similarly, Linear Regression (VAR) is formulated as a sparse Ordinary Least Squares (OLS) problem, ensuring that coefficients are estimated exclusively for variables within the causal neighborhood $P(i)$.

\subsection{Comparative Analysis with CUTS+ Baseline}
\label{CUTS+ forecasting appendix}
Table \ref{tab:relative_m2c_vs_cp_final} compares the performance of different forecasting models for one-step-ahead prediction task by performing causal pruning using causal masks from \textbf{Mask2Cause} and \textbf{CUTS+} baseline. We observe that for DREAM3 and Lorenz96 datasets, Mask2Cause provides causal matrices that have a better inductive bias. For VAR dataset, we see a slight drop in accuracy which is compensated by lower computational complexity.

\begin{table}[htbp]
    \centering
    \caption{Relative improvement of Mask2Cause over CUTS+ baseline (\%) in MSE reduction across datasets}
    \label{tab:relative_m2c_vs_cp_final}
    \scriptsize 
    \setlength{\tabcolsep}{3.5pt} 
    \renewcommand{\arraystretch}{0.75} 
    \begin{tabular}{lcccccccc}
        \toprule
        & \multicolumn{5}{c}{\textbf{DREAM3}} & \multicolumn{2}{c}{\textbf{Lorenz}} & \textbf{VAR} \\
        \cmidrule(lr){2-6} \cmidrule(lr){7-8} \cmidrule(lr){9-9}
        \textbf{Model} & \textbf{Ecoli1} & \textbf{Ecoli2} & \textbf{Yeast1} & \textbf{Yeast2} & \textbf{Yeast3} & \textbf{F=10} & \textbf{F=40} & \textbf{T=500} \\
        \midrule
        ARIMAX      & +2.25  & +4.16  & +0.67  & +2.57  & +2.74  & +11.40 & -1.92  & -0.38 \\
        MLP         & +6.31  & +9.70  & +6.31  & +8.02  & +7.93  & +20.66 & +11.70 & -0.52 \\
        N-BEATS     & +8.19  & +10.79 & +6.36  & +7.82  & +10.34 & +31.24 & +20.59 & -0.62 \\
        Linear Reg. & +2.14  & +4.08  & +0.82  & +2.61  & +2.93  & +9.49  & -1.33  & -0.20 \\
        \bottomrule
    \end{tabular}
\end{table}

\subsection{Comparison with other Feature Selection Methods} 
\label{Comparison with other Feature Selection Methods}
Table \ref{tab:comprehensive_pruning_results} provides details of various feature selection mechanisms (Causal Pruning, Mutual Information and Linear Granger Causality) evaluated on parameter reduction and gain in mean square error for common forecasting models.
We observe that Causal Pruning, which is a zero-shot byproduct of the Mask2Cause framework, demonstrates competitive results against alternative feature selection methods. 

\begin{table*}[h]
\centering
\caption{Model-wise comparison of parameter reduction (PR, \%) and MSE reduction (MSE-R, \%) across pruning methods. \textbf{Bold} indicates best performer, \underline{underline} indicates second best.}
\label{tab:comprehensive_pruning_results}
\setlength{\tabcolsep}{4pt} 
\resizebox{\textwidth}{!}{
\begin{tabular}{l c c c c c c c c c c c c}
\toprule
& \multicolumn{4}{c}{\textbf{VAR}} & \multicolumn{4}{c}{\textbf{Lorenz} ($F=10$)} & \multicolumn{4}{c}{\textbf{Lorenz} ($F=40$)} \\
\cmidrule(lr){2-5} \cmidrule(lr){6-9} \cmidrule(lr){10-13}
\textbf{Model / Method} & \multicolumn{2}{c}{\textbf{T=500}} & \multicolumn{2}{c}{\textbf{T=1000}} & \multicolumn{2}{c}{\textbf{T=250}} & \multicolumn{2}{c}{\textbf{T=500}} & \multicolumn{2}{c}{\textbf{T=250}} & \multicolumn{2}{c}{\textbf{T=500}} \\
\cmidrule(lr){2-3} \cmidrule(lr){4-5} \cmidrule(lr){6-7} \cmidrule(lr){8-9} \cmidrule(lr){10-11} \cmidrule(lr){12-13}
& PR & MSE-R & PR & MSE-R & PR & MSE-R & PR & MSE-R & PR & MSE-R & PR & MSE-R \\
\midrule
\textbf{MLP} \\
\quad M2C (Ours) & \textbf{60.01} & -4.41 & \underline{58.60} & -4.06 & \textbf{50.84} & \textbf{+49.35} & \textbf{50.04} & \textbf{+27.61} & \underline{48.37} & \textbf{+10.33} & \underline{48.87} & \textbf{+0.55} \\
\quad Mutual Info. & \underline{54.19} & \underline{-1.74} & \textbf{59.43} & \underline{-3.41} & 0.83 & +4.36 & 1.16 & +1.09 & 33.08 & \underline{-3.21} & 46.71 & -16.74 \\
\quad Linear Granger & 20.45 & \textbf{+0.85} & 9.00 & \textbf{+0.53} & \underline{26.46} & \underline{+27.78} & \underline{18.62} & \underline{+22.25} & \textbf{59.18} & -9.61 & \textbf{51.70} & \underline{-10.52} \\
\midrule
\textbf{ARIMAX} \\
\quad M2C (Ours) & \textbf{72.20} & -5.11 & \underline{70.50} & -5.62 & \textbf{61.17} & \underline{+2.41} & \textbf{60.20} & \textbf{+2.27} & \underline{58.20} & +1.13 & \underline{58.80} & \textbf{+1.27} \\
\quad Mutual Info. & \underline{65.20} & \underline{-2.88} & \textbf{71.50} & \underline{-3.85} & 1.00 & -0.24 & 1.40 & -0.91 & 39.80 & \underline{+3.01} & 56.20 & +0.37 \\
\quad Linear Granger & 24.60 & \textbf{+0.45} & 10.83 & \textbf{+0.30} & \underline{31.83} & \textbf{+2.42} & \underline{22.40} & \underline{-0.33} & \textbf{71.20} & \textbf{+3.18} & \textbf{62.20} & \underline{+1.02} \\

\midrule
\textbf{N-BEATS} \\
\quad M2C (Ours) & \underline{4.8} & \textbf{+10.72} & \underline{4.7} & \textbf{+1.13} & \textbf{7.9} & \textbf{+20.67} & \textbf{7.9} & \underline{+5.18} & \underline{7.6} & \textbf{-11.26} & \underline{7.6} & \textbf{-27.67} \\
\quad Mutual Info. & \textbf{8.47} & \underline{+5.68} & \textbf{9.28} & -0.69 & 0.13 & -7.82 & 0.18 & -15.73 & 5.17 & \underline{-58.71} & 7.30 & -103.78 \\
\quad Linear Granger & 3.19 & +3.34 & 1.41 & \underline{+0.92} & \underline{4.13} & \underline{+13.35} & \underline{2.91} & \textbf{+19.04} & \textbf{9.24} & -76.56 & \textbf{8.08} & \underline{-92.93} \\
\bottomrule
\end{tabular}
}
\end{table*}

\section{Downstream Utility: Root Cause Localisation}
\label{rca}

This appendix adds a second downstream task on real data, following the protocol of AERCA \citep{han2025aerca}. We use MSDS \citep{nedelkoski2020multi}, a real cloud microservice telemetry benchmark. No ground-truth causal graph exists for MSDS, so graph quality is assessed indirectly, by supplying the discovered adjacency matrix to Personalized PageRank and measuring root cause localisation accuracy at $k$ (AC@$k$).

\begin{table}[htbp]
\centering
\small
\caption{\textbf{Root cause localisation on MSDS.} AC@$k$ using the discovered causal matrix as input to Personalized PageRank.}
\label{tab:rca}
\begin{tabular}{lcccc}
\toprule
\textbf{Method} & \textbf{AC@1} & \textbf{AC@3} & \textbf{AC@5} & \textbf{AC@10} \\
\midrule
AERCA & $0.381 \pm 0.408$ & $\mathbf{0.908 \pm 0.062}$ & $0.974 \pm 0.027$ & $\mathbf{1.000 \pm 0.000}$ \\
\textbf{M2C-NLL} & $\mathbf{0.446 \pm 0.108}$ & $0.834 \pm 0.046$ & $\mathbf{0.995 \pm 0.003}$ & $\mathbf{1.000 \pm 0.000}$ \\
\bottomrule
\end{tabular}
\end{table}

M2C-NLL localises the root cause more often at $k=1$ and $k=5$, and the two methods are tied at $k=10$. The wide standard deviation on AERCA's AC@1 is worth noting, since it indicates that its top-ranked candidate is far less stable across runs than the aggregate figures suggest.

\section{Comparison Under Published Protocols}
\label{published}

The tables in Section 5 report methods we ran ourselves through a single scoring harness. Three recent methods could not be placed in those tables on the same footing, because we were unable to obtain a runnable implementation of any of them. We could not locate public code for CSAM \citep{liu2025} or KANGCI \citep{liu2025kolmogorov}, and the only implementation of LLM-GC \citep{liu2026llm} we found was an anonymised review artifact that we could not access. LLM-GC would in any case need the semantic modality it distils from a language model, which our synthetic benchmarks do not carry.

For these three, we quote their published numbers unchanged and run M2C on the configurations they report. This keeps the comparison meaningful at the cost of not passing every method through one harness, and we report it separately so that the provenance of each number is unambiguous. Two further limits apply to CSAM: it reports only precision, recall and F1, so F1 is the only comparable metric, and it does not state the Lorenz-96 forcing constant, so we report M2C at both values we evaluate. Table \ref{tab:llmgc} also lists JRNGC \citep{zhou2024jacobian}, a shared-weight method distinct from the component-wise NGC family, because LLM-GC reports it under the same protocol.

\begin{table}[htbp]
\centering
\small
\caption{\textbf{Lorenz-96 configurations from the LLM-GC protocol.} Baseline numbers are as published; M2C is run on their configurations. Settings are denoted $(N, T, F)$.}
\label{tab:llmgc}
\begin{tabular}{llcccc}
\toprule
\textbf{Setting} & \textbf{Metric} & \textbf{KANGCI} & \textbf{JRNGC} & \textbf{LLM-GC} & \textbf{M2C} \\
\midrule
\multirow{2}{*}{$(20, 1000, 10)$} & AUROC & $0.862 \pm 0.018$ & $0.934 \pm 0.018$ & $0.979 \pm 0.033$ & $\mathbf{0.989 \pm 0.003}$ \\
 & AUPRC & $0.875 \pm 0.009$ & $0.946 \pm 0.015$ & $\mathbf{0.984 \pm 0.018}$ & $0.964 \pm 0.005$ \\
\midrule
\multirow{2}{*}{$(20, 500, 20)$} & AUROC & $0.775 \pm 0.016$ & $0.903 \pm 0.020$ & $0.943 \pm 0.008$ & $\mathbf{0.998 \pm 0.001}$ \\
 & AUPRC & $0.780 \pm 0.014$ & $0.925 \pm 0.022$ & $0.950 \pm 0.015$ & $\mathbf{0.993 \pm 0.002}$ \\
\midrule
\multirow{2}{*}{$(40, 1000, 20)$} & AUROC & $0.719 \pm 0.015$ & $0.907 \pm 0.008$ & $0.932 \pm 0.005$ & $\mathbf{0.996 \pm 0.002}$ \\
 & AUPRC & $0.766 \pm 0.011$ & $0.909 \pm 0.017$ & $0.940 \pm 0.016$ & $\mathbf{0.968 \pm 0.004}$ \\
\bottomrule
\end{tabular}
\end{table}

M2C leads on five of the six entries, and is the only method whose AUROC stays above 0.98 in all three settings.

\begin{table}[htbp]
\centering
\small
\caption{\textbf{CSAM comparison on Lorenz-96 ($N=10$), F1.} The CSAM number is taken directly from the published paper.}
\label{tab:csam}
\begin{tabular}{lcc}
\toprule
\textbf{Method} & \textbf{Samples} & \textbf{F1} \\
\midrule
CSAM \citep{liu2025} & 2048 & 0.769 \\
M2C-MSE ($F=10$) & 2048 & $\mathbf{1.000}$ \\
M2C-MSE ($F=40$) & 2048 & $\mathbf{1.000}$ \\
\bottomrule
\end{tabular}
\end{table}

\section{Heteroscedastic Structure in Real Data}
\label{realhetero}

The Mixed Physics benchmark is synthetic, and this appendix sets out why that is difficult to avoid and what indirect evidence real data can supply.

\textbf{Why no validated real benchmark exists.} We looked for a curated real-world time series carrying a known variance-driven causal graph and found none. Volatility spillover in finance and coherence modulation in neuroscience are the canonical examples, but in both fields the structure would have to be taken from tools such as Diebold-Yilmaz connectedness indices or coherence measures, and those tools are themselves estimators of that structure. Scoring a discovery method against their output would measure how closely two estimators agree, not whether either one is correct. This is why we treat Mixed Physics as a contribution in its own right: a capability cannot be evaluated until a benchmark exists that isolates it.

\textbf{Indirect evidence on ETTh1.} During our experiments we observed that ETTh1, an electricity transformer temperature dataset, behaves as a heteroscedastic system would predict. We prune a one-step-ahead forecaster to the parents each method discovers, and compare against a random graph matched to the same edge density, which controls for capacity rather than structure.

\begin{table}[htbp]
\centering
\small
\caption{\textbf{Causal pruning on ETTh1.} One-step-ahead forecasting MSE after restricting each target to its discovered parents, against a random graph matched to the same edge density.}
\label{tab:etth1}
\begin{tabular}{lcccc}
\toprule
\textbf{Pruning source} & \textbf{Edges} & \textbf{Param. reduction} & \textbf{MSE} & \textbf{MSE (density-matched random)} \\
\midrule
Vanilla (all edges) & 42 & 0\% & 0.179 & --- \\
AERCA \citep{han2025aerca} & 3 & 72.6\% & 0.176 & \textbf{0.174} \\
M2C-MSE & 20 & 40.9\% & \textbf{0.174} & 0.182 \\
M2C-NLL & 19 & 42.8\% & \textbf{0.169} & 0.178 \\
\bottomrule
\end{tabular}
\end{table}

Both M2C graphs beat their density-matched random controls while AERCA's does not, which indicates that the discovered structure carries information beyond sparsity alone. The NLL variant's graph is the more useful of the two, and both improve on the unpruned model. On a purely homoscedastic system we would expect the NLL and MSE variants to be interchangeable, as they are on VAR and Lorenz-96, so the gap between them is itself the signal of interest.

This is suggestive evidence, not proof. ETTh1 has no causal ground truth and the differences are small. What it shows is consistent with variance-driven structure being present and exploitable in a real system, which is the most that can be claimed without a labelled benchmark.

\section{Practical Guidelines: Choosing Between Homoscedastic (MSE) and Heteroscedastic (NLL) Objectives}
\label{guide}

To contextualize the empirical results presented in the main text, it is important to note that the underlying data-generating processes for the VAR, Lorenz-96, CausalTime, and DREAM3 benchmarks are inherently homoscedastic. Therefore, the Mean Squared Error (MSE) objective is usually optimal for these systems. 

The introduction of the Negative Log-Likelihood (NLL) objective—which jointly estimates a time-varying variance—carries a trade-off between representational capacity and parameter efficiency. On highly complex datasets with scarce samples, such as the Yeast1 network (which contains 100 genes and 166 regulatory edges, with up to seven regulators per gene), the variance estimation can destabilize, degrading the structural learning signal. Conversely, on simpler DREAM3 datasets, the NLL objective occasionally outperforms MSE. In these specific regimes, the dynamic variance acts as an adaptive regularizer, intelligently down-weighting noisy or highly volatile transitions during gradient descent.

When deploying Mask2Cause on a novel, real-world dataset where the ground truth is unknown, we recommend the following decision framework:

\begin{itemize}
    \item \textbf{Data Scarcity:} If the dataset is highly constrained in the time dimension ($T$) relative to the number of variables ($N$), default to the \textbf{MSE} objective. The parameter efficiency and optimization stability provided by a fixed-variance assumption will generally outweigh the benefits of adaptive variance modeling.
    \item \textbf{Rolling Variance Check:} For appropriately sized datasets, practitioners should pre-compute the rolling variance of the target variables. If the variance appears stationary across time, default to \textbf{MSE}. However, if the time series exhibits distinct volatility clustering or heteroscedastic shocks (e.g., high-frequency financial econometrics, neurological seizure data), deploy the \textbf{NLL} objective.
\end{itemize}

\section{Proofs}
\label{proof}
\subsection{Neural Granger Causality is a Special Case of Directed Information}
\label{proof1}

\textbf{Theorem} 
Let $x^i_t = g_i(\mathbf{x}_{<t}) + \epsilon^i_t$, where $\epsilon^i_t$ has a density, is independent of the full history, and has zero mean. Assume finite first moments. Then $X^j$ is functionally non-causal for $X^i$ almost surely (i.e., $g_i$ admits a representation using only $\mathbf{x}^{-j}_{<t}$) if and only if $I(X^j \to X^i \mid \mathbf{X}^{-\{{i,j}\}})=0$.

\begin{proof}

The Causally Conditioned Directed Information (DI) from a source $X^j$ to a target $X^i$ is defined as the expected sum of Kullback-Leibler (KL) divergences between the full conditional distribution and the restricted conditional distribution:
\begin{equation}
    I(X^j \to X^i \mid \mathbf{X}^{-\{{i,j}\}})
    = \sum_{t=1}^T \mathbb{E} \Big[
    D_{\mathrm{KL}} \big(
    P(x^i_t \mid \mathbf{x}_{<t}) \,\big\| \,
    P(x^i_t \mid \mathbf{x}^{i}_{<t}, \mathbf{x}^{-\{{i,j}\}}_{<t})
    \big)
    \Big]
\end{equation}

Under the Additive Noise Model (ANM) assumption, the randomness in $x^i_t$ arises solely from the noise term $\epsilon^i_t$. Consequently, the conditional probability density function is equivalent to the noise density shifted by the deterministic prediction function $g_i$:
\begin{equation}
    P(x^i_t \mid \mathbf{x}_{<t}) = P_{\epsilon} \left( x^i_t - g_i(\mathbf{x}_{<t}) \right)
\end{equation}

If $X^j$ is functionally non-causal for $X^i$, then by definition, the function $g_i$ does not depend on the history of $X^j$. Mathematically, this implies invariance:
\begin{equation}
    g_i(\mathbf{x}_{<t}) = g_i(\mathbf{x}^{i}_{<t}, \mathbf{x}^{-\{{i,j}\}}_{<t})
\end{equation}

Substituting this invariance into the density formulation:
\begin{equation}
    \begin{aligned}
        P(x^i_t \mid \mathbf{x}_{<t}) &= P_{\epsilon} \left( x^i_t - g_i(\mathbf{x}^{i}_{<t}, \mathbf{x}^{-\{{i,j}\}}_{<t}) \right) \\
        &= P(x^i_t \mid \mathbf{x}^{i}_{<t}, \mathbf{x}^{-\{{i,j}\}}_{<t})
    \end{aligned}
\end{equation}

Since the full conditional distribution $P(x^i_t \mid \mathbf{x}_{<t})$ is identical to the restricted distribution $P(x^i_t \mid \mathbf{x}^{i}_{<t}, \mathbf{x}^{-\{{i,j}\}}_{<t})$ at every time step $t$, the KL divergence term vanishes:
\begin{equation}
    D_{\text{KL}} \Big( P(x^i_t \mid \mathbf{x}_{<t}) \,\Big\| \, P(x^i_t \mid \mathbf{x}^{i}_{<t}, \mathbf{x}^{-\{{i,j}\}}_{<t}) \Big) = \mathbb{E} \left[ \log \frac{\cancel{P(x^i_t \mid \mathbf{x}_{<t})}}{\cancel{P(x^i_t \mid \mathbf{x}^{i}_{<t}, \mathbf{x}^{-\{{i,j}\}}_{<t})}} \right] =  0
\end{equation}
Thus, the Causally Conditioned Directed Information is exactly zero.

Conversely, if DI is zero, nonnegativity of KL divergence implies equality of the full and source-excluded conditional distributions almost surely at every evaluated time step. Their conditional means therefore agree. Since the noise is independent of the history and has zero mean,
\[
 g_i(\mathbf{x}_{<t})
 = \mathbb{E}[x_t^i\mid\mathbf{x}_{<t}]
 = \mathbb{E}[x_t^i\mid\mathbf{x}^{-j}_{<t}]
 \quad\text{almost surely}.
\]
The right-hand side depends only on the remaining histories, establishing functional noncausality almost surely.
\end{proof}

\noindent\textbf{Remark (Failure of the Converse in Heteroscedastic Systems).} 
It is crucial to note that, for general dynamical systems with finite first moments, zero DI implies invariance of the conditional mean, but the converse does not hold. If a variable $X^j$ influences the \textit{variance} (or higher-order moments) of $X^i$ on a set of positive probability without affecting its conditional mean $g_i$, then $g_i(\mathbf{x}_{<t}) = g_i(\mathbf{x}^{-j}_{<t})$ holds, satisfying the condition for Functional Non-Causality (zero NGC). However, in such a case, the probability distributions would differ in shape (e.g., width), i.e., $P(x^i_t \mid \mathbf{x}_{<t}) \neq P(x^i_t \mid \mathbf{x}^{-j}_{<t})$. The Directed Information, which detects any distributional discrepancy, would correctly be non-zero ($I > 0$). This gap highlights precisely why standard functional methods fail on the Mixed Physics benchmark, while probabilistic objectives succeed.

\subsection{Optimal Gaussian Prediction and Conditional Moments}
\label{proof:gaussian_moments}

\textbf{Proposition.} Fix a target $x_t^i$ with finite second moment. Let $(\mu,v)$ and $(\mu_{-j},v_{-j})$ be its conditional mean and variance given $\mathbf{x}_{<t}$ and $\mathbf{x}^{-j}_{<t}$, respectively. Assume $v,v_{-j}>0$ almost surely and $\mathbb{E}|\log v|+\mathbb{E}|\log v_{-j}|<\infty$. Among unrestricted measurable Gaussian predictors with positive variance, the population-optimal predictions equal these conditional moments. Excluding $X^j$ strictly increases optimal Gaussian NLL if and only if either moment changes on a set of positive probability; it strictly increases optimal MSE if and only if the conditional mean changes on such a set.

\begin{proof}
For either choice of history, the conditional expected Gaussian NLL at predicted mean $m$ and variance $s^2>0$, omitting the constant $\tfrac12\log(2\pi)$, is
\[
 \frac12\log s^2+\frac{v+(\mu-m)^2}{2s^2}.
\]
It is uniquely minimized by $m=\mu$ and $s^2=v$; the same argument applies to the source-excluded moments. Evaluating both optimal predictors under the full history and then taking expectations gives
\begin{align*}
 \mathcal{L}_{\mathrm{NLL},-j}^{*}-\mathcal{L}_{\mathrm{NLL},\mathrm{full}}^{*}
 &=\frac12\mathbb{E}\!\left[\log\frac{v_{-j}}{v}
       +\frac{v+(\mu-\mu_{-j})^2}{v_{-j}}-1\right] \\
 &=\mathbb{E}\!\left[D_{\mathrm{KL}}\!\left(
       \mathcal{N}(\mu,v)\,\Vert\,\mathcal{N}(\mu_{-j},v_{-j})\right)\right].
\end{align*}
This is zero exactly when both moments agree almost surely, and strictly positive otherwise. For squared error, the conditional expectation minimizes risk and the orthogonality of its residual yields
\[
 \mathcal{L}_{\mathrm{MSE},-j}^{*}-\mathcal{L}_{\mathrm{MSE},\mathrm{full}}^{*}
 =\mathbb{E}[(\mu-\mu_{-j})^2].
\]
Thus a variance-only parent improves optimal Gaussian prediction while leaving optimal squared error unchanged. The data need not be Gaussian: the Gaussian distributions in the KL expression are the optimal predictive models, not necessarily the true conditional densities.
\end{proof}

\paragraph{Corollary (Gaussian full conditionals and DI).}
\label{proof:gaussian_di_support}
Under the preceding moment and integrability assumptions at each evaluated time step, suppose additionally that $p(x_t^i\mid\mathbf x_{<t})$ is Gaussian almost surely. For $j\ne i$, the cumulative optimal Gaussian NLL increases upon excluding source $j$ if and only if $I(X^j\to X^i\Vert\mathbf X^{-\{i,j\}})>0$. The source-excluded conditional density need not be Gaussian.

\begin{proof}
Zero DI means that the full and source-excluded conditional distributions agree almost surely at every evaluated time step. Their moments then agree, so the preceding proposition gives zero Gaussian NLL increase. Conversely, zero cumulative Gaussian NLL increase implies equality of the full and reduced moments at every step, since each step's increase is nonnegative. The full conditional law is therefore $\mathcal N(\mu_{-j},v_{-j})$, whose parameters depend only on the source-excluded history. Averaging this law over the missing history leaves it unchanged, so it also equals the reduced conditional law. Every DI summand is consequently zero. Taking the contrapositive establishes the equivalence of strict positivity. This identifies the same population edges, not necessarily the same numerical loss reduction as DI, and does not assert recovery by a fitted network.
\end{proof}

\subsection{Proof of Zero Self-Information under Causal Conditioning}
\label{proof:zero_self_info}

\textbf{Proposition.} \textit{For any variable $X^i$, the Causally Conditioned Directed Information from $X^i$ to itself is exactly zero:}
\begin{equation}
    I(X^i \to X^i \mid \mathbf{X}^{-\{{i}\}}) = 0
\end{equation}

\begin{proof}
Recall the definition of Causally Conditioned Directed Information from a source $X^j$ to a target $X^i$, conditioned on the set of all other variables $\mathbf{X}^{-\{{i,j}\}}$:
\begin{equation}
    I(X^j \to X^i \mid \mathbf{X}^{-\{{i,j}\}})
    = \sum_{t=1}^T \mathbb{E} \Big[
    D_{\mathrm{KL}} \big(
    P(x^i_t \mid \mathbf{x}_{<t}) \,\big\| \,
    P(x^i_t \mid \mathbf{x}^{i}_{<t}, \mathbf{x}^{-\{{i,j}\}}_{<t})
    \big)
    \Big]
\end{equation}
where $\mathbf{x}_{<t}$ denotes the full history of all variables in the system up to time $t-1$.

To evaluate the self-information flow, we set the source variable to be the target itself, i.e., let $j = i$.
In this case, the exclusion set $\mathbf{X}^{-\{{i,j}\}}$ becomes $\mathbf{X}^{-\{{i,i}\}} = \mathbf{X}^{-\{{i}\}}$. The term inside the second probability distribution becomes:
\begin{equation}
    \{ \mathbf{x}^{i}_{<t}, \mathbf{x}^{-\{{i}\}}_{<t} \} \equiv \mathbf{x}_{<t}
\end{equation}
This is simply the full system history $\mathbf{x}_{<t}$.

Substituting this back into the KL divergence term:
\begin{equation}
    D_{\mathrm{KL}} \big( P(x^i_t \mid \mathbf{x}_{<t}) \,\|\, P(x^i_t \mid \mathbf{x}_{<t}) \big)
\end{equation}
Since the Kullback-Leibler divergence between two identical distributions is exactly zero ($D_{\mathrm{KL}}(P \| P) = 0$), every term in the summation vanishes. Therefore:
\begin{equation}
    I(X^i \to X^i \mid \mathbf{X}^{-\{{i}\}}) = \sum_{t=1}^T 0 = 0
\end{equation}
\end{proof}

\subsection{A Concrete Variance-Only Dependency}
\label{proof:variance_only_example}

Let $(X_t)$ and $(\epsilon_t)$ be mutually independent sequences of independent standard Gaussian variables, and define
\[
Y_t=\sqrt{\beta+aX_{t-1}^{2}}\,\epsilon_t,
\qquad \beta>0,\quad a>0.
\]
The source $X_{t-1}$ is independent of $Y_{<t}$. Consequently,
\[
\mathbb E[Y_t\mid X_{<t},Y_{<t}]
=\mathbb E[Y_t\mid Y_{<t}]=0,
\]
whereas
\[
\operatorname{Var}(Y_t\mid X_{<t},Y_{<t})=\beta+aX_{t-1}^{2},
\qquad
\operatorname{Var}(Y_t\mid Y_{<t})=\beta+a.
\]
Both optimal squared-error predictors therefore output zero and have expected loss $\beta+a$. Nevertheless, the full and reduced conditional distributions differ with positive probability, since their variances differ. Thus $I(X_{<t};Y_t\mid Y_{<t})>0$, yielding a positive contribution to the strictly lagged DI from $X$ to $Y$.

The optimal expected Gaussian NLL reduction is
\[
\frac12\left\{\log(\beta+a)
-\mathbb E\big[\log(\beta+aX_{t-1}^{2})\big]\right\}>0.
\]
Strict positivity follows from strict concavity of the logarithm and the nonconstant conditional variance; $\beta>0$ and finite second moments ensure the required integrability. This is a variance-only dependency that optimal Gaussian prediction detects and optimal squared-error prediction misses. It lies outside the history-independent additive-noise premise of Appendix~\ref{proof1}.

\subsection{Idealized Mean-Based AUROC Reference}
\label{proof:auroc_reference}

\textbf{Proposition.} Suppose a graph has $m$ mean-dependent edges, $v$ variance-only edges and at least one non-edge, with $m+v>0$. Assume every mean-dependent edge scores above every non-edge, while a uniformly selected variance-only edge versus a uniformly selected non-edge has probability of winning plus half its probability of tying equal to $1/2$. Then the tie-adjusted AUROC is
\[
\operatorname{AUROC}_{\mathrm{ref}}
=\frac{m+\tfrac12 v}{m+v}.
\]

\begin{proof}
AUROC is the probability that a uniformly selected positive edge scores above a uniformly selected non-edge, with ties contributing one half. The positive edge is mean-dependent with probability $m/(m+v)$ and variance-only with probability $v/(m+v)$. Conditioning on these two cases gives
\[
\operatorname{AUROC}_{\mathrm{ref}}
=\frac{m}{m+v}\cdot 1+\frac{v}{m+v}\cdot\frac12.
\]
If the chance-ranking premise holds in expectation over random scores, the same calculation gives expected AUROC.
\end{proof}

The Mixed Physics configurations have $(m,v)=(13,14),(7,20),(0,27)$, giving $20/27\approx0.741$, $17/27\approx0.630$ and $1/2=0.500$, respectively. These are reference values under the stated ranking assumptions, not universal bounds or guaranteed AUROCs for MSE-trained networks. The absence of a population MSE gain for a variance-only source does not itself determine the edge-score rankings of a fitted model.

\subsection{Log-Loss Reduction for Fitted Density Predictors}
\label{proof:fitted_density}

The identity in Eq.~\eqref{eq:di_logloss} uses the true conditional densities. A neural density predictor may approximate these imperfectly. The following standard log-loss decomposition separates the information supplied by a source from the errors in predicting the two conditional distributions.

\textbf{Proposition.} For a fixed target $i$ and source $j\ne i$, let $p_{\mathrm{full},t}$ and $p_{-j,t}$ be the true full-history and source-excluded conditional densities. Let $q_{\mathrm{full},t}$ and $q_{-j,t}$ be any fixed predictive densities using only their respective histories; in particular, $q_{-j,t}$ cannot access source $j$'s past. Assume the expected absolute log-densities evaluated at $x_t^i$ are finite. Define the cumulative expected modeling errors
\[
\mathcal E_{\mathrm{full}}=\sum_{t=1}^T\mathbb E\big[D_{\mathrm{KL}}(p_{\mathrm{full},t}\Vert q_{\mathrm{full},t})\big],
\qquad
\mathcal E_{-j}=\sum_{t=1}^T\mathbb E\big[D_{\mathrm{KL}}(p_{-j,t}\Vert q_{-j,t})\big].
\]
The cumulative expected NLL reduction from using the full predictor is
\begin{equation}
\begin{split}
\Delta_q
&:=\mathbb E\!\left[\sum_{t=1}^T
\log\frac{q_{\mathrm{full},t}(x_t^i)}{q_{-j,t}(x_t^i)}\right]\\
&=I(X^j\to X^i\Vert\mathbf X^{-\{i,j\}})
+\mathcal E_{-j}-\mathcal E_{\mathrm{full}}.
\end{split}
\label{eq:fitted_density_di}
\end{equation}

\begin{proof}
For either history choice, conditional expectation gives
\[
\mathbb E[-\log q_t(x_t^i)]
=\mathbb E[-\log p_t(x_t^i)]
+\mathbb E[D_{\mathrm{KL}}(p_t\Vert q_t)].
\]
Subtract the full-history identity from the source-excluded identity, sum over time, and apply Eq.~\eqref{eq:di_logloss}, attributed to \citet{quinn2015directed}.
\end{proof}

\paragraph{Interpretation.}
Each predictor's expected loss equals the best achievable loss given its available histories plus its error in modeling the corresponding distribution. Their loss difference therefore reflects both the source's predictive information and the difference between their modeling errors. A reduced predictor can model its conditional distribution perfectly despite having less information: the information loss is already accounted for by DI. If both modeling errors are zero, $\Delta_q$ equals DI; if each is at most $\varepsilon$, then $|\Delta_q-I|\le\varepsilon$, where $I$ denotes the DI in Eq.~\eqref{eq:fitted_density_di}.

\section{Limitations and Broader Impacts}
\label{app:limitations_impacts}

\subsection{Limitations} 
\label{limit}
While Mask2Cause successfully captures both mean and variance-driven causality, the heteroscedastic NLL objective uses a Gaussian predictive density, although its population-optimal moments remain valid for non-Gaussian data under the conditions in Appendix \ref{proof:gaussian_moments}. Because Gaussian NLL penalizes errors quadratically, extreme outliers in heavy-tailed data force the network to artificially inflate the predicted variance, which can wash out the structural gradient (potentially leading to spurious edge detection). If heavy tails are detected in the dataset (e.g., via a high kurtosis statistic), the framework can however be trivially adjusted by swapping the Gaussian likelihood for a Student-t or Laplace distribution NLL. Like all methods grounded in the Granger Causality framework, Mask2Cause relies on the assumption of causal sufficiency. It assumes there are no unobserved latent confounders driving the observed variables. In the presence of hidden confounders, Granger-based methods may infer spurious direct edges. Extending this framework to detect or mitigate hidden confounding remains an important direction for future work. Our architecture assumes the causal graph is fixed. While the conditional mean and variance can change with the observed history even in a stationary process, the underlying structural adjacency matrix $\hat{\mathbf{A}}$ is assumed to be static across the evaluated time horizon $T$. Adapting the masking mechanism to handle non-stationary, regime-switching causal structures is a promising avenue for further research.

\subsection{Broader Impacts}
\label{impact}
Causal discovery from time series has profound implications across domains ranging from climate science to financial econometrics. On the positive side, robust causal graphs can improve the interpretability and safety of automated decision-making systems (e.g., medical interventions or traffic grid optimization) by distinguishing true mechanisms from spurious correlations. A potential negative impact arises if the models are applied to sensitive human-centric data without accounting for unobserved confounders, which could lead to flawed causal conclusions and biased policy interventions. To mitigate this, practitioners must validate the discovered graphs against domain expertise before deployment.

\end{document}